\documentclass{article}

\usepackage{microtype}
\usepackage{graphicx}
\usepackage{subcaption}
\usepackage{booktabs} %
\usepackage{amsthm}
\usepackage{enumitem}
\usepackage{amssymb}
\usepackage{amsmath}
\usepackage{wrapfig}
\usepackage{verbatim}
\usepackage[hyphens]{url}
\usepackage{chngcntr}

\newcommand{\repo}{\url{github.com/tonyduan/rs4a}}

\usepackage[accepted,nohyperref]{icml2020}
\usepackage[dvipsnames]{xcolor}
\definecolor{mydarkblue}{rgb}{0,0.08,0.45}
\usepackage[colorlinks=true,
            linkcolor=mydarkblue,
            citecolor=mydarkblue,
            filecolor=mydarkblue,
            urlcolor=mydarkblue,
            bookmarksopen=true,
            bookmarks=true]{hyperref}
\usepackage{bookmark}
\usepackage[capitalize,nameinlink]{cleveref}
\usepackage{mymacros}

\theoremstyle{}
\newtheorem{claim}[thm]{Claim}

\newcommand{\Vol}{\mathrm{Vol}}
\newcommand{\Zon}{\mathrm{Zon}}
\newcommand{\defeq}{\overset{\mathrm{def}}{=}}
\newcommand{\Vrt}{\mathrm{Vert}}
\newcommand{\npset}{\mathcal{NP}}
\newcommand{\growth}{\mathcal{G}}
\newcommand{\cQ}{\mathcal{Q}}
\renewcommand{\Pr}{\mathbb{P}}

\newcommand{\BinomCDF}{\mathrm{BinomCDF}}
\newcommand{\Rademacher}{\mathrm{Rademacher}}
\newcommand{\Gammadist}{\mathrm{Gamma}}
\newcommand{\GammaCDF}{\mathrm{GammaCDF}}
\newcommand{\GaussianCDF}{\mathrm{GaussianCDF}}
\newcommand{\BetaCDF}{\mathrm{BetaCDF}}
\newcommand{\BetaVar}{\mathrm{Beta}}
\newcommand{\BetaPrime}{\mathrm{BetaPrime}}
\newcommand{\BetaPrimeCDF}{\mathrm{BetaPrimeCDF}}
\newcommand{\eps}{\varepsilon}
\newcommand{\dtv}{d_{\mathrm{TV}}}
\newcommand{\arctanh}{\operatorname{arctanh}}
\newcommand{\BV}{\mathrm{BV}}
\renewcommand{\div}{\operatorname{div}}

\newcommand{\mygamma}[2]{\gamma_{#1}^{(#2)}}

\DeclareFontFamily{U}{matha}{\hyphenchar\font45}
\DeclareFontShape{U}{matha}{m}{n}{
      <5> <6> <7> <8> <9> <10> gen * matha
      <10.95> matha10 <12> <14.4> <17.28> <20.74> <24.88> matha12
      }{}
\DeclareSymbolFont{matha}{U}{matha}{m}{n}
\DeclareFontSubstitution{U}{matha}{m}{n}

\DeclareFontFamily{U}{mathx}{\hyphenchar\font45}
\DeclareFontShape{U}{mathx}{m}{n}{
      <5> <6> <7> <8> <9> <10>
      <10.95> <12> <14.4> <17.28> <20.74> <24.88>
      mathx10
      }{}
\DeclareSymbolFont{mathx}{U}{mathx}{m}{n}
\DeclareFontSubstitution{U}{mathx}{m}{n}

\DeclareMathDelimiter{\vvvert}{0}{matha}{"7E}{mathx}{"17}

\newcommand{\mynorm}[1]{\left\vvvert#1\right\vvvert}

\icmltitlerunning{Randomized Smoothing of All Shapes and Sizes}

\begin{document}

\twocolumn[
\icmltitle{Randomized Smoothing of All Shapes and Sizes}

\icmlsetsymbol{equal}{*}

\begin{icmlauthorlist}
\icmlauthor{Greg Yang}{equal,ms}
\icmlauthor{Tony Duan}{equal,ms,res}
\icmlauthor{J. Edward Hu}{ms,res}
\icmlauthor{Hadi Salman}{ms}
\icmlauthor{Ilya Razenshteyn}{ms}
\icmlauthor{Jerry Li}{ms}
\end{icmlauthorlist}

\icmlaffiliation{ms}{Microsoft Research AI}
\icmlaffiliation{res}{Work done as part of the \href{https://www.microsoft.com/en-us/research/academic-program/microsoft-ai-residency-program/}{Microsoft AI Residency Program}}

\icmlcorrespondingauthor{Greg Yang}{\texttt{gregyang@microsoft.com}}
\icmlcorrespondingauthor{Tony Duan}{\texttt{tony.duan@microsoft.com}}
\icmlcorrespondingauthor{Jerry Li}{\texttt{jerrl@microsoft.com}}

\icmlkeywords{Machine Learning, ICML}

\vskip 0.3in
]

\printAffiliationsAndNotice{\icmlEqualContribution} %

\begin{abstract}

Randomized smoothing is the current state-of-the-art defense with provable robustness against $\ell_2$ adversarial attacks.
Many works have devised new randomized smoothing schemes for other metrics, such as $\ell_1$ or $\ell_\infty$; however, substantial effort was needed to derive such new guarantees.
This begs the question: can we find a general theory for randomized smoothing?

We propose a novel framework for devising and analyzing randomized smoothing schemes, and validate its effectiveness in practice.
Our theoretical contributions are: (1) we show that for an appropriate notion of ``optimal'', the optimal smoothing distributions for any ``nice'' norms have level sets given by the norm's \emph{Wulff Crystal};
(2) we propose two novel and complementary methods for deriving provably robust radii for any smoothing distribution;
and, (3) we show fundamental limits to current randomized smoothing techniques via the theory of \emph{Banach space cotypes}.
By combining (1) and (2), we significantly improve the state-of-the-art certified accuracy in $\ell_1$ on standard datasets.
Meanwhile, we show using (3) that with only label statistics under random input perturbations, randomized smoothing cannot achieve nontrivial certified accuracy against perturbations of $\ell_p$-norm $\Omega(\min(1, d^{\f 1 p - \f 1 2}))$, when the input dimension $d$ is large.
We provide code in \repo{}.
\end{abstract}

\section{Introduction}
\enlargethispage{\baselineskip}

Deep learning models are vulnerable to adversarial examples -- small imperceptible perturbations to their inputs that lead to misclassification \citep{goodfellow_explaining_2015, szegedy_intriguing_2014}.
To solve this problem, recent works proposed heuristic defenses that are robust to specific classes of perturbations, but many would later be broken by stronger attacking algorithms~\citep{carlini2017adversarial, athalye2018obfuscated, uesato2018adversarial}.
This led the community to both strengthen empirical defenses \cite{kurakin2016adversarial,madry2017towards} as well as build \emph{certified} defenses that provide robustness guarantees, i.e., models whose predictions are constant within a neighborhood of their inputs \citep{wong2018provable,raghunathan2018certified}.
In particular, \emph{randomized smoothing} is a recent method that has achieved state-of-the-art provable robustness~\citep{lecuyer2018certified, li_certified_2019, cohen_certified_2019}.
In short, given an input, it outputs the class most likely to be returned by a base classifier, typically a neural network, under random noise perturbation of the input.
This mechanism confers stability of the output against $\ell_p$ perturbations, even if the base classifier itself is highly non-Lipschitz.
Canonically, this noise has been Gaussian, and the adversarial perturbation it protects against has been $\ell_2$ \citep{cohen_certified_2019,salman2019provably,Zhai2020MACER}, but some have explored other kinds of noises and adversaries as well \citep{lecuyer2018certified,li_certified_2019,dvijotham_framework_2019}.
In this paper, we seek to comprehensively understand the interaction between the choice of smoothing distribution and the perturbation norm.%
\footnote{V2 update: we added results using stability training, semi-supervised learning, and ImageNet pre-training. See \cref{tab:sota}.}
\enlargethispage{\baselineskip}
\begin{enumerate}[nosep]
    \item We propose two new methods to compute robust certificates for additive randomized smoothing against different norms.
    \item We show that, for $\ell_1, \ell_2, \ell_\infty$ adversaries, the optimal smoothing distributions have level sets that are their respective \emph{Wulff Crystals} --- a kind of equilibrated crystal structure studied in physics since 1901 (\citeauthor{wulffZurFrageGeschwindigkeit1901}).
    \item Using the above advances, we obtain state-of-the-art $\ell_1$ certified accuracy on CIFAR-10 and ImageNet. With stability training \citep{li_certified_2019}, semi-supervised learning \citep{carmon_unlabeled_2019}, and pre-training in the fashion of \citet{hendrycks_using_2019}, we further improve CIFAR-10 certified accuracies, with $>30\%$ advantage over prior SOTA for $\ell_1$ radius $\ge 1.5$.
    See \cref{tab:sota}.
    \item Finally, we leverage the classical theory of Banach space \emph{cotypes} \citep{wojtaszczyk1996banach} to show that current techniques for randomized smoothing cannot certify nontrivial accuracy at more than $\Omega(\min(1, d^{\f 1 p - \f 1 2}))$ $\ell_p$-radius, if all one uses are the probabilities of labels when classifying randomly perturbed input.
\end{enumerate}

\begin{table*}[t!]
    \centering
    \begin{tabular}{llrrrrrrrr} \toprule
         ImageNet & $\ell_1$ Radius &  0.5 & 1.0 & 1.5 & 2.0 & 2.5 & 3.0 & 3.5 & 4.0 \\ \midrule
         & Laplace, \citet{teng_ell_1_2019} (\%)  & 48 & 40 & 31 & 26 & 22 & 19 & 17 & 14\\
         & Uniform, Ours (\%) & 55 & 49 & 46 & 42 & 37 & 33 & 28 & 25\\
         & + Stability Training & \bf 60 & \bf 55 & \bf 51 & \bf 48 & \bf 45 & \bf 43 & \bf 41 & \bf 39\\
          \midrule
         CIFAR-10 & $\ell_1$ Radius & 0.5 & 1.0 & 1.5 & 2.0 & 2.5 & 3.0 & 3.5 & 4.0 \\ \midrule
         & Laplace, \citet{teng_ell_1_2019} (\%) & 61 & 39 & 24 & 16 & 11 & 7 & 4 & 3\\
         & Uniform, Ours (\%) & 70 & 59 & 51 & 43 & 33 & 27 & 22 & 18\\
         & + Stability Training & 70 & 60 & 53 & 47 & \bf 43 & 39 & 35 & 28\\
         & + Stability Training, Semi-supervision & \bf 74 & \bf 63 & 54 & \bf 48 & \bf 43 & 38 & 34 & 31 \\
         & + Stability Training, Pre-training & \bf 74 & 62 & \bf 55 & \bf 48 & \bf 43 & \bf 40 & \bf 37 & \bf 33\\
         \bottomrule
    \end{tabular}
    \caption[Caption]{Certified top-1 accuracies of our $\ell_1$-robust classifiers, vs previous state-of-the-art, at various radii, for ImageNet and CIFAR-10.\protect\footnotemark}
    \label{tab:sota}
\end{table*}

\section{Related Works}
\label{sec:relatedWorks}

Defences against adversarial examples are mainly divided into \emph{empirical} defenses and \emph{certified} defenses.

\textbf{Empirical defenses} are heuristics designed to make learned models empirically robust. An example of these are \emph{adversarial training} based defenses  \citep{ kurakin2016adversarial, madry2017towards} which optimize the parameters of a model by minimizing the worst-case loss over a neighborhood around the input to these models~\citep{carlini2017adversarial, laidlaw2019functional, wong2019wasserstein, hu2020improved}.
Such defenses may seem powerful, but have no guarantees that they are not ``breakable''. In fact, the majority of the empirical defenses proposed in the literature were later ``broken'' by stronger attacks \citep{carlini2017adversarial, athalye2018obfuscated, uesato2018adversarial, athalye2018robustness}.

\textbf{Certified defenses} guarantee that for any input $x$, the classifier's output is constant within a small neighborhood of $x$. Such defenses are typically based on certification methods that are either \textit{exact} or \textit{conservative}. Exact methods include those based on Satisfiability Modulo Theories solvers~\citep{katz2017reluplex, ehlers2017formal} or
mixed integer linear programming \citep{tjeng2018evaluating, lomuscio2017approach, fischetti2017deep}, which, although guaranteed to find adversarial examples if they exist,  are unfortunately computationally inefficient.
On the other hand, conservative methods are more computationally efficient, but might mistakenly flag a ``safe'' data point as vulnerable to adversarial examples \citep{wong2018provable, wang2018mixtrain, wang2018efficient, raghunathan2018certified, raghunathan2018semidefinite, wong2018scaling, dvijotham2018dual, dvijotham2018training, croce2018provable, salman2019convex, gehr2018ai2, mirman2018differentiable, singh2018fast, gowal2018effectiveness, weng2018towards, zhang2018efficient}.
However, none of these defenses scale to practical networks.
Recently, a new method called randomized smoothing has been proposed as a \textit{probabilistically} certified defense, whose architecture-independence makes it scalable.

\footnotetext{Unless stated otherwise, these models were trained with noise augmentation. In our replication of \citet{teng_ell_1_2019}, our noise augmentation results matched their adversarial training results.}

\paragraph{Randomized smoothing}

Randomized smoothing was first proposed as a heuristic defense without any guarantees~\citep{liu2018towards, cao2017mitigating}.
Later on, \citet{lecuyer2018certified} proved a robustness guarantee for smoothed classifiers from a differential privacy perspective.
Subsequently, \citet{li_certified_2019} gave a stronger robustness guarantee utilizing tools from information theory.
Recently, \citet{cohen_certified_2019} provided a tight $\ell_2$ robustness guarantee for randomized smoothing, applied by \citet{salman2020blackbox} to provably defend pre-trained models for the first time.
Furthermore, a series of papers came out recently that developed robustness guarantees against other adversaries such as $\ell_1$-bounded \citep{teng_ell_1_2019}, $\ell_\infty$-bounded~\cite{zhang*2020filling}, $\ell_0$-bounded~\cite{levine_robustness_2019,lee_tight_2019}, and Wasserstein attacks~\cite{levine_wasserstein_2019}.
In \cref{sec:main-comp}, we give a more in-depth comparison on how our techniques compare to their results.

\paragraph{Wulff Crystal}
We are the first to relate to adversarial robustness the theory of \emph{Wulff Crystals}.
Just as the round soap bubble minimizes surface tension for a given volume, the Wulff Crystal minimizes certain similar surface energy that arises when the crystal interfaces with another material.
The Russian physicist George Wulff first proposed this shape via physical arguments in 1901 \citep{wulffZurFrageGeschwindigkeit1901}, but its energy minimization property was not proven in full generality until relatively recently, building on a century worth of work \citep{gibbsEquilibriumHeterogeneousSubstances1875,wulffZurFrageGeschwindigkeit1901,hiltonMathematicalCrystallography1903,liebmannCurieWulffScheSatz1914,vonlaueWulffscheSatzFur,dinghasUberGcometrischenSatz1944,burtonGrowthCrystalsEquilibrium1951,herringKonferenzUberStruktur,constableKineticsMechanismCrystallization1968GoogleBooks,taylorUniqueStructureSolutions1975,taylorCrystallineVariationalProblems1978ProjectEuclid,fonseca_muller_1991,brothersIsoperimetricTheoremGeneral1994DOI.orgCrossref,cerfWulffCrystalIsing2006www.springer.com}.

\paragraph{No-go theorems for randomized smoothing} Prior to the initial submission of this manuscript, the only other no-go theorem for randomized smoothing in the context of adversarial robustness is~\citet{zheng2020a}.
However, they are only concerned with a non-standard notion of certified robustness that does not imply anything for the original problem.
Moreover, they show that, under this different notion of robustness, if they are robust for $\ell_\infty$, then the $\ell_2$ norm of the noise must be large on average.
While this provides indirect evidence for the hardness of certifying $\ell_\infty$, it does not actually address the question.
Our result, on the other hand, directly rules out a large suite of current techniques for deriving robust certificates for all $\ell_p$ norms for $p > 2$, for the standard notion of certified robustness.

After the initial submission of this manuscript, we became aware of two concurrent works~\cite{blum2020random,kumar2020curse} that claim impossibility results for randomized smoothing.
\citet{blum2020random} demonstrate that, under some mild conditions, any smoothing distribution for $\ell_p$ with $p > 2$ must have large component-wise magnitude.
This gives indirect evidence for the hardness of the problem, but does not directly show a limit for the utility of randomized smoothness for the robust classification problem, which we do in this work.
\citet{kumar2020curse} demonstrate that certain classes of smoothing distributions cannot certify $\ell_\infty$ without losing dimension-dependent factors.
Our result is more general, as it rules out \emph{any} class of smoothing distributions, and in fact, any smoothing scheme that allows the distribution to vary arbitrarily with the input point.

\section{Randomized Smoothing}

Consider a classifier $f$ from $\mathbb{R}^d$ to classes $\mathcal{Y}$ and a distribution $q$ on $\R^d$.
Randomized smoothing with $q$ is a method that constructs a new, \textit{smoothed} classifier $g$ from the \textit{base} classifier $f$.
The smoothed classifier $g$ assigns to a query point $x$ the class which is most likely to be returned by the base classifier $f$ when $x$ is perturbed by a random noise sampled from $q$, i.e.,
\begin{align}
g(x) &\defeq \argmax_{c \in \mathcal{Y}} q(U_c - x)
\label{eq:smoothed-hard}
\end{align}
where $U_c$ is the decision region $\{x' \in \R^d: f(x') = c\}$, $U_c - x$ denotes the translation of $U_c$ by $-x$, and $q(U)$ is the measure of $U$ under $q$, i.e.\ $q(U) = \Pr_{\delta \sim q}(\delta \in U)$.

\paragraph{Robustness guarantee for smoothed classifiers}

For $p \in [0, 1], v \in \R^d,$ define the \emph{growth function}
\begin{align*}
    \growth_q(p, v) \defeq
    \sup_{U \sbe \R^d: q(U) = p} q(U - v),
    \numberthis\label{eqn:growth}
\end{align*}
One can think of $U$ has the decision region of some base classifier.
Thus $\growth_q(p, v)$ gives the maximal growth of measure of a set (i.e. decision region) when $q$ is shifted by the vector $v$, if we only know the initial measure $p$ of the set.

Consider an adversary that can perturb an input additively by any vector $v$ inside an allowed set $\mathcal B$.
In the case when $\mathcal B$ is the $\ell_2$ ball and $q$ is the Gaussian measure, \citet{cohen_certified_2019} gave a simple expression for $\growth_q$ involving the Gaussian CDF, derived via the Neyman-Pearson lemma, which is later rederived by \citet{salman2019provably} as a nonlinear Lipschitz property.
Likewise, the expression for Laplace distributions was derived by \citet{teng_ell_1_2019}.
(See \cref{thm:gaussianl2} and \cref{thm:laplaceL1} for their expressions.)

Suppose when the base classifier $f$ classifies $x + \delta$, $\delta \sim q$, the class $c \in \mathcal Y$ is returned with probability $\rho = \Pr_{\delta \sim q} (f(x + \delta) = c) > 1/2$.
Then the smoothed classifier $g$ will not change its prediction under the adversary's perturbations if~\footnote{Many earlier works state robustness guarantees in terms of estimates of $p_A = \rho$ of the top class and $p_B$ of the runner up class; however, their implementations are all in the form provided here, as $p_B$ is usually taken to be $1 - p_A$.}
\begin{equation}
    \sup_{v \in \mathcal B} \growth_q(1-\rho, v) < 1/2.
    \label{eqn:robustguarantee}
\end{equation}

\section{Methods for Deriving Robust Radii}
Let $q$ be a distribution with a density function, and we shall write $q(x), x \in \R^d$, for the value of the density function on $x$.
Then, given a shift vector $v \in \R^d$ and a ratio $\kappa > 0$, define the \emph{Neyman-Pearson set}
\begin{equation}
    \npset_\kappa \defeq \{x \in \R^d: \kappa q(x-v) \ge q(x)\}.
    \label{eqn:neymanpearsonset}
\end{equation}
Then the Neyman-Pearson lemma tells us that \cite{neyman_ix_1933,cohen_certified_2019}
\begin{align}
    \growth_q(q(\npset_\kappa), v) = q(\npset_\kappa - v).
    \label{eqn:NP}
    \tag{NP}
\end{align}
While this gives way to a simple expression for the growth function when $q$ is Gaussian \cite{cohen_certified_2019}, it is difficult for more general distributions as the geometry of $\npset_\kappa$ becomes hard to grasp.
To overcome this difficulty, we propose the \emph{level set method} that  decomposes this geometry so as to compute the growth function exactly, and the \emph{differential method} that upper bounds the growth function derivative, loosely speaking.

\subsection{The Level Set Method}

For each $t > 0$, let $U_t$ be the superlevel set
\begin{align*}
    U_t \defeq \{x \in \R^d: q(x) \ge t\}.
\end{align*}
Then its boundary $\pd U_t$ is the level set with $q(x) = t$ under regularity assumptions.
The integral of $q$'s density is of course 1, but this integral can be expressed as the integral of the volumes of its superlevel sets:
\begin{equation}
    1 = \int q(x) \dd x = \int_0^\infty \Vol(U_t) \dd t.
    \label{eqn:suplevelint}
    \tag{$\bullet$}
\end{equation}
If $q$ has a differentiable density, then we may rewrite this as an integral of \emph{level} sets (\cref{thm:coarea}):
\begin{equation}
    1 = \int_0^\infty \int_{\pd U_t} \f t{\|\nabla q(x)\|_2} \dd x \dd t.
    \label{eqn:levelint}
    \tag{$\circ$}
\end{equation}
\begin{center}
    \includegraphics[width=0.35\textwidth]{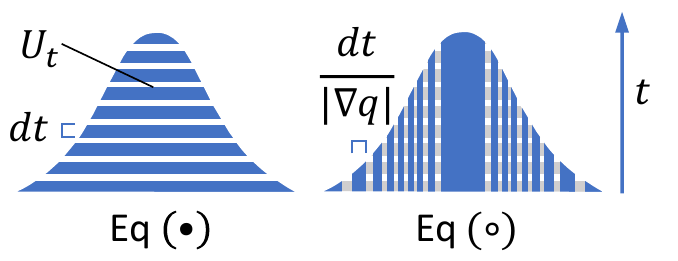}
\end{center}
The graphics above illustrate the two integral expressions (best viewed on screen).
In this level set perspective, the Neyman-Pearson set $\npset_\kappa$ (\cref{eqn:neymanpearsonset}) can be written as
\begin{align*}
    \npset_\kappa &=
        \bigcup_{t > 0} \{x: q(x) = t \text{  and  } q(x - v) \ge t / \kappa\}\\
    &=
    \bigcup_{t > 0} \{\pd U_t \cap (U_{t / \kappa} + v)\}.
\end{align*}
Then naturally, its measure is calculated by
\begin{align}
    q(\npset_\kappa)
    &=
        \int_0^\infty \int_{\pd U_t \cap (U_{t / \kappa} + v)}\f t{\|\nabla q(x)\|_2} \dd x \dd t.
    \label{eqn:psmall}
    \tag{$\vee$}
\end{align}
Similarly, the Neyman-Pearson set can also be written from the perspective of $q(\cdot - v)$,
\begin{align*}
    \npset_\kappa
    &=
        \bigcup_{t > 0}\{x : q(x - v) = t \text{  and  } q(x) \le t \kappa \}\\
    &=
        \bigcup_{t > 0} \{ (\pd U_t + v) \setminus \mathring U_{t\kappa}\},
\end{align*}
where $\mathring U$ is the interior of the closed set $U$.
So its measure under $q(\cdot - v)$ is
\begin{align}
    q(\npset_\kappa - v)
    &=
        \!\int_0^\infty\!\!\!\! \int_{\pd U_t \setminus (\mathring U_{t\kappa} - v)}
         \f{t}{\|\nabla q(x)\|_2} \dd x \dd t.
    \label{eqn:pbig}
    \tag{$\wedge$}
\end{align}
\begin{center}
    \includegraphics[width=0.45\textwidth]{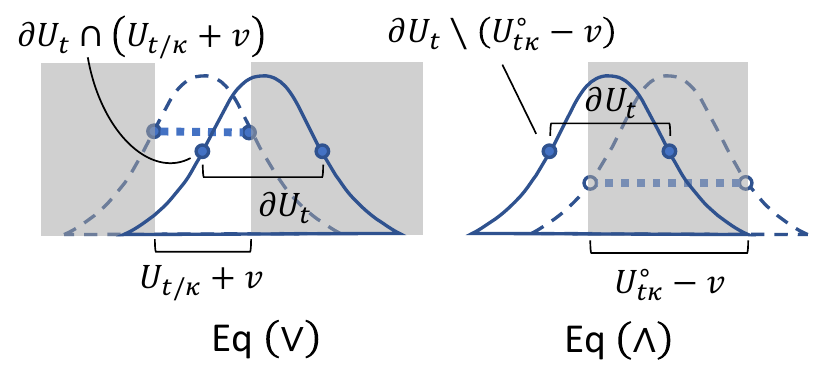}
\end{center}
The graphics above illustrate the integration domains of $x$ in \cref{eqn:psmall,eqn:pbig}.
In general, the geometry of $\pd U_t \cap (U_{t/\kappa} + v)$ or $\pd U_t \setminus (\mathring U_{t\kappa} - v)$ is still difficult to handle, but in highly symmetric cases when $U_t$ are concentric balls or cubes, \cref{eqn:psmall,eqn:pbig} can be calculated efficiently.

\paragraph{Computing Robust Radius}
\cref{eqn:psmall,eqn:pbig} allow us to compute the growth function by \cref{eqn:NP}.
In general, this yields an \emph{upper bound} of the robust radius
\begin{align*}
        &\phantomeq
            \sup \left\{r: \sup_{\|v\|_p \le r} \growth_q(1-\rho, v) < 1/2\right\}\\
        &\le
            \sup \left\{r: \growth_q(1-\rho, r u) < 1/2\right\}
\end{align*}
for any particular $u$ with $\|u\|_p = 1$.
With sufficient symmetry, e.g. with $\ell_2$ adversary and distributions with spherical level sets, this upper bound becomes \emph{tight} for well-chosen $u$, and we can build a lookup table of certified radii.
See \cref{alg:levelsettable,alg:levelsetcert}.

\begin{algorithm}[tb]
   \caption{Pre-Computing Robust Radius Table
   via Level Set Method for Spherical Distributions Againt $\ell_2$ Adversary}
   \label{alg:levelsettable}
\begin{algorithmic}
   \STATE {\bfseries Input:} Radii $r_1 < \ldots < r_N$
   \STATE Initialize $u=(1, 0, \ldots, 0) \in \R^d$.
   \FOR{$i=1$ {\bfseries to} $N$}
   \STATE Find $\kappa$ s.t.\ $q(\npset_\kappa-r_i u) = 1/2$ (via \cref{eqn:pbig} or \cref{thm:levelsetl2}) by binary search
   \STATE Compute $p_i \gets q(\npset_\kappa)$ via \cref{eqn:psmall} or \cref{thm:levelsetl2}
   \ENDFOR
   \STATE {\bfseries Output:} $p_1 > \cdots > p_N$
\end{algorithmic}
\end{algorithm}

\begin{algorithm}[tb]
    \caption{Certification with Table}
    \label{alg:levelsetcert}
\begin{algorithmic}
   \STATE {\bfseries Input:} Probability of correct class $\rho$
   \STATE {\bfseries Output:} Look up $r_i$ where $p_i \ge 1-\rho > p_{i+1}$
\end{algorithmic}
\end{algorithm}

\begin{figure*}
    \centering
    \includegraphics[width=\linewidth]{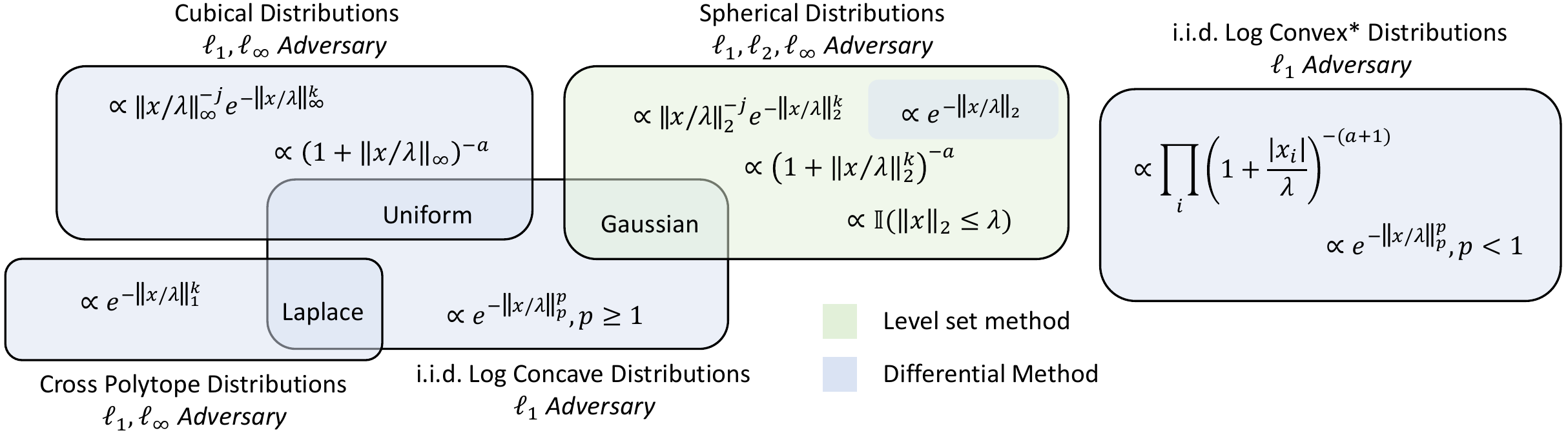}
    \caption{\textbf{Smoothing distributions for which we derive robustness guarantees in this paper.}
    Each box represents a family of distributions that obtain guarantees through similar proofs.
    Text beside each box indicates the name of the family and the $\ell_p$ adversaries against which we have guarantees.
    \emph{Log Convex*} means log convex on the positive and negative half lines, but not necessarily on the whole line.
    The color indicates the basic technique used, among the two proposed techniques in this paper.
    We explicitly list example densities in each box.
    For the robust radii formulas, see \cref{tab:bounds}.
    }
    \label{fig:distribs}
\end{figure*}

\subsection{The Differential Method}

To derive certification (robust radius \emph{lower bounds}) for more general distributions, we propose a \emph{differential method}, which can be thought of as a vast generalization of the proof in \citet{salman2019provably} of the Gaussian robust radius.
The idea is to compute the largest possible \emph{infinitesimal increase in $q$-measure} due to an \emph{infinitesimal adversarial perturbation}.
More precisely, given a norm $\| \cdot \|$, and a smoothing measure $q$, we define
\begin{align}
    \Phi(p) \defeq \sup_{\|v\|=1} \sup_{U \sbe \R^d: q(U) = p} \lim_{r \searrow 0} \f{q(U-rv) - p}r.
    \label{eqn:PhiMaintext}
\end{align}
Intuitively, one can then think of $1/\Phi(p)$ as the \emph{smallest} possible perturbation in $\|\cdot\|$ needed to effect a unit of infinitesimal increase in $p$.
Therefore,
\begin{thm}[\cref{thm:differentialMethodMasterTheorem}]\label{thm:differentialMethodMainText}
The robust radius in $\| \cdot \|$ is at least
\begin{align*}
    R \defeq \int_{1-\rho}^{1/2} \f 1 {\Phi(p)} \dd p,
\end{align*}
where $\rho$ is the probability that the base classifier predicts the right label under random perturbation by $q$.
\end{thm}
By exchanging differentiation and integration and applying a similar greedy reasoning as in the Neyman-Pearson lemma, $\Phi(p)$ can be derived for many distributions $q$ and integrated symbolically to obtain expressions for $R$.
We demonstrate the technique with a simple example below, but much of it can be automated; see \cref{thm:differentialMethodMasterTheorem}.

\begin{exmp}[see \cref{thm:ExpInfL1}]\label{exmp:explinf}
If the smoothing distribution is $q(x) \propto \exp(-\|x\|_\infty/\lambda)$, then the robust radius against an $\ell_1$ adversary is at least
\[R = 2d \lambda (\rho - 1/2),\]
when $\rho$ is the probability of the correct class as in \cref{thm:differentialMethodMainText}.
\end{exmp}
\begin{proof}[Proof Sketch]
By linearity in $\lambda$, we WLOG assume $\lambda = 1$.
By \cref{thm:differentialMethodMainText} and the monotonicity of $\Phi$, it suffices to show that $\Phi(p) = 1/2d$ for $p \ge 1/2d.$
For any fixed $U$ with $q(U) = p$,
\begin{align*}
    \lim_{r \searrow 0} \f{q(U-rv) - p}r
    &= \f{d}{dr} \left.\int_U q(x-rv) \dd x\right|_{r=0}
        \\
    &= \int_U \la v, \nabla q(x)\ra \dd x.
\end{align*}
Note $\nabla q(x) = e_x q(x)$, where $e_x = \sgn(x_{i^*}) e_{i^*}$, $e_i$ is the $i$th unit vector, and $i^* = \argmax_i |x_i|$.
Additionally, the above integral is linear in $v$, so the supremum over $\|v\|_1 = 1$ is achieved on one of the vertices of the $\ell_1$ ball.
So we may WLOG consider only $v = \pm e_i$; furthermore, due to symmetry of $\nabla q(x)$, we can just assume $v = e_1$:
\begin{align*}
    \Phi(p)
        &=
            \sup_{U} \lim_{r \searrow 0} \f{q(U-re_1) - p}r
        =
            \sup_{U}
            \int_U \la e_1, e_x\ra q(x) \dd x,
\end{align*}
where $U$ ranges over all $q(U) = p$.
Note $\la e_1, e_x \ra = 0$ if $i^* \ne 1$, and $\sgn(x_{i^*})$ otherwise.
Thus, to maximize $\lim_{r \searrow 0} \f{q(U-re_1) - p}r$ subject to the constraint that $q(U) = p$, we should put as much $q$-mass on those $x$ with large $\la e_1, e_x \ra$.
For $p \ge 1/2d$, we thus should occupy the entire region $\{x: \la e_1, e_x\ra = 1\}$, which has $q$-mass $1/2d$, and then assign the rest of the $q$-mass (amounting to $p - 1/2d$) to the region $\{x: \la e_1, e_x\ra = 0\}$, which has $q$-mass $1 - 1/d$.
This shows that
\[\Phi(p) = 1/2d, \quad\forall p \in [1/2d, 1-1/2d]\]
as desired.
\end{proof}

\subsection{Comparison of the Two Methods and Prior Works}
\label{sec:main-comp}

We summarize the distributions our methods cover in \cref{fig:distribs} and the bounds we derive in \cref{tab:bounds}.
We highlight a few broadly applicable robustness guarantees:
\begin{exmp}[\cref{thm:iidlogconcaveL1}]\label{exmp:iidlogconcaveL1}
Let $\phi: \R \to \R$ be convex and even, and let  $\mathrm{CDF}_\phi^{-1}$ be the inverse CDF of the 1D random variable with density $\propto \exp(-\phi(x))$.
If $q(x) \propto \prod_i e^{-\phi(x_i)}$, and $\rho$ is the probability of the correct class, then the robust radius in $\ell_1$ is
\begin{align*}
    R = \mathrm{CDF}_\phi^{-1}(\rho)
\end{align*}
and this radius is \emph{tight}.
This in particular recovers the Gaussian bound of \citet{cohen_certified_2019}, Laplace bound of \citet{teng_ell_1_2019}, and Uniform bound of \citet{lee_tight_2019} in the setting of $\ell_1$ adversary.
\end{exmp}

\begin{exmp}[\cref{sec:cubicalExponentialL1,sec:cubicalPowerL1}]
Facing an $\ell_1$ adversary, cubical distributions, like that in \cref{exmp:explinf}, typically enjoy, via the differential method, $\ell_1$ robust radii of the form
\begin{align*}
    R = c(\rho - 1/2)
\end{align*}
for some constant $c$ depending on the distribution.
\end{exmp}

In general, the level set method always gives certificate as tight as Neyman-Pearson, while the differential method is tight only for infinitesimal perturbations, but can be shown to be tight for certain families, like in \cref{exmp:iidlogconcaveL1} above.
On the other hand, the latter will often give efficiently evaluable symbolic expressions and apply to more general distributions, while the former in general will only yield a table of robust radii, and only for distributions whose level sets are sufficiently symmetric (such as a sphere or cube).

For distributions that are covered by both methods, we compare the bounds obtained and note that the differential and level set methods yield almost identical robustness certificates in high dimensions (e.g. number of pixels in CIFAR-10 or ImageNet images).
See \cref{sec:levelset_vs_differential}.

Many earlier works used differential privacy or $f$-divergence methods to compute robust radii of smoothed models \citep{lecuyer2018certified,li_certified_2019,dvijotham_framework_2019}.
In particular, \citet{dvijotham_framework_2019} proposed a general $f$-divergence framework that subsumed all such works.
Our robust radii are computed only from $\rho$; \citeauthor{dvijotham_framework_2019} called this the ``information-limited'' setting, and we shall compare with their robustness guarantees of this type.
While their algorithm in a certain limit becomes as good as Neyman-Pearson, in practice outside the Gaussian distribution, their robust radii are too loose.
This is evident by comparing our baseline Laplace results in \cref{tab:sota} with theirs, which are trained the same way.
Additionally, our differential method often yields symbolic expressions for robust radii, making the certification algorithm easy to implement, verify, and run.
Moreover, we derive robustness guarantees for many more (distributions, adversary) pairs (\cref{fig:distribs,tab:bounds}).
See \cref{sec:fdivCompare} for a more detailed comparison.

\section{Wulff Crystals}
\label{sec:wulffMain}
A priori, it is a daunting task to understand the relationship between the adversary $\mathcal B$ and the smoothing distribution $q$.
In this section, we shall begin our investigation by looking at uniform distributions, and then end with an optimality theorem for all ``reasonable'' distributions.

Let $q$ be the uniform distribution supported on a measurable set $S \sbe \R^d$.
WLOG, assume $S$ has (Lebesgue) volume 1, $\Vol(S) = 1$.
Then for any $v \in \R^d$ and any $p \in [0, 1]$,
\begin{align*}
    \growth_q(p, v) = \min\left(1, p + {\Vol((S + v) \setminus S)}\right).
\end{align*}
\begin{wrapfigure}{r}{0.18\textwidth}
\begin{center}
    \includegraphics[width=0.15\textwidth]{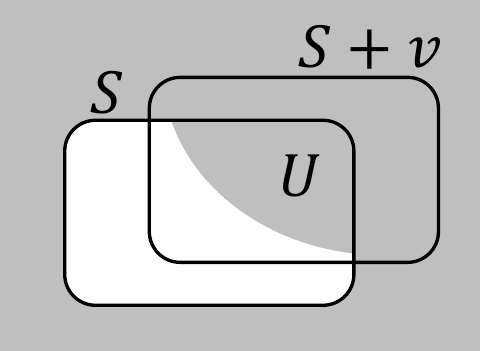}
\end{center}
\end{wrapfigure}
This can be seen easily by taking $U$ in \cref{eqn:growth} to be a subset of $(S + v) \cap S$ with volume $p$ (or any set of volume $p$ containing $(S + v) \cap S$ if $p \ge \Vol((S + v) \cap S)$) unioned with the complement of $S$.
For example, in the figure here, $U$ would be the gray region, if $U \cap S$ has volume $p$.

If $S$ is convex, and we take $v$ to be an infinitesimal translation, then the RHS above is infinitesimally larger than $p$, as follows:
\begin{align*}
    \lim_{r \to 0} \f{\growth_q(p, rv) - p}r &= \lim_{r \to 0} \f{\Vol((S + r v) \setminus S)}r\\
        &= \|v\|_2 \Vol(\Pi_{v} S)
        \numberthis \label{eqn:differentialUniform}
\end{align*}
\begin{wrapfigure}{r}{0.2\textwidth}
\vspace{-20pt}
\begin{center}
    \includegraphics[width=0.2\textwidth]{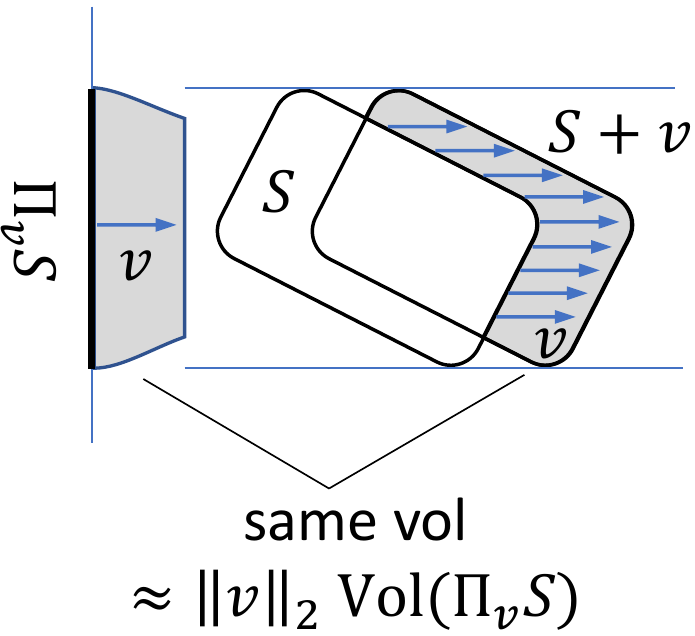}
\end{center}
\vspace{-10pt}
\end{wrapfigure}
where $\Pi_{v} S$ is the projection of $S$ along the direction $v/\|v\|_2$, and $\Vol(\Pi_{v} S)$ is its $(d-1)$-dimensional Lebesgue measure.
A similar formula holds when $S$ is not convex as well (\cref{eqn:leakage}).
In the context of randomized smoothing, this means that the classifier $g$ smoothed by $q$ is robust at $x$ under a perturbation
$\f{\f 1 2 - p}{\|v\|_2 \Vol(\Pi_{v} S)} v$
when $1/2 - p$ is small, and $p$ is the probability the base classifier $f$ \emph{mis}classifies $x + \delta$, $\delta \sim q$.
Thus, for $r$ small, we have
\begin{align*}
    \sup_{v \in r \mathcal B} \growth_q(p, v) \approx p + r \sup_{v \in \mathcal B} \|v\|_2 \Vol(\Pi_{v} S) = p + r\Phi(p),
\end{align*}
with $\Phi$ as in \cref{eqn:PhiMaintext}.
The smaller $\sup_{v \in \mathcal B} \|v\|_2 \Vol(\Pi_{v} S)$ is, the more robust the smoothed classifier $g$ is, for a fixed $p$.
A natural question, then, is: among convex sets of volume 1,
\begin{align*}
    \text{which set $S$ minimizes $\Phi = \sup_{v \in \mathcal B} \|v\|_2 \Vol(\Pi_{v} S)$?}
    \label{eqn:minquestion}
\end{align*}

If $\mathcal B$ is the $\ell_p$ ball, the reader might guess $S$ should either be the $\ell_p$ ball or the $\ell_r$ ball with $\f 1 r + \f 1 p = 1$.
It turns out the correct answer, at least in the case when $\mathcal B$ is a highly symmetric polytope (e.g.\ $\ell_1, \ell_2, \ell_\infty$ balls), is a kind of \emph{energy-minimizing} crystals studied in physics since 1901 (\citeauthor{wulffZurFrageGeschwindigkeit1901}).
\begin{defn}\label{defn:wulff}
The \emph{Wulff Crystal} (w.r.t.\ $\mathcal B$) is defined as the unit ball of the norm dual to $\|\cdot\|_*$, where $\|x\|_* = \EV_{y \sim \Vrt(\mathcal B)} |\la x, y \ra|$ and $y$ is sampled uniformly from the vertices of $\mathcal B$~\footnote{When $\mathcal B$ is the $\ell_2$ ball, $\Vrt(\mathcal B)$ is the entire boundary.}.
\end{defn}
In fact, Wulff Crystals solve the more general problem without convexity constraint.
\begin{thm}[\cref{thm:wulffOptimalFull}, informal]\label{thm:wulffOptimalMainText}
The Wulff Crystal w.r.t.\ $\mathcal B$ minimizes
\[\Phi = \sup_{v \in \mathcal B} \lim_{r \to 0} \inv r \Vol((S + rv) \setminus S)\]
among all measurable (not necessarily convex) sets $S$ of the same volume, when $\mathcal B$ is sufficiently symmetric (e.g. $\ell_1, \ell_2,\ell_\infty$ balls).
\end{thm}
When $\Vrt(\mathcal B)$ is a finite set, the Wulff Crystal has an elegant description as the \emph{zonotope} of $\Vrt(\mathcal B)$, i.e.\ the Minkowski sum of the vertices of $\mathcal B$ as vectors (\cref{prop:wulffzonotope}), from which we can derive the following examples.
\begin{exmp}
The Wulff Crystal w.r.t. $\ell_2$ ball is the $\ell_2$ ball itself.
The Wulff Crystal w.r.t. $\ell_1$ ball is a cube ($\ell_\infty$ ball).
The Wulff Crystal w.r.t. $\ell_\infty$ in 2 dimensions is a rhombus; in 3 dimensions, it is a rhombic dodecahedron; in higher dimension $d$, there is no simpler description of it other than the zonotope of the vectors $\{\pm1\}^d$.
\end{exmp}

\begin{center}
    \includegraphics[width=0.3\textwidth]{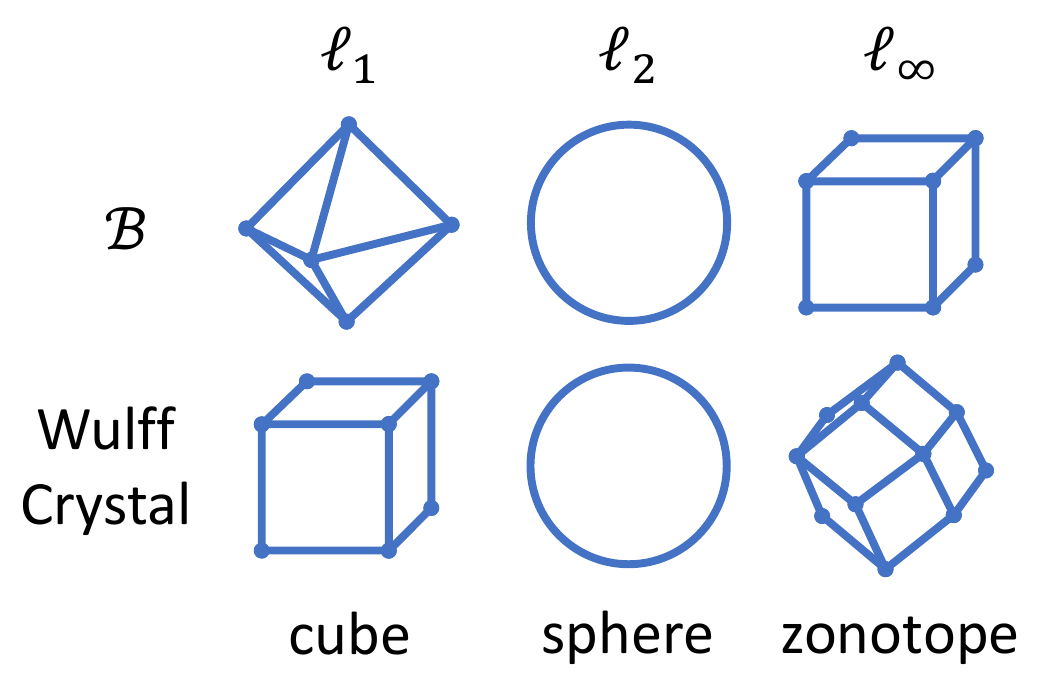}
\end{center}

In fact, distributions with Wulff Crystal level sets more generally maximizes the robust radii for ``hard'' inputs.
\begin{thm}[\cref{thm:wulffOptimalGeneral}, informal]
Let $\mathcal B$ be sufficiently symmetric.
Let $q_0$ be any distribution with a ``reasonable''\footnote{\emph{Reasonable} here roughly means Sobolev, i.e. has weak derivative that is integrable, and this can be further relaxed to \emph{bounded variations}; for details see \cref{thm:wulffOptimalGeneral} and \cref{thm:wulffOptimalGeneralBV}.} and even density function.
Among all ``reasonable'' and even density functions $q$ whose superlevel sets $\{x: q(x) \ge t\}$ have the same volumes as those of $q_0$, the quantity
\[\Phi(1/2) = \sup_{v \in \mathcal B} \sup_{q(U) = 1/2} \lim_{r \searrow 0} \f{q(U-rv) - 1/2}r\]
is minimized by the unique distribution $q^*$ whose superlevel sets are proportional to the Wulff Crystal w.r.t.\ $\mathcal B$.
\end{thm}
This theorem implies that distributions with Wulff Crystal level sets give the best robust radii for those \emph{hard} inputs $x$ that a smooth classifier classifies correctly but only barely, in that the probability of the correct class $\rho = 1/2 + \epsilon$ for some small $\epsilon$.
The constraint on the volumes of superlevel sets indirectly controls the variance of the distribution.
While this theorem says nothing about the robust radii for $\rho$ away from $1/2$, we find the Wulff Crystal distributions empirically to be highly effective, as we describe next in \cref{sec:experiments}.

\section{Experiments}
\label{sec:experiments}

We empirically study the performance of different smoothing distributions on image classification datasets, using the bounds derived via the level set or the differential method, and verify predictions made by the Wulff Crystal theory.
We follow the experimental procedure in \citet{cohen_certified_2019} and further works on randomized smoothing \citep{salman2019provably,li_certified_2019,Zhai2020MACER} using ImageNet \citep{deng_imagenet_2009} and CIFAR-10 \citep{krizhevsky_learning_2009}.

The certified accuracy at a radius $\epsilon$ is defined as the fraction of the test set for which the smoothed classifier $g$ correctly classifies and certifies robust at an $\ell_p$ radius of $\epsilon$.
All results were certified with $N=100,000$ samples and failure probability $\alpha=0.001$.
For each distribution $q$, we train models across a range of scale parameter $\lambda$ (see\ \cref{tab:bounds}), corresponding to the same range of noise variances $\sigma^2 \defeq \EV_{\delta \sim q}[\frac{1}{d}\|\delta\|^2_2]$ across different distributions.
Then we calculate for each model the certified accuracies across the range of considered $\epsilon$.
Finally, in our plots, we present, for each distribution, the upper envelopes of certified accuracies attained over the range of considered $\sigma$.
Further details of experimental procedures are described in \cref{sec:experiments_details}.

We focus on the effect of the noise distribution in this section and only train models with noise augmentation.
In \cref{sec:experiments_details} we also study (1) stability training, and (2) the use of more data through (a) pre-training on downsampled ImageNet \citep{hendrycks_using_2019} and (b) semi-supervised self-training with data from 80 Million Tiny Images \citep{carmon_unlabeled_2019}. As shown in Table \ref{tab:sota}, these techniques further improve upon our results in this section.

\subsection{\texorpdfstring{$\ell_1$ Adversary}{L1 Adversary}}

\begin{figure}[t]
    \centering
    \caption{\textbf{SOTA $\ell_1$ Certified Accuracies.} Certified $\ell_1$ top-1 accuracies for ImageNet (left) and CIFAR-10 (right). For each distribution $q$, we train models across a range of $\sigma^2 \defeq \EV_{\delta \sim q}[\frac{1}{d}\|\delta\|_2^2]$, and at each level of $\ell_1$ adversarial perturbation radius $\epsilon$ we report the best certified accuracy.}
    \begin{subfigure}[t]{0.49\linewidth}
        \caption{ImageNet}
        \includegraphics[width=\linewidth]{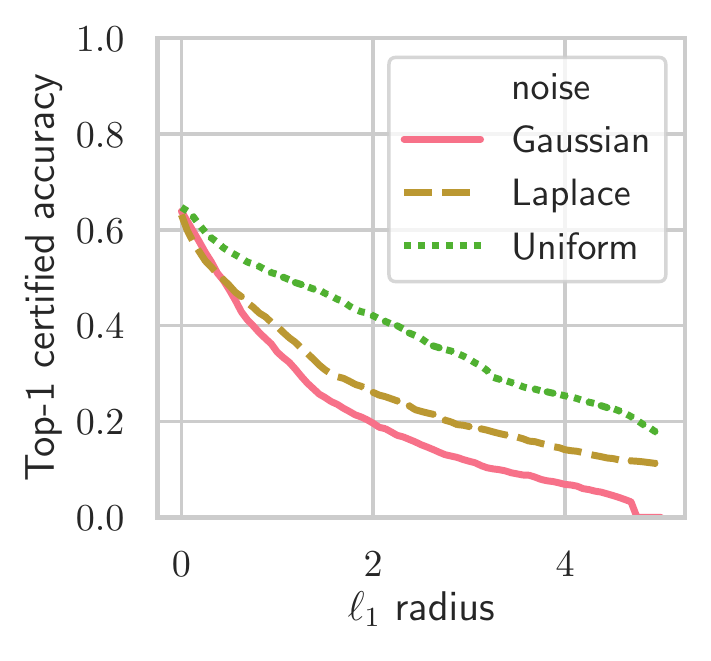}
    \end{subfigure}
    \begin{subfigure}[t]{0.49\linewidth}
        \caption{CIFAR-10}
        \includegraphics[width=\linewidth]{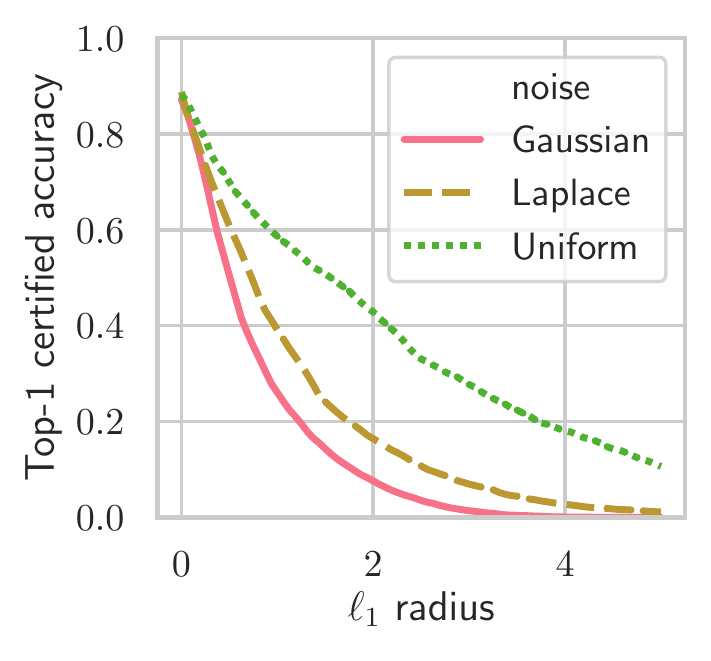}
    \end{subfigure}
    \label{fig:l1_certificates}
\end{figure}

As previously mentioned, the Wulff Crystal for the $\ell_1$ ball is a cube.
With this motivation, we explore certified accuracies attained by distributions with cubical level sets.
\begin{enumerate}[nosep]
    \item Uniform, $\propto \ind(\|x\|_\infty \le \lambda)$
    \item Exponential, $\propto \|x\|_\infty^{-j} e^{-\|x/\lambda\|^k_\infty}$
    \item Power law, $\propto (1 + \|x/\lambda\|_\infty)^{-a}$
\end{enumerate}
We compare to previous state-of-the-art approaches using the Gaussian and Laplace distributions, as well as new non-cubical distributions.
\begin{enumerate}[nosep]
    \setcounter{enumi}{3}
    \item Exponential $\ell_1$ (non-cubical), $\propto \|x\|_1^{-j} e^{-\|x/\lambda\|_1^k}$
    \item Pareto i.i.d. (non-cubical), $\propto \prod_i (1 + |x_i|/\lambda)^{-a}.$
\end{enumerate}
The relevant certified bounds are given in \cref{tab:bounds}.

We obtain state-of-the-art robust certificates for ImageNet and CIFAR-10, finding that the Uniform distribution performs best, significantly better than the Gaussian and Laplace distributions (\cref{tab:sota}, \cref{fig:l1_certificates}).
The other distributions with cubic level sets match but do not exceed the performance of Uniform distribution, after sweeping hyper-parameters.
This verifies that distributions with cubical level sets are significantly better for $\ell_1$ certified accuracy than those with spherical or cross-polytope level sets.
See results for other distributions in \cref{sec:additional_results}.

\subsection{\texorpdfstring{$\ell_2$ Adversary}{L2 Adversary}}

\begin{figure}[t]
    \centering
    \caption{CIFAR-10 certified accuracies for $\ell_2$ (left) and $\ell_\infty$ (right) adversaries. For each distribution $q$ we train models across a range of $\sigma^2 \defeq \mathbb{E}[\frac{1}{d}\|\delta\|_2^2]$, and at each level of $\ell_p$ adversarial perturbation radius $\epsilon$, we pick the model that maximizes certified accuracy.}
    \begin{subfigure}[t]{0.49\linewidth}
        \includegraphics[width=\linewidth]{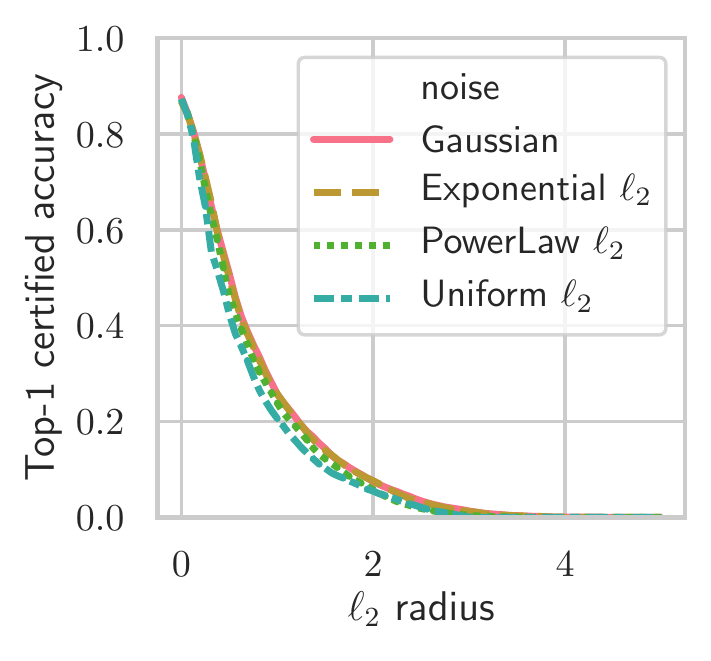}
    \end{subfigure}
    \begin{subfigure}[t]{0.49\linewidth}
        \includegraphics[width=\linewidth]{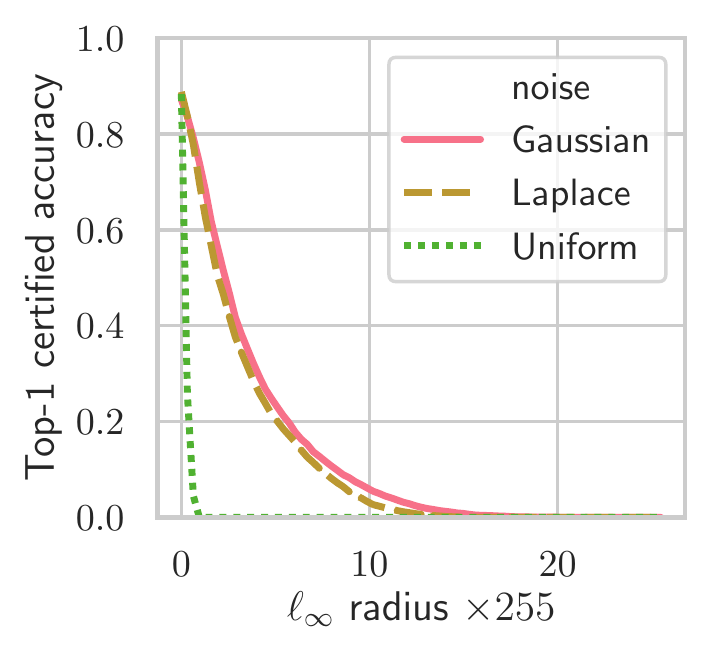}
    \end{subfigure}
    \label{fig:l2_certificates}
    \vspace{-2em}
\end{figure}

The Wulff Crystal w.r.t.\ the $\ell_2$ ball is a sphere, so we explore distributions with spherical level sets (\cref{tab:bounds}):
 \begin{enumerate}[nosep]
     \item Uniform, $\propto \ind(\|x\|_2 \le \lambda)$
     \item Exponential, $\propto\|x\|_2^{-j} e^{-\|x/\lambda\|_2^k}$
     \item Power law, $\propto (1 + \|x/\lambda\|_2)^{-a}$
 \end{enumerate}
We find these distributions perform similarly to, though do not surpass the Gaussian (\cref{fig:l2_certificates}, left).

\subsection{\texorpdfstring{$\ell_\infty$ Adversary}{Linf Adversary}}

The Wulff Crystal for the $\ell_\infty$ ball is the zonotope of vectors $\{\pm 1\}^d$, which is a highly complex polytope hard to sample from and related to many open problems in polytope theory \citep{ziegler_lectures_1995}.
However, we can note that it is approximated by a sphere with constant ratio (\cref{prop:linfWulffCrystalBasic}), and in high dimension $d$, the sphere gets closer and closer to minimizing $\Phi$ (\cref{thm:wulffOptimalMainText}), but the cube and the cross polytope do not (\cref{claim:linfSphereOptimal}).
Accordingly, we find that distributions with spherical level sets outperform those with cubical or cross polytope level sets in certifying $\ell_\infty$ robustness (\cref{fig:l2_certificates}, right).
In fact, in the next section we show that up to a dimension-independent factor, the Gaussian distribution is optimal for defending against $\ell_\infty$ adversary if we don't use a more powerful technique than Neyman-Pearson.

\section{No-Go Results for Randomized Smoothing}

Recall that given a smoothing distribution $q$, a point $x \in \R^d$, and a binary base classifier $U \sbe \R^d$ (identified wth its decision region), the smoothed classifier outputs $\sgn(\rho - 1/2)$ where $\rho = q(U - x)$ is the ``confidence'' of this prediction (\cref{eq:smoothed-hard}).
Randomized smoothing (via Neyman-Pearson) tells us that, if $\rho$ is large enough, then, no matter what $U$ is, a small perturbation of $x$ cannot decrease $\rho$ too much to change $\sgn(\rho - 1/2)$ (\cref{eqn:robustguarantee}).

If all we care about is robustness, then the optimal strategy would set $q$ to be an arbitrarily wide distribution (say, e.g. a wide Gaussian), and the resulting smoothed classifier is roughly constant.
Of course, such a smoothed classifier can never achieve good clean accuracy, so it is not useful.
Thus there is an inherent tension between 1) having to have large enough noise variance to be robust and 2) having to have small enough noise variance to avoid trivializing the smoothed classifier.
In this section, we seek to formalize this tradeoff.
As we'll show, even if we only assume a very weak condition on the accuracy, we can show strong upper bounds on the best robust radius for each $\ell_p$ norm.

In fact, our negative results below will hold for a more general class of smoothing schemes than those in our positive results in previous sections:
In what follows, a \emph{smoothing scheme} for $\R^d$ is any family of probability distributions $\cQ = \{ q_x\}_{x \in \R^d}$.
In practice, including in our paper, almost all smoothing schemes are \emph{translational}, that is, there is some base distribution $q$, and for every $x$, the smoothing distribution at $x$ is defined by $q_x (U) = q(U - x)$, for all base classifiers $U \subseteq \R^d$.
The above discussion then motivates the following
\begin{defn}
\label{def:useful}
Let $\| \cdot \|$ be a norm over $\R^d$, and let $\cQ = \{q_x\}_{x \in \R^d}$ be a smoothing scheme for $\R^d$.
We say that $\cQ$ satisfies \emph{$(\eps, s, \ell)$-useful smoothing} with respect to $\| \cdot \|$ if:
\begin{enumerate}[nosep]
    \item {\bf ($(\eps, s)$-Robustness)} For all $x, y$ with $\| x- y \| < \eps$, if $U \sbe \R^d$ is any set (\emph{read: base classifier}) satisfying $q_x (U) \geq 1/2 + s$, then $q_y (U) \geq 1/2$.
    \item {\bf ($\ell$-Accuracy)} For all $x, y$ with $\| x - y \| \geq 1$, there exists a set (\emph{read: base classifier}) $U \subseteq \R^d$ so that $|q_x (U) - q_y(U)| \ge \ell$.
\end{enumerate}
\end{defn}

\begin{center}
    \includegraphics[width=0.45\textwidth]{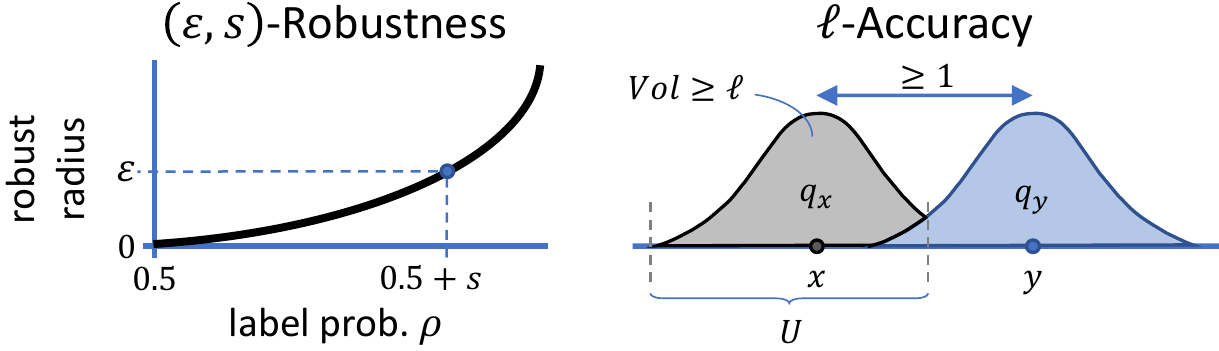}
\end{center}

We pause to interpret this definition.
Condition (1) indicates how large the certified radii can be for a classifier at any given point $x$, if the smoothed classifier assigns likelihood at least $1/2 + s$ to it; i.e. $(1/2 + s, \eps)$ is a point on the robust radii curve in the style of \cref{fig:radii_vs_rho}.
The goal of the smoothing scheme is to achieve the largest possible $\eps$, for every fixed $s$.
In particular, observe that for $\ell_2$, Gaussian smoothing achieves dimension-independent $\eps$, for every fixed choice of $s$ (\cref{thm:gaussianl2}).

Condition (2) says that the resulting smoothing should not ``collapse'' points: in particular, if $x, y$ are far in norm, then there should be some smoothed classifier that distinguishes them.
We argue that this is a very mild assumption.
For Condition (2) to be satisfied, the $U$ which distinguishes these two points can be completely arbitrary.
Thus, if it is violated for $\ell = o(1)$, the two distributions are indistinguishable by any statistical test in high dimension, implying the impossibility of classifying between $x$ and $y$ after smoothing.

We seek to show that, for constant $s$ and $l$, any $(\eps, s, \ell)$-useful smoothing scheme must have $\eps = o(1)$ for a number of norms, including $\ell_\infty$.
This would imply that any smoothing scheme that satisfying our weak notion of accuracy can only certify a vanishingly small radius, even when the confidence of the classifier is strictly bounded away from $1/2$ by a constant.

\vspace{1em}
\textbf{Randomized Smoothing as Metric Embedding}\ \
A smoothing scheme can be thought of as a mapping from a normed space supported on $\R^d$ to the space of distributions, e.g. each point $x$ is mapped to the distribution $q_x$.
We will show that \cref{def:useful} is roughly equivalent to a bi-Lipschitz condition on this mapping, where the target distributions are equipped with the total variation distance.
Then the existence of a \emph{useful} smoothing scheme is equivalent to whether $(\R^d, \|\cdot\|)$ can be embedded \emph{with low distortion} into the total variation space of distributions.
Classical mathematics has a definitive answer to this question in the form of a geometric invariant, called the \emph{cotype}.
\begin{defn}[see e.g.~\citet{wojtaszczyk1996banach}]
A normed space $T = (X, \| \cdot \|)$ is said to have \emph{cotype} $p$ for $2 \leq p \leq \infty$ if there exists $C$ such that for all finite sequences $x_1, \ldots, x_n \in X$, we have
\[\mathbb{E} \left[ \left\| \sum_{j = 1}^n \sigma_j x_j \right\| \right] \ge C^{-1} \left( \sum_{j = 1}^n \| x_i \|^p  \right)^{1/p},\]
where the $\sigma_j$ are independent Rademacher random variables.
The smallest such $C$ is denoted $C_p (T)$.
\end{defn}
When the underlying space of the normed space $T$ is $\R^d$, John's theorem~\citep{john1948extremum} implies that any norm has cotype $2$ with $C_2 (T) \leq O(d^{1/2})$.
Because $C_2$ lower bounds the distortion of a metric embedding of $T$, by the aforementioned connection with randomized smothing, $C_2$ also limits the usefulness of any smoothing scheme of $T$:

\begin{thm}
    \label{thm:main-lb}
    Let $T$ be any normed space over $\R^d$.
    There exist universal constants $c, K > 0$ so that any $(\eps, s, \ell)$-useful smoothing scheme for $T$ with $s/\ell < c$ must have
    \[\eps \le K \sqrt[4]{s / \ell} \cdot C_2(T)^{-1}.\]
\end{thm}
In particular, it is well-known that $C_2 ((\R^d, \| \cdot \|_p)) = \Omega (\max(1, d^{1/2 - 1/p}))$, for all $p \in [1, \infty]$.
Thus, as an immediately corollary, we get:
\begin{cor}\label{cor:lp-lb}
For the value of $c$ in \cref{thm:main-lb} and for $p \in [1, \infty]$, any $(\eps, s, \ell)$-useful smoothing scheme for $(\R^d, \| \cdot \|_p)$ with $s/\ell < c$ must have
\[\eps \le O(\min(1, d^{-1/2 + 1/p})).\]
\end{cor}
It is easy to see that, up to constants, the Gaussian smoothing scheme achieves equality, and thus is optimal (in terms of dimension dependence), for all $p \in [1, \infty]$.

\paragraph{Discussion}
After \citet{cohen_certified_2019} showed the surprising scalability of Gaussian randomized smoothing to high-dimensional $\ell_2$-robust classification problems, many anticipated that this can be extended to $\ell_\infty$ as well.
One might also hope that, even though it seems like we cannot certify $\ell_2$ radius that grows with input dimension, we could do so for $\ell_1$.
But \cref{thm:main-lb,cor:lp-lb} present a strong barrier to such hopes.
In words:
\begin{center}
\it
\parbox{0.8\linewidth}{
Without using more than the information of the probability $\rho$ of correctly classifying an input under random noise, no smoothing techniques can certify nontrivial robust accuracy at $\ell_\infty$ radius $\Omega (d^{-1/2})$, or at $\ell_2$ or $\ell_1$ radius $\Omega(1)$.
}
\end{center}
Indeed, the $\ell_1$-radii we can obtain nontrivial certified accuracy at are on the same order between CIFAR10 and Imagenet (\cref{fig:l1_certificates}).

However, there are some ways to bypass this barrier.
For one, more information about the base classifier can be collected to produce better robustness certificates.
In fact, \citet{dvijotham_framework_2019} proposed a ``full-information'' algorithm that computes many moments of the base classifier in a convex optimization procedure to improve certified radius, but it is 100 times slower than the ``information-limited'' algorithms we discuss here that use only $\rho$.
It would be interesting to see whether this technique can be scaled up, and whether other methods can leverage more information\!~\footnote{
\citet{lee_tight_2019} also used the decision tree structure of their base classifier to improve $\ell_0$ certification, but the $\ell_0$-adversary does not fall within our framework.}.

Another route is to directly look for better randomized smoothing schemes for multi-class classification.
We formulated our no-go result in the setting of binary classification, and it is not clear whether a similarly strong barrier applies for multi-class classification.
However, current techniques for certification only look at the two most likely classes, and separately reason about how much each one can change by perturbing the input.
Our no-go result then straightforwardly applies to this case as well.

\section{Conclusion}

In this work, we have showed how far we can push randomized smoothing with different smoothing distributions against different $\ell_p$ adversaries, by presenting two new techniques for deriving robustness guarantees, by elucidating the geometry connecting the noise and the norm, and by empirically achieving state-of-the-art in $\ell_1$ provable defense.
At the same time, we have showed the limit current techniques face against $\ell_p$ adversaries when $p > 2$, especially $\ell_\infty$.
Our results point out ways to bypass this barrier, by either leveraging more information about the base classifier or by taking advantage of the multi-class problem structure better.
We wish to investigate both directions in the future.

More broadly, randomized smoothing is a method for inducing stability in a mechanism while maintaining utility --- precisely the bread and butter of differential privacy.
We suspect our methods for deriving robustness guarantees here and for optimizing the noise distribution can be useful in that setting as well, where Laplace and Gaussian noise dominate the discussion.
Whereas previous work \citet{lecuyer2018certified} has applied differential privacy tools to randomized smoothing, we hope to go the other way around in the future.

\newpage

\subsection*{Acknowledgements}
We thank Huan Zhang for brainstorming of ideas and performing a few experiments that unfortunately did not work out.
We also thank Aleksandar Nikolov, Sebastien Bubeck, Aleksander Madry, Zico Kolter, Nicholas Carlini, Judy Shen, Pengchuan Zhang, and Maksim Andriushchenko for discussions and feedback.

\bibliography{bib/wulff_crystal,bib/wasserstein,bib/robustness}
\bibliographystyle{icml2020}
\newpage
\appendix

\counterwithin{figure}{section}
\counterwithin{table}{section}

\newpage
\onecolumn
\section{Table of Robust Radii}

\begin{table*}[h]
    \centering
    \begin{tabular}{llllll} \toprule
         Distribution & Density & Adv. & Certified radius & Reference\\ \midrule
         iid Log Concave & $\propto e^{-\sum_i\phi(x_i)}$ & $\ell_1$ & $\mathrm{CDF}_\phi^{-1}(\rho)$ & \cref{thm:iidlogconcaveL1}\\
         iid Log Convex* & $\propto e^{-\sum_i \phi(|x_i|)}$& $\ell_1$ &
            $\int^{\infty}_{\varphi^{-1}(1-\rho)}
            \f{1}{e^{\phi(c) - \phi(0)} - 1}
            \dd c$\hfill for $\varphi$, see $\rightarrow$
            &\cref{thm:iidlogconvexL1}\\
         Exp. $\ell_p, p \ge 1$ & $\propto e^{-\|\f x \lambda\|_p^p}$ & $\ell_1$ & $\lambda\sqrt[p]{\GammaCDF^{-1}(2\rho - 1; 1/p)}$ & \cref{cor:lpExponentialL1Adv}\\
         Exp. $\ell_p, p < 1$ & $\propto e^{-\|\f x \lambda\|_p^p}$ & $\ell_1$ & $\lambda \int^{\infty}_{\varphi^{-1}(1-\rho)}\f{\dd c}{e^{c^p} - 1}$\hfill for $\varphi$, see $\rightarrow$ & \cref{cor:lpExponentialLogconvexL1Adv} \\
         Gaussian  & $\propto e^{-\|\f x \lambda\|^2_2 / 2}$ & $\ell_2$ &  $\lambda\GaussianCDF^{-1}(\rho;0,1)$ & \cref{thm:gaussianl2}$^{\text{C}}$\\
         && $\ell_1$ & $\lambda\GaussianCDF^{-1}(\rho;0,1)$ & Symmetry \\
         & & $\ell_\infty$ & $\lambda\GaussianCDF^{-1}(\rho;0,1) / \sqrt{d}$ & Symmetry\\
         Laplace  & $\propto e^{-\|\f x \lambda\|_1}$ & $\ell_1$ &  $-\lambda\log (2(1-\rho))$ & \cref{thm:laplaceL1}$^{\text{T}}$\\
         && $\ell_\infty$ &  $\approx\lambda\mathrm{GaussianCDF}^{-1}(\rho;0,1)/\sqrt{d}$\hfill see $\rightarrow$
         & \cref{thm:laplacelinf}\\
         Exp. $\ell_\infty$ & $\propto e^{-\|\f x \lambda\|_\infty}$ & $\ell_1$ & $2d \lambda(\rho-\f 1 2)$ & \cref{thm:ExpInfL1}\\
         && $\ell_\infty$ & $\lambda \log \f {1}{2(1-\rho)}$ & \cref{thm:ExpInfLinfty}\\
         Exp. $\ell_2$ & $\propto e^{-\|\f x \lambda\|_2}$ & $\ell_2$ & $\lambda (d-1)\arctanh(
             1 - 2\beta^{-1}\lp
                 1 - \rho
                 ;
                 \f{d-1}2, \f{d-1}2
             \rp
        )$ & \cref{eqn:exp2l2}\\
        && $\ell_1$ & $\lambda(d-1)\arctanh(
            1 - 2\beta^{-1}\lp
                1 - \rho
                ;
                \f{d-1}2, \f{d-1}2
            \rp
        )$ & Symmetry\\
        && $\ell_\infty$ & $\f{\lambda(d-1)}{\sqrt d}\arctanh(
            1 - 2\beta^{-1}\lp
                1 - \rho
                ;
                \f{d-1}2, \f{d-1}2
            \rp
        )$ & Symmetry\\
         Uniform $\ell_\infty$  & $\propto \ind(\|x\|_\infty \le \lambda) $ & $\ell_1$ & $2\lambda(\rho-\frac{1}{2})$ & \cref{thm:UniformL1}$^{\text{L}}$\\
         && $\ell_\infty$ & $2\lambda(1-\sqrt[d]{\frac{3}{2}-\rho})$ & \cref{thm:UniformLinfty}$^{\text{L}}$\\
         Uniform $\ell_2$ & $\propto \ind(\|x\|_2 \le \lambda)$ & $\ell_2$ & $\lambda\left(2-4\beta^{-1}\left(\frac{3}{4}-\frac{\rho}{2};\frac{d+1}{2},\frac{d+1}{2}\right)\right)$ & \cref{thm:BallUniformAdvL2}\\
         && $\ell_1$ & $\lambda\left(2-4\beta^{-1}\left(\frac{3}{4}-\frac{\rho}{2};\frac{d+1}{2},\frac{d+1}{2}\right)\right)$ & Symmetry\\
         && $\ell_\infty$ & $\f\lambda{\sqrt d}\left(2-4\beta^{-1}\left(\frac{3}{4}-\frac{\rho}{2};\frac{d+1}{2},\frac{d+1}{2}\right)\right)$ & Symmetry\\
         General Exp. $\ell_\infty$ & $\propto \|\f x \lambda\|_\infty^{-j}e^{-\|\f x \lambda\|_\infty^k}$ & $\ell_1$ & $\frac{2d\lambda}{d-1}\Gamma\left(\frac{d-j}{k}\right)/\Gamma\left(\frac{d-1-j}{k}\right)\left(\rho-\frac{1}{2}\right)$ & \cref{thm:ExpInfkjL1}\\
         && $\ell_\infty$ & $\lambda \int_{1-\rho}^{1/2} \f 1 {\Phi(p)} \dd p$\hfill for $\Phi$, see $\rightarrow$ & \cref{thm:ExpInfjkLinfty} \\
         General Exp. $\ell_2$ & $\propto \|\f x \lambda\|_2^{-j} e^{-\|\f x \lambda\|_2^k}$ & $\ell_2$ & \emph{level set method} & \cref{sec:l2exponential} \\
         General Exp. $\ell_1$ & $\propto e^{-\|\f x \lambda\|_1^k}$ & $\ell_1$ & $\lambda\int_{1-\rho}^{1/2} \f R {\Psi(p)} \dd p$\hfill for $R, \Psi$, see $\rightarrow$ & \cref{thm:Exp1kL1} \\
         && $\ell_\infty$ & $\lambda \int_{1-\rho}^{1/2} \f 1 {\Phi(p)} \dd p$\hfill for $\Phi$, see $\rightarrow$ & \cref{thm:exp1kLinf} \\
         Power Law $\ell_\infty$ & $\propto \f{1}{(1+\|\f x \lambda\|_\infty)^a}$ & $\ell_1$ & $\frac{2d\lambda}{a-d}\left(\rho-\frac{1}{2}\right)$ & \cref{thm:ExpInfpolyL1}\\
         && $\ell_\infty$ & $\f{2\lambda}{a-d} \int_{1-\rho}^{1/2} \f {\dd p} {\Upsilon(\Upsilon^{-1}(2p; d, a-d); d, a+1-d)}$ & \cref{thm:ExpInfpolyLinf} \\
         Power Law $\ell_2$ & $\propto \f 1 {(1+\|\f x \lambda\|_2^k)^{a}}$ & $\ell_2$ & \emph{level set method} &  \cref{sec:l2powerlaw}\\
         Pareto (i.i.d.) & $\propto \f{1}{\prod_i \left(1+\frac{|x_i|}{\lambda}\right)^{a+1}}$  & $\ell_1$ & $\lambda \frac{2\rho-1}{a}\, _2F_1\left(1, \f a {a+1}, \frac{2a+1}{a+1};(2\rho-1)^{1+1/a}\right)$ & \cref{thm:paretoL1}\\
         \bottomrule
    \end{tabular}
    \caption{Distributions we derive robust radii for and assess experimentally.
    Here $\rho$ is the probability the base classifier answers correctly when input is perturbed by the smoothing noise,
    $d$ is the dimensionality of the noise, $\mathrm{CDF}_\phi^{-1}$ is the inverse CDF of the 1D random variable with density $\propto e^{-\phi(x)},$ $\beta^{-1}(\cdot; a, b)$ is the inverse Beta CDF function with shape parameters $a$ and $b$,
    $\Upsilon(\cdot; a, b)$ (resp. $\Upsilon^{-1}(\cdot; a, b)$) is the Beta Prime (resp. inverse) CDF function with shape parameters $a$ and $b$,
    $\Gamma$ is the Gamma function, and $_2F_1$ is the Gaussian hypergeometric function.
    Under \emph{Reference}, superscript $^\text{C}$ refers to \citet{cohen_certified_2019}, superscript $^\text{L}$ refers to \citet{lee_tight_2019}, and superscript $^\text{T}$ refers to \citet{teng_ell_1_2019}.
    }

    \label{tab:bounds}
\end{table*}

\begin{table*}[h]
    \centering
    \begin{tabular}{llllll} \toprule
         Distribution $q$ & Density & $\lambda/\sigma = \lambda / \sqrt{\EV_{\delta \sim q} \f 1 d \|\delta\|^2}$\\ \midrule
         Exp. $\ell_p$ & $\propto e^{-\|\f x \lambda\|_p^p}$ & $\sqrt{\f {\Gamma(1/p)} {\Gamma(3/p)}}$\\
         Gaussian  & $\propto e^{-\|\f x \lambda\|^2_2 / 2}$ & $1$\\
         Laplace  & $\propto e^{-\|\f x \lambda\|_1}$ & $1/\sqrt 2$ \\
         Exp. $\ell_\infty$ & $\propto e^{-\|\f x \lambda\|_\infty}$ & $\sqrt{\f 1 {(d+1) ((d-1)/3 + 1)}}$\\
         Exp. $\ell_2$ & $\propto e^{-\|\f x \lambda\|_2}$ & $\sqrt{\f 1 {d+1}}$\\
         Uniform $\ell_\infty$  & $\propto \ind(\|x\|_\infty \le \lambda) $ & $\sqrt 3$\\
         Uniform $\ell_2$ & $\propto \ind(\|x\|_2 \le \lambda)$ & $\sqrt{\f 1 {d+2}}$\\
         General Exp. $\ell_\infty$ & $\propto \|\f x \lambda\|_\infty^{-j}e^{-\|\f x \lambda\|_\infty^k}$ &
            $\sqrt{
                \f{
                    d\Gamma(\f{d-j}k)
                }{
                    ((d-1)/3+1)\Gamma(\f{d+2-j}k)
                }}$\\
         General Exp. $\ell_2$ & $\propto \|\f x \lambda\|_2^{-j} e^{-\|\f x \lambda\|_2^k}$ & $\sqrt{\f{d\Gamma(\f{d-j}k)}{\Gamma(\f{d+2-j}k)}}$\\
         General Exp. $\ell_1$ & $\propto e^{-\|\f x \lambda\|_1^k}$ & $\sqrt{\f{d(d+1)\Gamma(\f{d}k)}{2\Gamma(\f{d+2}k)}}$\\
         Power Law $\ell_\infty$ & $\propto \f{1}{(1+\|\f x \lambda\|_\infty)^a}$ & $\sqrt{\f{(a-d-1)(a-d-2)}{(d+1)((d-1)/3 + 1)}}$\\
         Power Law $\ell_2$ & $\propto \f 1 {(1+\|\f x \lambda\|_2^k)^{a}}$ &
         $\sqrt{\frac{\Gamma \left(\frac{d+2}{k}\right) \Gamma \left(a-\frac{d+2}{k}\right)}{d\Gamma \left(\frac{d}{k}\right) \Gamma \left(a-\frac{d}{k}\right)}}$\\
         Pareto (i.i.d.) & $\propto \f{1}{\prod_i \left(1+\frac{|x_i|}{\lambda}\right)^{a+1}}$  & $\sqrt{\frac{1}{2}(a-1)(a-2)}$\\
         \bottomrule
    \end{tabular}
    \caption{Relation between the scale parameter $\lambda$ and the variance $\sigma^2 = \EV_{\delta \sim q} \f 1 d \|\delta\|^2$ of each distribution.
    This table is used to choose the correct $\lambda$ to match $\sigma$ across different distributions.
    All quantities can be computed easily using \cref{lemma:GammaExpect,lemma:expNormSampling}.
    They are also tested to be numerically correct in the test suite of our code base \repo{}.
    }

    \label{tab:distinfo}
\end{table*}

\begin{figure*}[t]
    \centering
    \includegraphics[width=\linewidth]{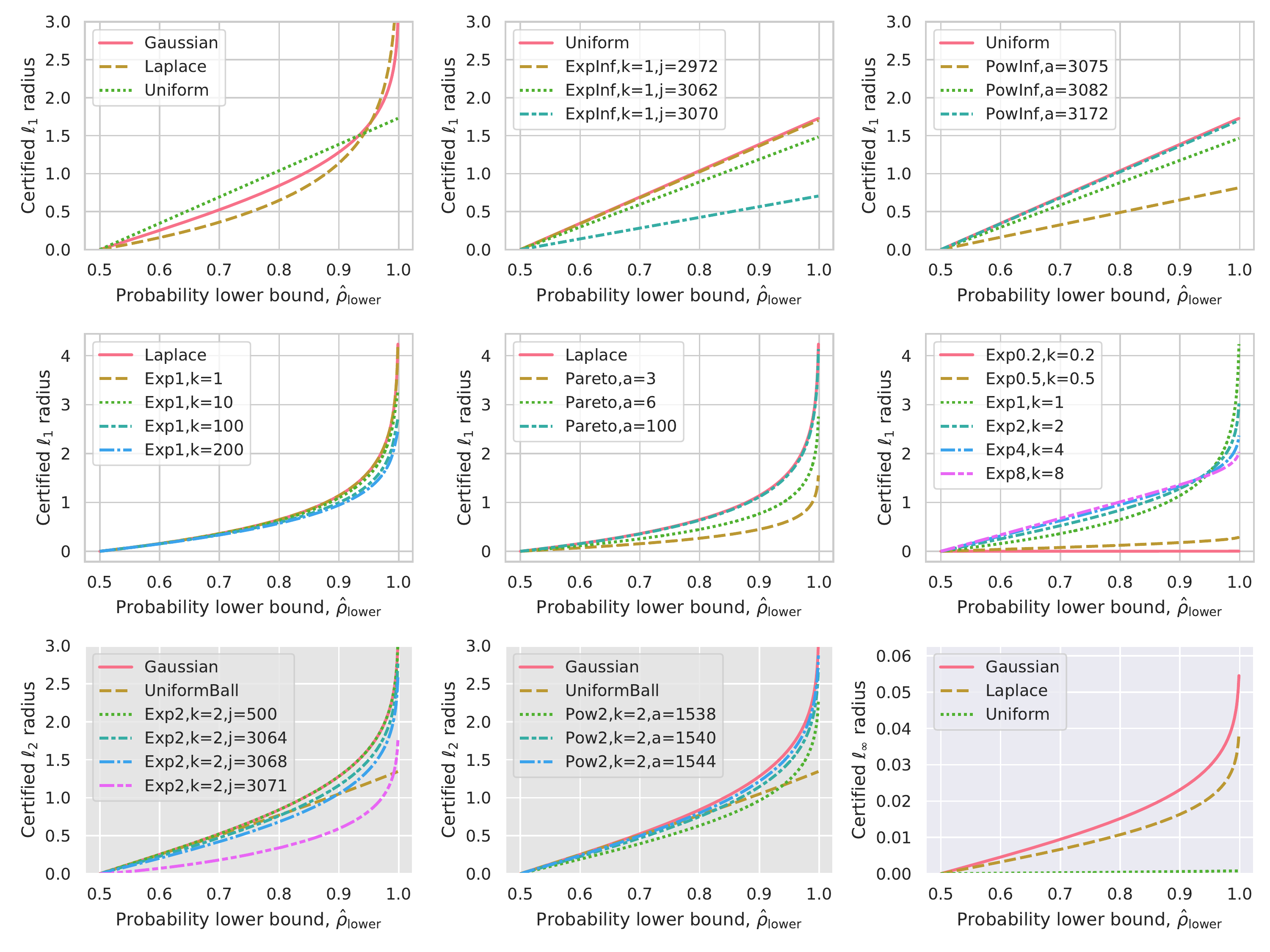}
    \caption{Certified robust radii of a selection of the distributions in \cref{tab:bounds}, with input dimension $d = 3072$ and normalized variance $\sigma^2 = 1$, across a range of $\hat \rho_{\mathrm{lower}}$, the high probability lower bound of $\rho$ (the probability that the base classifier answers correctly when perturbed by smoothing noise).
    The first two rows are for the $\ell_1$ adversary while the last row is for the $\ell_2$ and $\ell_\infty$ adversaries.}
    \label{fig:radii_vs_rho}
\end{figure*}

\twocolumn

\section{Analysis of Robust Radii}

Here we make a few observations about the robust radii of the distributions studied in this paper.

\paragraph{Distributions that concentrate around the same level set have similar robust radii.}
This is evident, for example, in the top middle subplot of \cref{fig:radii_vs_rho}, where the distribution $\propto \|x\|_\infty^{-j} e^{-\|x/\lambda\|_\infty}$ with ``small'' $j=2972$ has robust radii almost the same as those of $\propto e^{-\|x/\lambda\|_\infty}$ (here $\lambda$ for each distribution is the one that sets $\sigma = 1$), and both distributions concentrate around the sphere of radius $\sqrt d$.
We can also see this in the top right ($\ell_\infty$-based power law), middle left ($\ell_1$-based exponential law), bottom left ($\ell_2$-based exponential law), and bottom middle ($\ell_2$-basd power law) subplots of \cref{fig:radii_vs_rho}.
This is also reflected in the center subplot of \cref{fig:radii_vs_rho}, which shows that Pareto distribution with large power gets the same robust radii as Laplace.
The reason is that such a high-power Pareto distribution concentrates around an $\ell_1$-ball in high dimension.

We can understand this phenomenon intuitively via the level set method:
Two distributions concentrating around the same level set will have \cref{eqn:psmall} and \cref{eqn:pbig} evaluate to similar quantities.

\paragraph{Among distributions concentrated around some level set, the shape of the level set is the biggest determinant of performance.}
This is evident in the top left, middle right, and bottom right subplots of \cref{fig:radii_vs_rho}.

\paragraph{Distributions that don't concentrate on a level set do worse than those that do.}
This is evident, for example, in the top middle subplot of \cref{fig:radii_vs_rho}, where the distribution $\propto \|x\|_\infty^{-j} e^{-\|x/\lambda\|_\infty}$ with ``large'' $j=3070$ has robust radii much smaller than those of $\propto e^{-\|x/\lambda\|_\infty}$.
Same thing can be observed in the top right, center, middle right, bottom left, bottom middle subplots of \cref{fig:radii_vs_rho}.

\paragraph{Introducing a singularity at the origin only reduces robust radii.}
The top middle and bottom left subplots of \cref{fig:radii_vs_rho} illustrate this point.
Thus, we see no evidence for the ``soap-bubble hypothesis'' put forth by \citet{zhang*2020filling}; see also \cref{fig:l2_additional}.

\paragraph{Introducing a fatter tail yields larger robust radii for large $\hat\rho_{\mathrm{lower}}$, as long as the level set concentration is not affected.}
The middle left and bottom left subplots of \cref{fig:radii_vs_rho} demonstrate this behavior.
The robust radii formulas for $\exp(-\|x\|_\infty)$ (\cref{thm:ExpInfL1}) and for the uniform distribution (\cref{thm:UniformL1}) also reflect this, as the former has robust radius $\to \infty$ as $\rho \to 1$, but the latter has a finite maximal robust radius.

\subsection{Level Set Method vs Differential Method}
\label{sec:levelset_vs_differential}

Here we concretely compare the robust radii obtained from the level set method and those obtained from the differential method for the distribution $\exp(-\|x\|_2 \sqrt d)$, for various input dimensions $d$ (we scale the distributions this way so each coordinate has size $\Theta(1)$).
For convenience, here's the robust radius from the differential method (\cref{eqn:exp2l2}):
\begin{align*}
    R
        &=
            \f{d-1}{\sqrt d} \arctanh\bigg(\\
        &\phantomeq\quad
            1 - 2\BetaCDF^{-1}\lp
                1 - \rho
                ;
                \f{d-1}2, \f{d-1}2
            \rp
        \bigg).
\end{align*}
The robust radii from level set method are computed as in \cref{thm:levelsetl2}, and they are tight.
As we see in \cref{fig:levelset_v_diff}, the differential method is \emph{very slightly} loose in low dimensions $d = 2$ and 4, but in high dimensions $d = 32$ or $1024$, the robust radii obtained from both methods are indistinguishable.

\begin{figure}
    \centering
    \includegraphics[width=\linewidth]{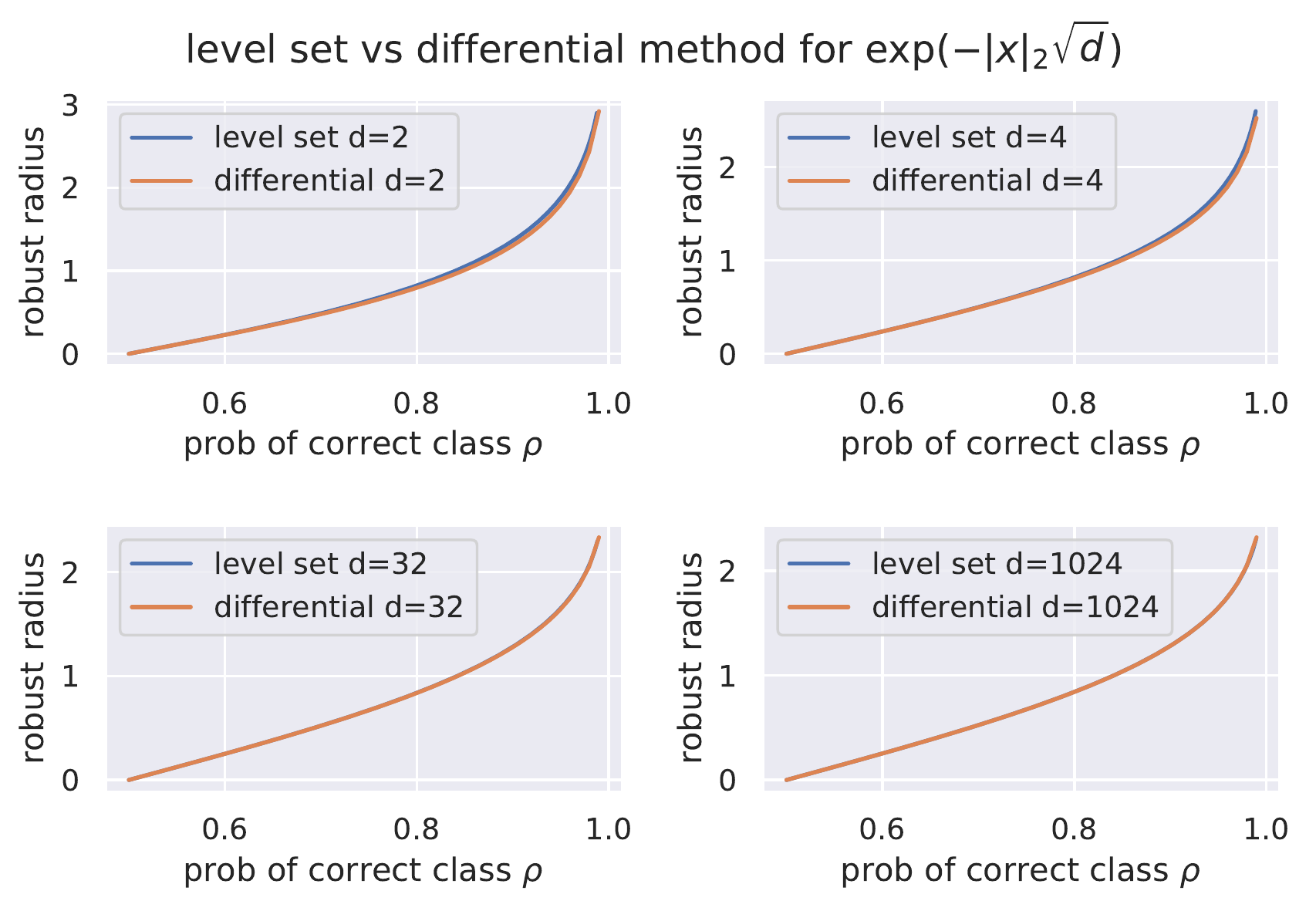}
    \caption{\textbf{Differential Method is Tight} for practical purposes in high dimension $d$.}
    \label{fig:levelset_v_diff}
\end{figure}

\subsection{In-Depth Comparison with \texorpdfstring{\citet{dvijotham_framework_2019}}{Dvijotham et al. (2019)}}
\label{sec:fdivCompare}

The information-limited certification algorithm in \citet{dvijotham_framework_2019} relaxes the optimization problem
\begin{align*}
    \sup_{v \in \mathcal B} \growth_q(p, v)
    &=
        \sup_{q' \in q_{\mathcal B}}
        \sup_{U: q(U) = p} q'(U)
        \\
    &\le
        \sup_{q' \in \mathcal D_F(q)}
        \sup_{U: q(U) = p} q'(U)
        \numberthis\label{eqn:fdivRelax}
\end{align*}
enlarging the set of shifted distributions $q_{\mathcal B} \defeq \{q(\cdot - v): v \in \mathcal B\}$ to the set of distributions close to $q$ in several $f$-divergences $\mathcal D_F \defeq \{q': \mathcal D_f(q'\| q) \le \epsilon_f, \forall f \in F\}$, for a set of functions $F$.
\citet{dvijotham_framework_2019} showed that when $F$ consists of all Hockey-Stick divergences, \cref{eqn:fdivRelax} becomes tight,, but in practice this is not feasible.
In fact, \citet{dvijotham_framework_2019} admits themselves that
\begin{quote}
\it
It turns out that the Renyi and KL divergences are computationally attractive for a broad class of smoothing measures, while the Hockey-Stick divergences are theoretically attractive as they lead to optimal certificates in the information-limited setting. However, Hockey-Stick divergences are harder to estimate in general, so we only use them for Gaussian smoothing measures.
\end{quote}
Concretely, the looseness of their relaxation can be observed when comparing our baseline Laplace results (\cref{tab:sota}) with theirs.

Operationally, their algorithm proceeds as follows
\begin{enumerate}
    \item For each distribution $q$ and function $f$, manually find the $f$-divergence ``ball'' that contains $\{q(\cdot - v): v \in \mathcal B\}$, i.e. compute $\{\epsilon_f\}_{f \in F}$ such that
    \begin{align*}
    \{q(\cdot - v): v \in \mathcal B\} \sbe \{q' : \mathcal D_f(q'\| q) \le \epsilon_f, \forall f \in F\}.
    \end{align*}
    \item Then they relax the original certification problem to the certification of all $q'$ close to $q$ in $f$-divergence, i.e. they solve \cref{eqn:fdivRelax} for the $\epsilon_f$ found in the previous step.
\end{enumerate}
The 2nd step is a straightforward low-dimensional convex optimization problem, but the trickiness of the 1st step limits the distributions they can apply their technique to.
For example, they only know how to do step 1 for $\exp(-\|x\|_p)$ against $\ell_p$ adversary, but not against $\ell_r$ for $r \ne p$; in contrast, our differential method computes robust radii for Laplace against $\ell_\infty$ perturbation, for example.

\section{Additional Experimental Results}
\label{sec:additional_results}

\begin{figure}[t!]
    \centering
    \begin{subfigure}[t]{0.91\linewidth}
        \caption{CIFAR-10}
        \includegraphics[width=\linewidth]{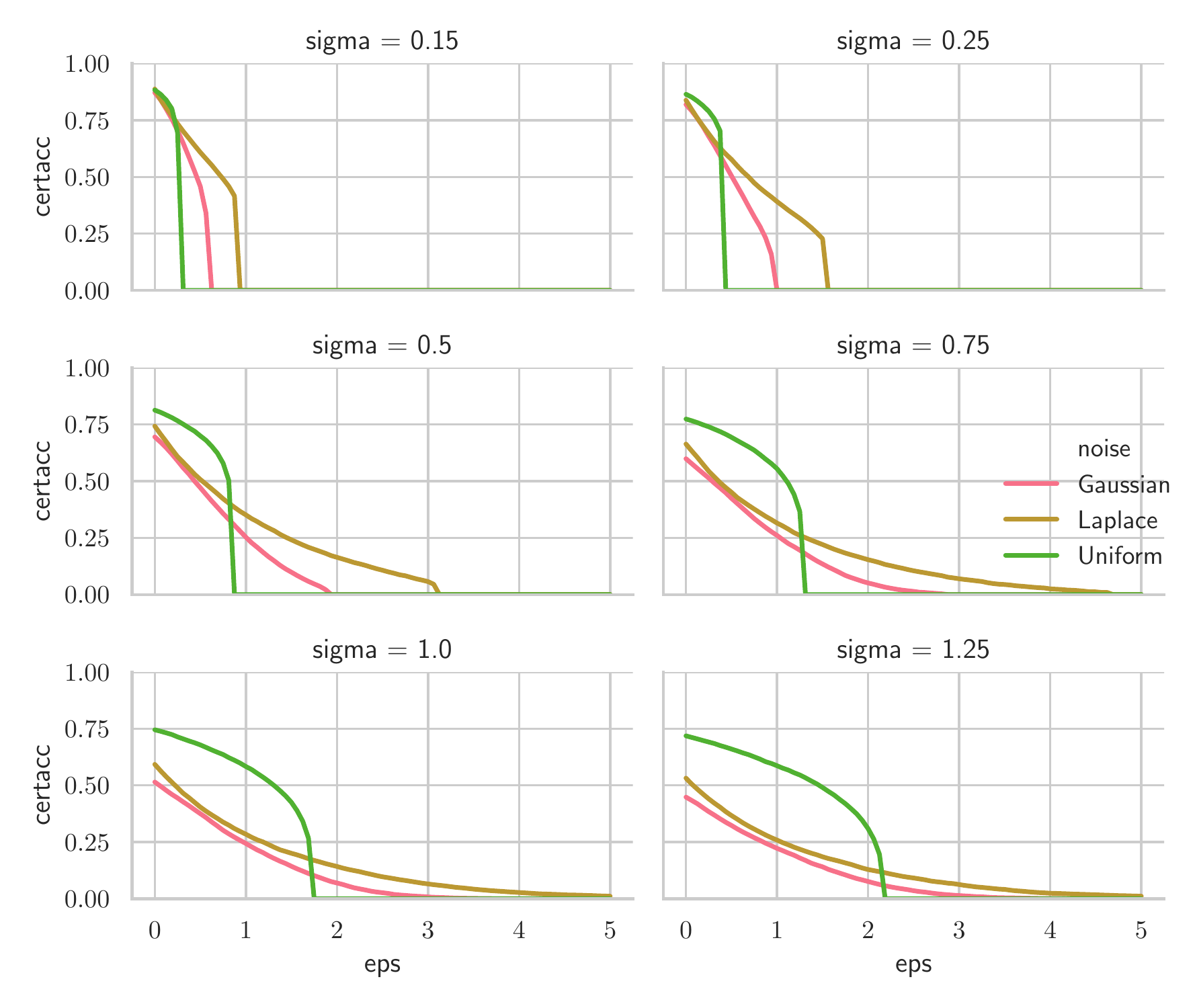}
    \end{subfigure}
    \begin{subfigure}[t]{0.91\linewidth}
        \caption{ImageNet}
        \includegraphics[width=\linewidth]{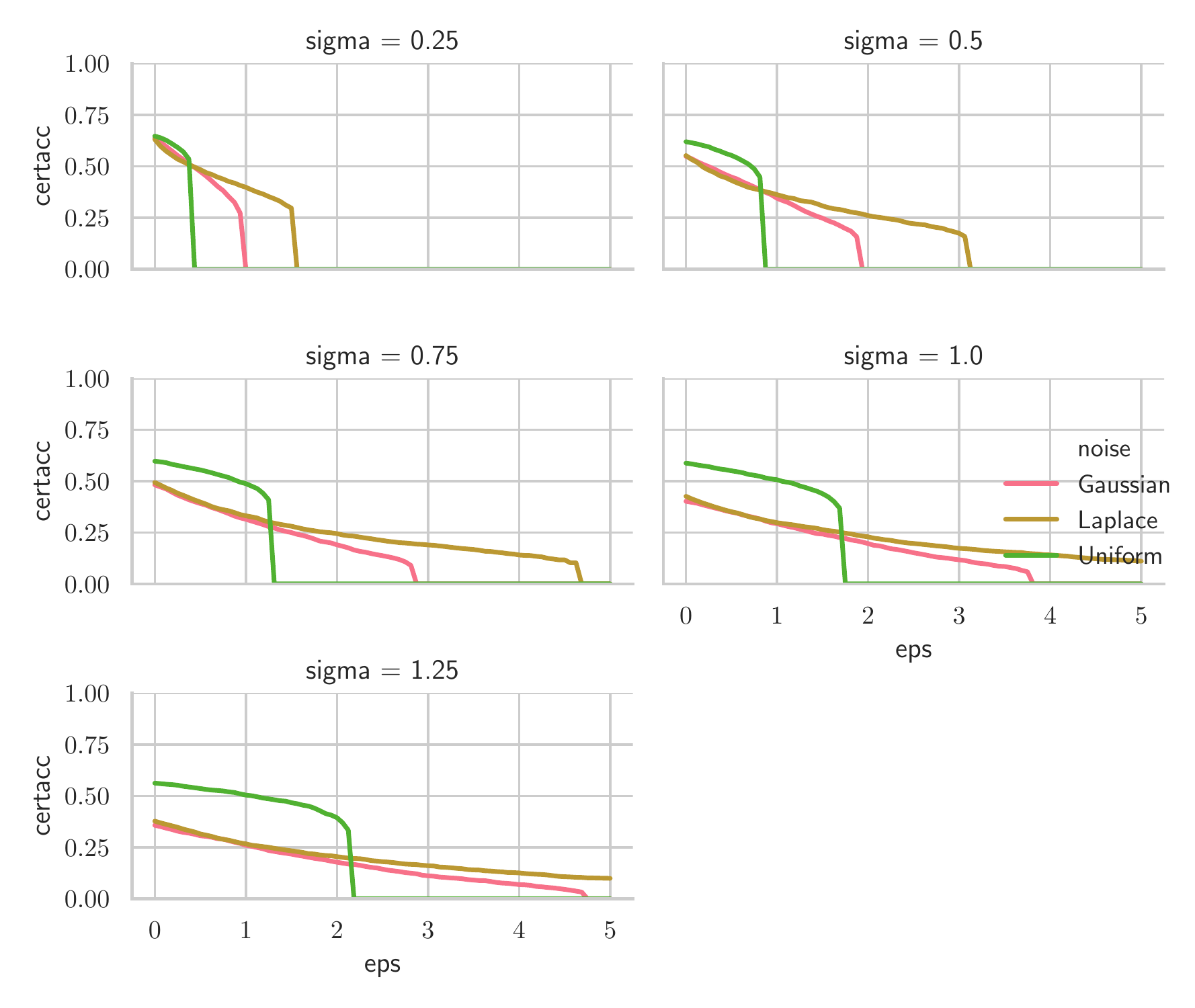}
    \end{subfigure}
    \caption{\textbf{Certified Accuracy per $\sigma$}. Certified accuracies against an $\ell_1$ adversary at each level of $\epsilon$, across the range of $\sigma$ with which models were trained (we omit $\sigma > 1.25$ for brevity).
    The upper envelope for each distribution is taken to be the maximum certified accuracy across values of $\sigma$.}
    \label{fig:per_sigma}
\end{figure}

\begin{figure}[t]
    \centering
    \begin{subfigure}[t]{0.49\linewidth}
    \includegraphics[width=\linewidth]{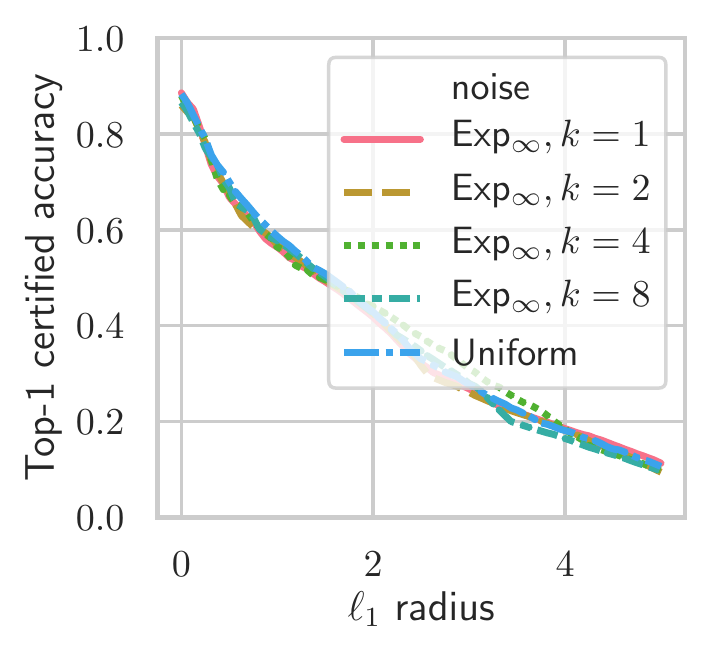}
    \end{subfigure}
    \begin{subfigure}[t]{0.49\linewidth}
    \includegraphics[width=\linewidth]{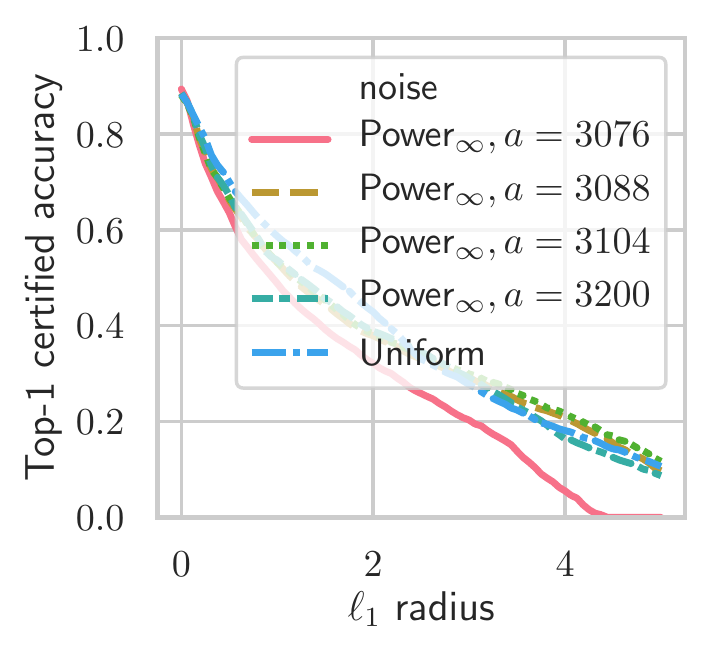}
    \end{subfigure}
    \begin{subfigure}[t]{0.49\linewidth}
    \includegraphics[width=\linewidth]{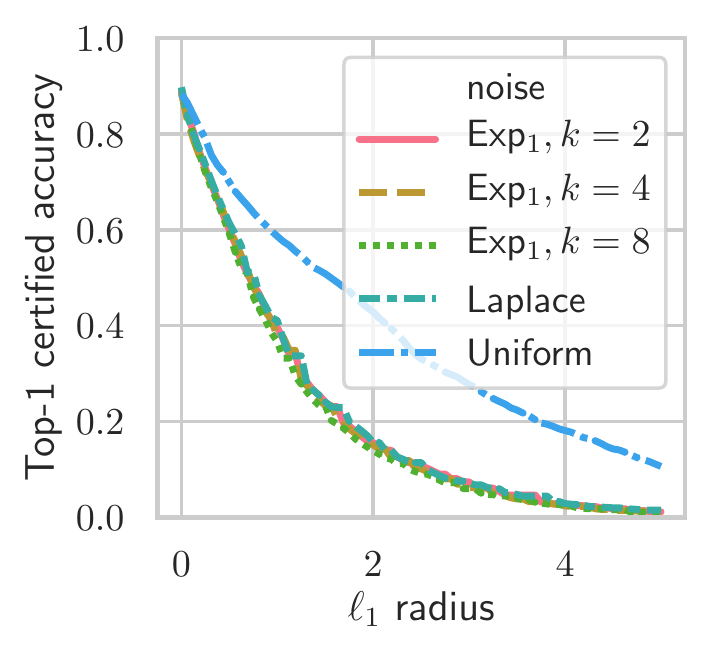}
    \end{subfigure}
    \begin{subfigure}[t]{0.49\linewidth}
    \includegraphics[width=\linewidth]{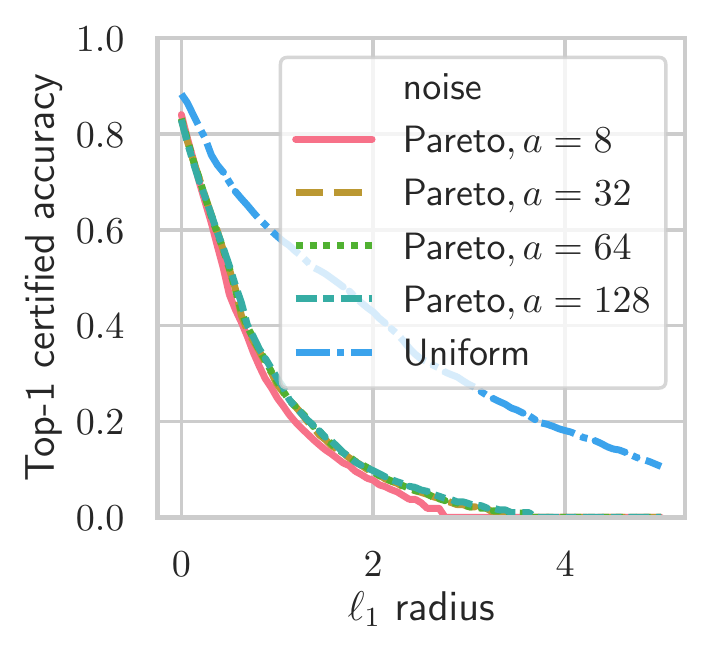}
    \end{subfigure}
    \caption{\textbf{More Distributions for $\ell_1$ Adversary.} CIFAR-10 certified top-1 accuracies of against the $\ell_1$ adversary, on generalized exponential law (with $\ell_\infty$ and $\ell_1$ level sets), power law (with $\ell_\infty$ level sets), and Pareto distributions.
    After appropriate hyperparameter search ($k$ or $a$), distributions with cubic level sets achieve performance roughly matching that of the Uniform distribution.}
    \label{fig:l1_additional}
\end{figure}

\begin{figure}[t]
    \centering
    \begin{subfigure}[t]{0.49\linewidth}
        \includegraphics[width=\linewidth]{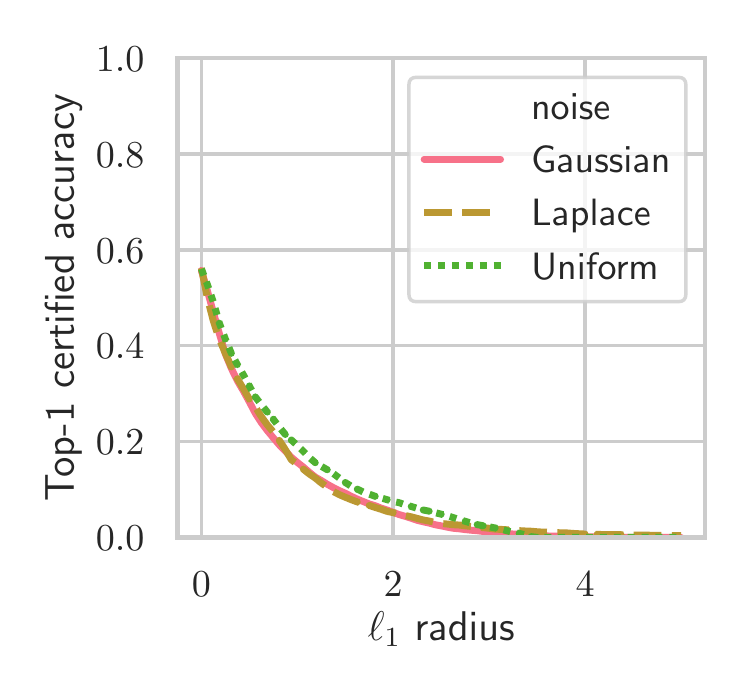}
    \end{subfigure}
    \begin{subfigure}[t]{0.49\linewidth}
        \includegraphics[width=\linewidth]{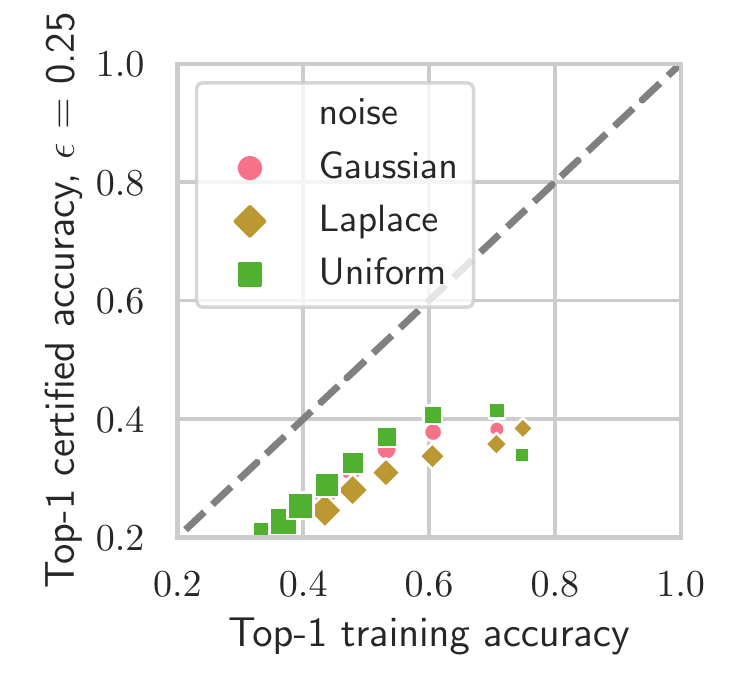}
    \end{subfigure}
    \caption{\textbf{Multi-layer Perceptron.} (Left) CIFAR-10 certified top-1 accuracies for the $\ell_1$ adversary, with a multi-layer perceptron. (Right) Certified accuracies at $\ell_1$ perturbation $\epsilon=0.25$ plotted against training accuracy under smoothing noise.}
    \label{fig:mlp_additional}
\end{figure}

\begin{figure}[t]
    \centering
    \caption{\textbf{Effect of Architecture.} Clean CIFAR-10 training (left) and testing (right) accuracies for Wide ResNet, AlexNet, and a fully connected neural network, at fixed levels of $\mathbb{E}[\frac{1}{d}\|\delta\|_2^2]\defeq \sigma^2$. For fixed $\sigma$, there is no difference between the distributions when smoothing a fully connected network, but differences arise when the architecture improves to AlexNet and ResNet.}
    \begin{subfigure}[t]{\linewidth}
        \caption{Wide ResNet 40-2}
        \includegraphics[width=\linewidth]{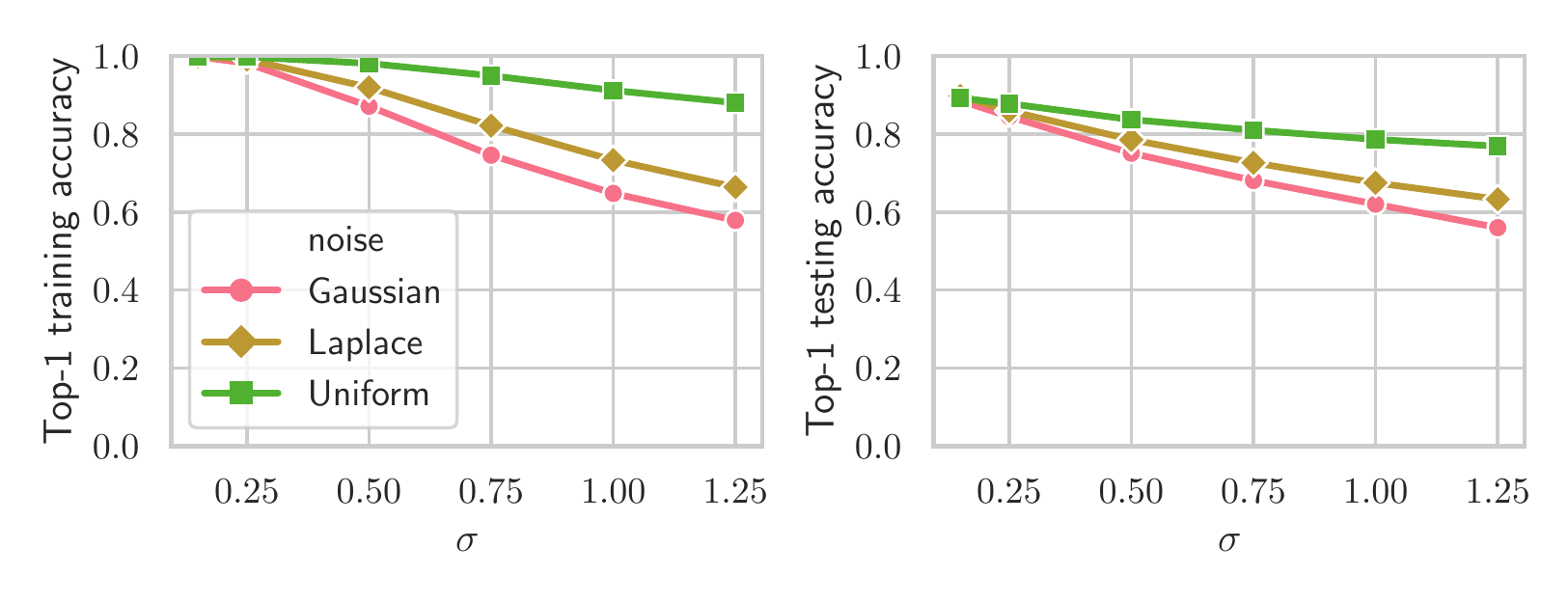}
    \end{subfigure}
    \begin{subfigure}[t]{\linewidth}
        \caption{AlexNet}
        \includegraphics[width=\linewidth]{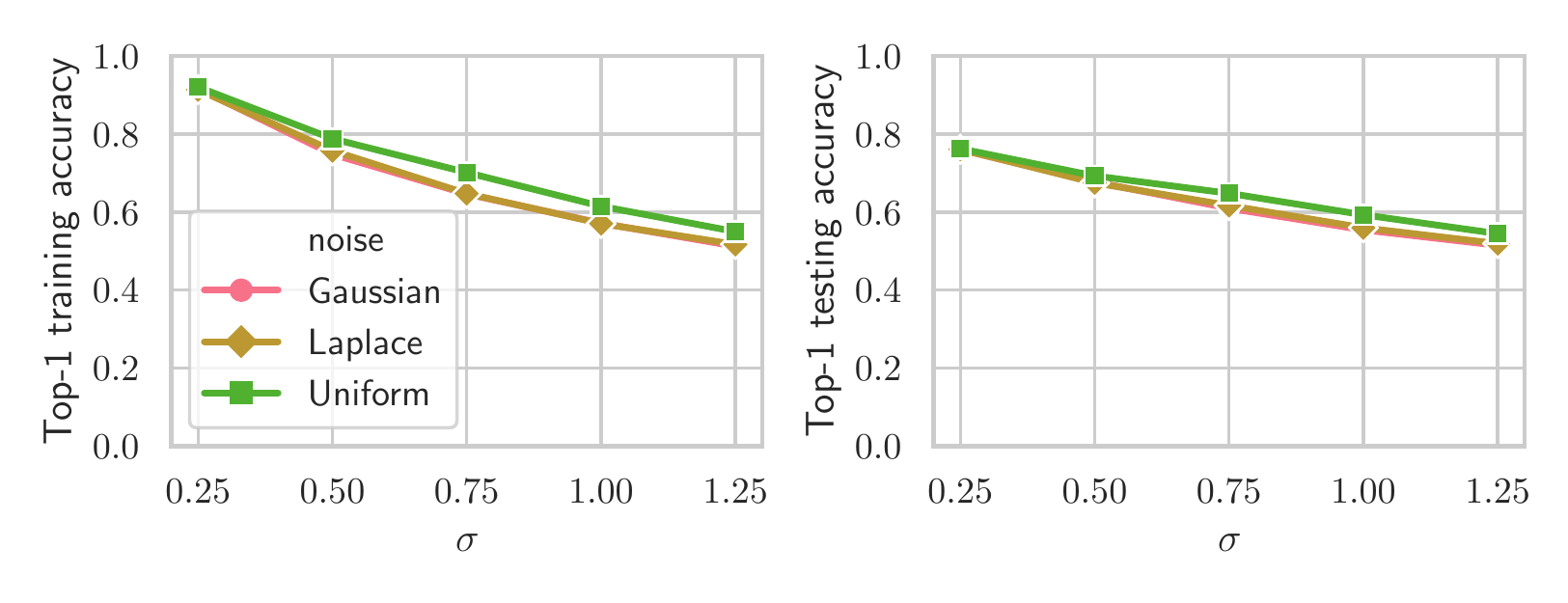}
    \end{subfigure}
    \begin{subfigure}[t]{\linewidth}
        \caption{FCNN}
        \includegraphics[width=\linewidth]{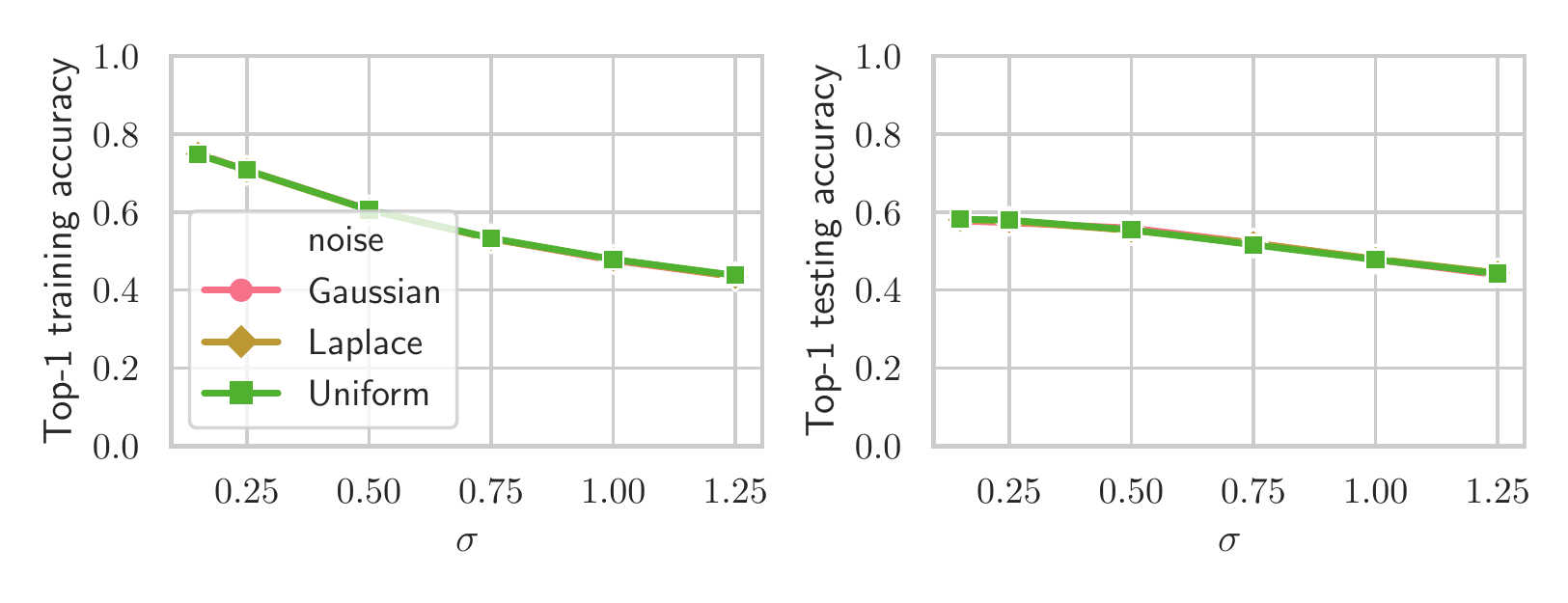}
    \end{subfigure}
    \label{fig:conv_vs_mlp}
\end{figure}

\begin{figure}[t]
    \centering
    \begin{subfigure}[t]{0.49\linewidth}
        \includegraphics[width=\linewidth]{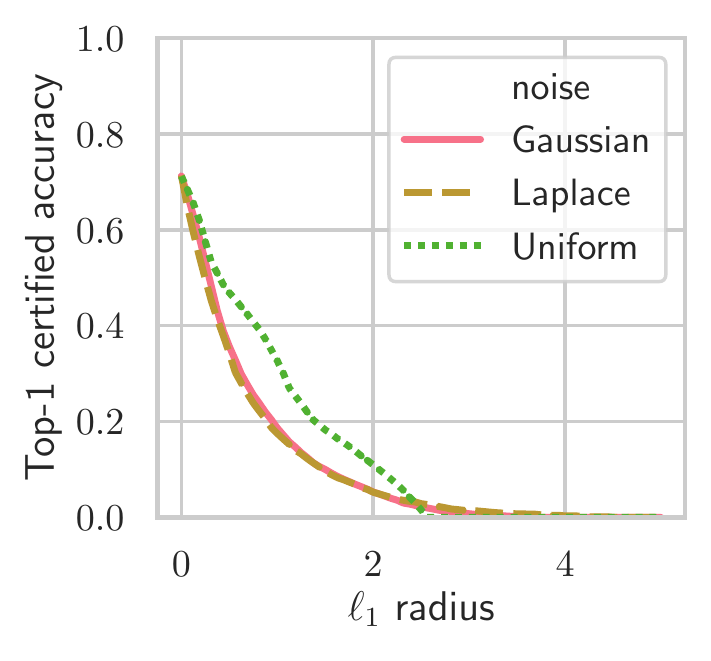}
    \end{subfigure}
    \begin{subfigure}[t]{0.49\linewidth}
        \includegraphics[width=\linewidth]{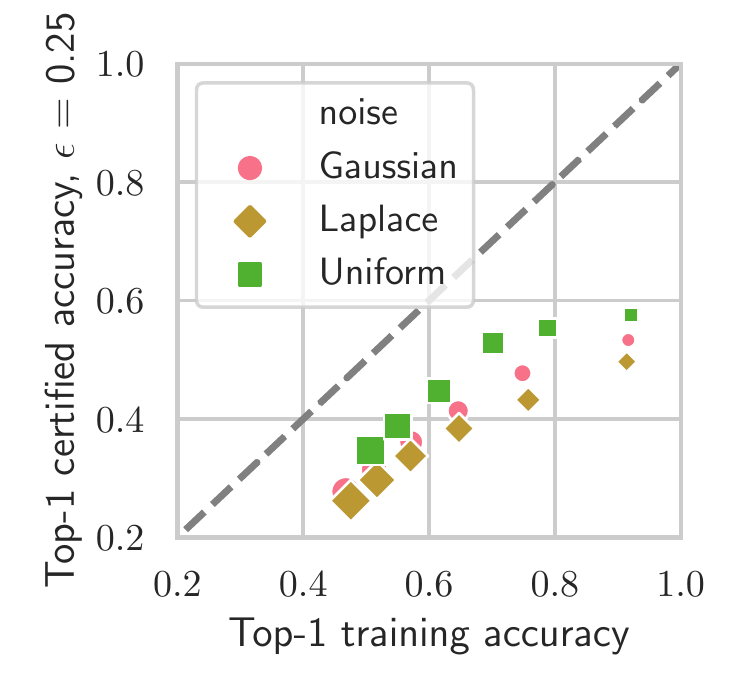}
    \end{subfigure}
    \caption{\textbf{AlexNet.} (Left) CIFAR-10 certified top-1 accuracies for the $\ell_1$ adversary, with an AlexNet architecture. (Right) Certified accuracies at $\epsilon = 0.25$, plotted against training accuracy under noise.}
    \label{fig:alexnet_additional}
\end{figure}

\begin{figure}[t]
    \centering
    \begin{subfigure}[t]{0.49\linewidth}
        \includegraphics[width=\linewidth]{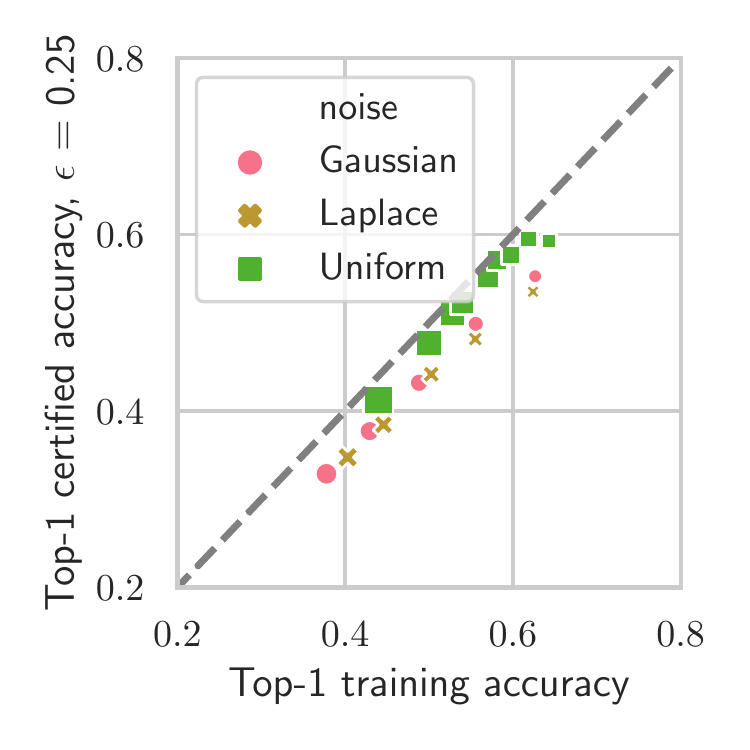}
    \end{subfigure}
    \begin{subfigure}[t]{0.49\linewidth}
        \includegraphics[width=\linewidth]{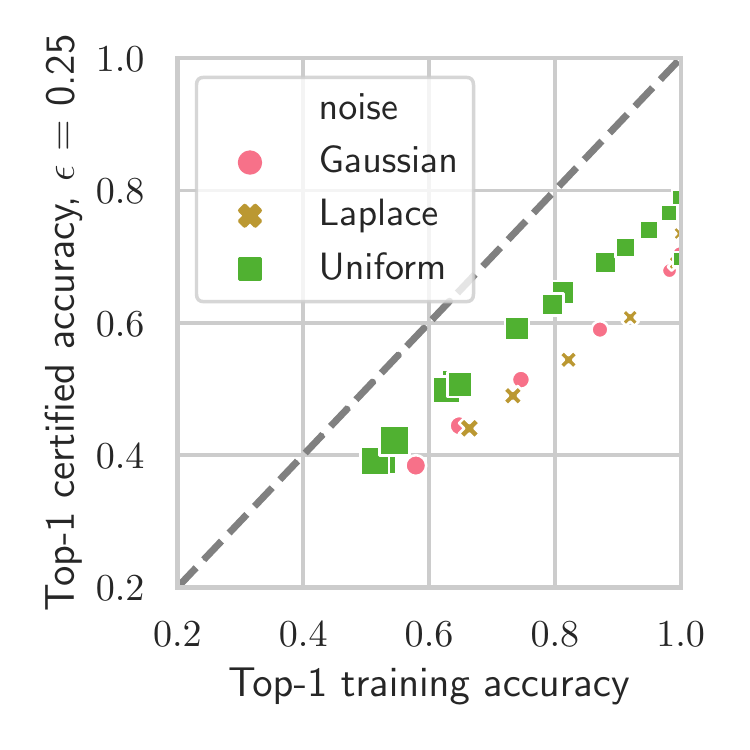}
    \end{subfigure}
    \caption{\textbf{Training Accuracy vs Certified Accuracy.} Top-1 $\ell_1$ certified accuracies for ImageNet (left) and CIFAR-10 (right) at pre-specified $\epsilon=0.25$, controlling for fixed training accuracy.
    Larger sized points denote larger $\sigma$.
    Predictably, as $\sigma$ increases, training and certified accuracy decreases.
    At fixed training accuracy, the Uniform distribution significantly outperforms Gaussian and Laplace.}
    \label{fig:zonotope_evidence_l1}
\end{figure}

\begin{figure}[t]
    \centering
    \begin{subfigure}[t]{\linewidth}
        \includegraphics[width=\linewidth]{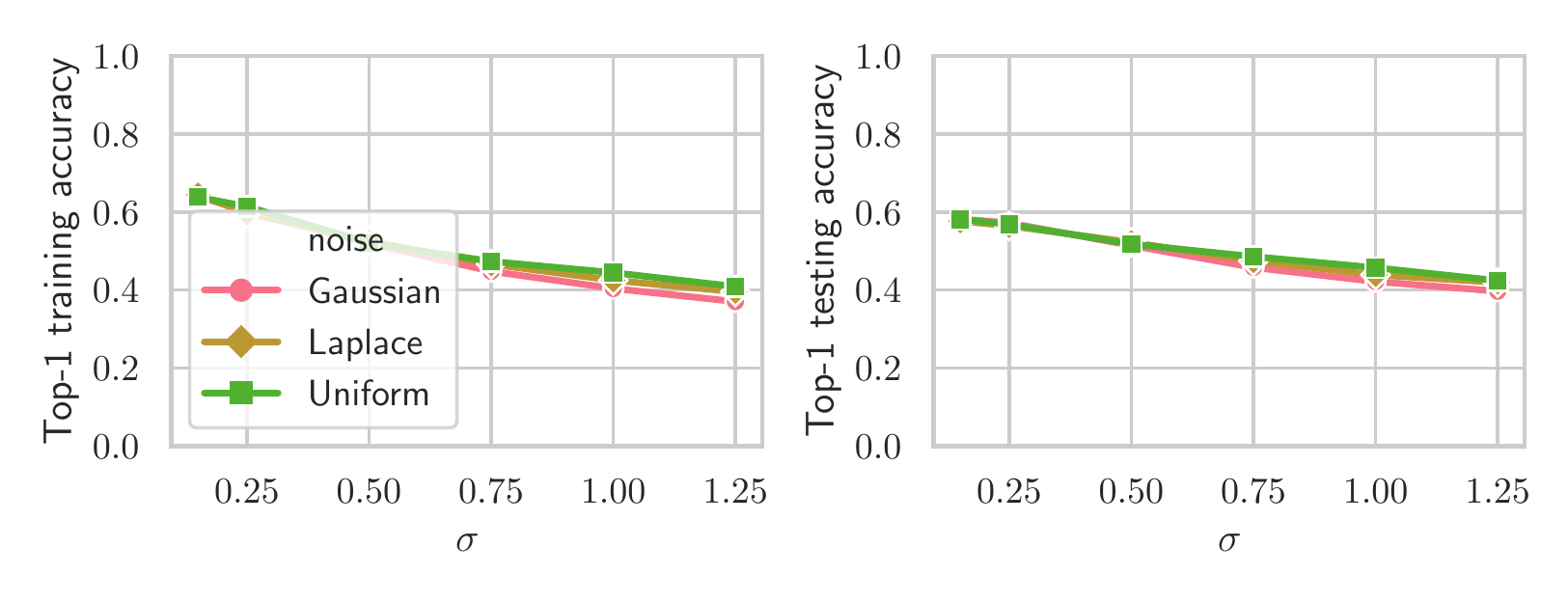}
        \caption{Unmodified noise, rotated images.}
    \end{subfigure}
    \begin{subfigure}[t]{\linewidth}
        \includegraphics[width=\linewidth]{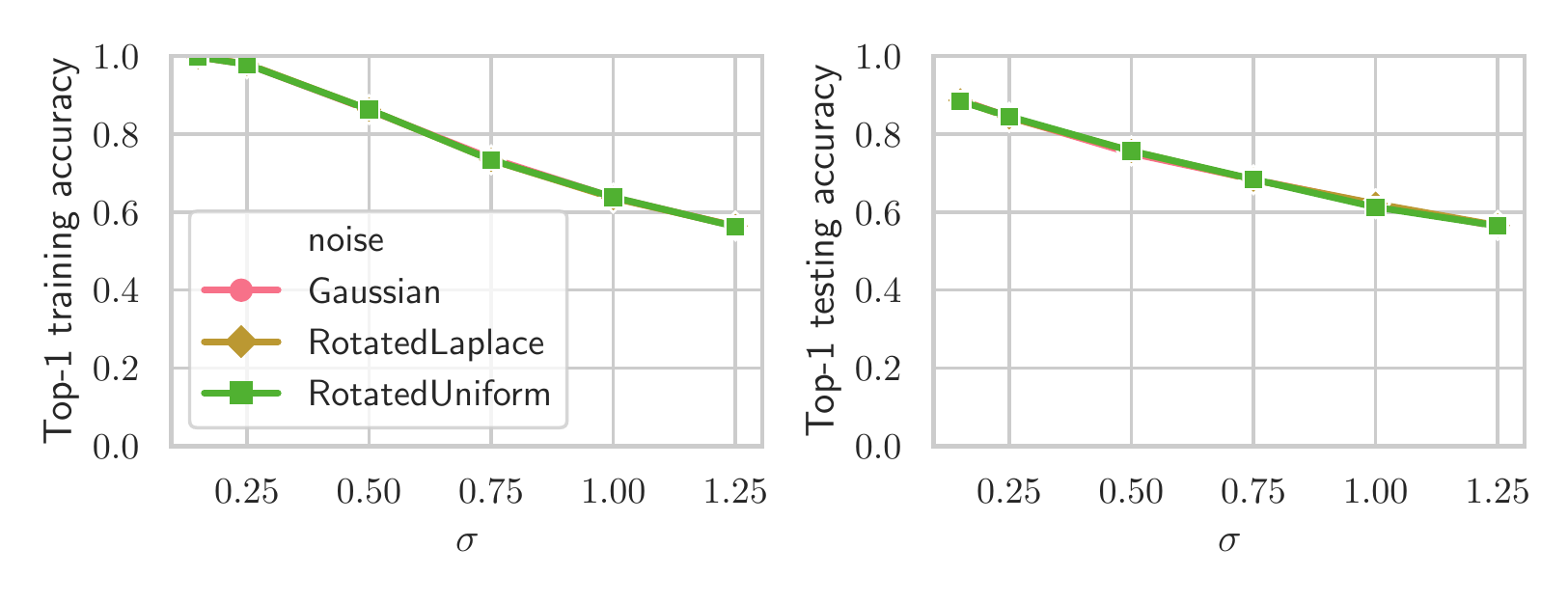}
        \caption{Rotated noise, unmodified images.}
    \end{subfigure}
    \caption{\textbf{Rotation Experiments.} Wide ResNet clean training/testing accuracies in the two rotation experiments.}
    \label{fig:rotation}
\end{figure}

\begin{figure}[t]
    \centering
    \begin{subfigure}[t]{0.49\linewidth}
    \includegraphics[width=\linewidth]{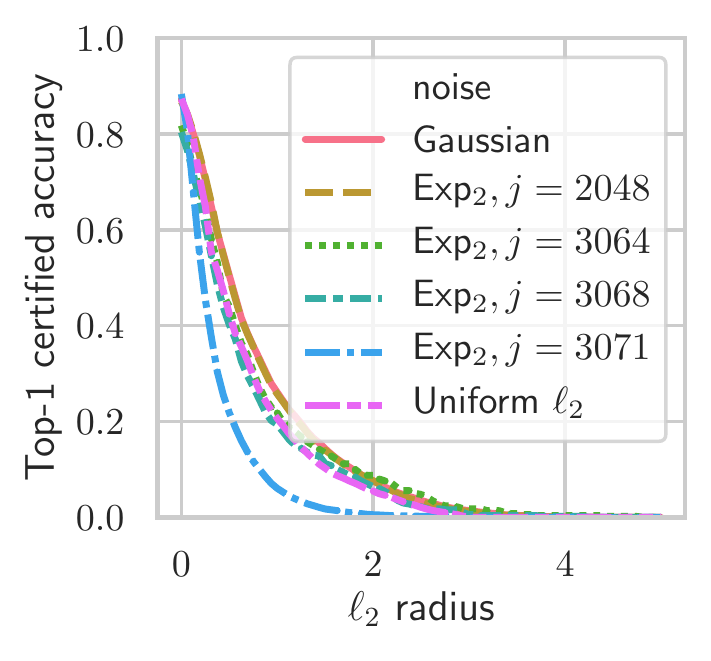}
    \end{subfigure}
    \begin{subfigure}[t]{0.49\linewidth}
    \includegraphics[width=\linewidth]{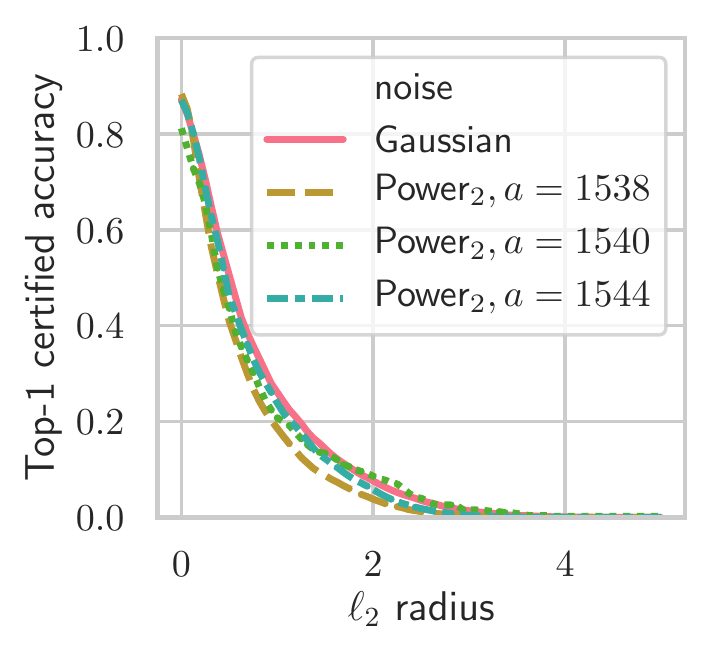}
    \end{subfigure}
    \caption{\textbf{Distributions with Spherical Level Sets.} CIFAR-10 certified top-1 accuracies against the $\ell_2$ adversary, on spherical level set exponential and power law distributions.
    After appropriate hyper-parameter search, performance matches that of the Gaussian distribution.}
    \label{fig:l2_additional}
\end{figure}

\begin{figure}[t]
    \centering
    \includegraphics[width=\linewidth]{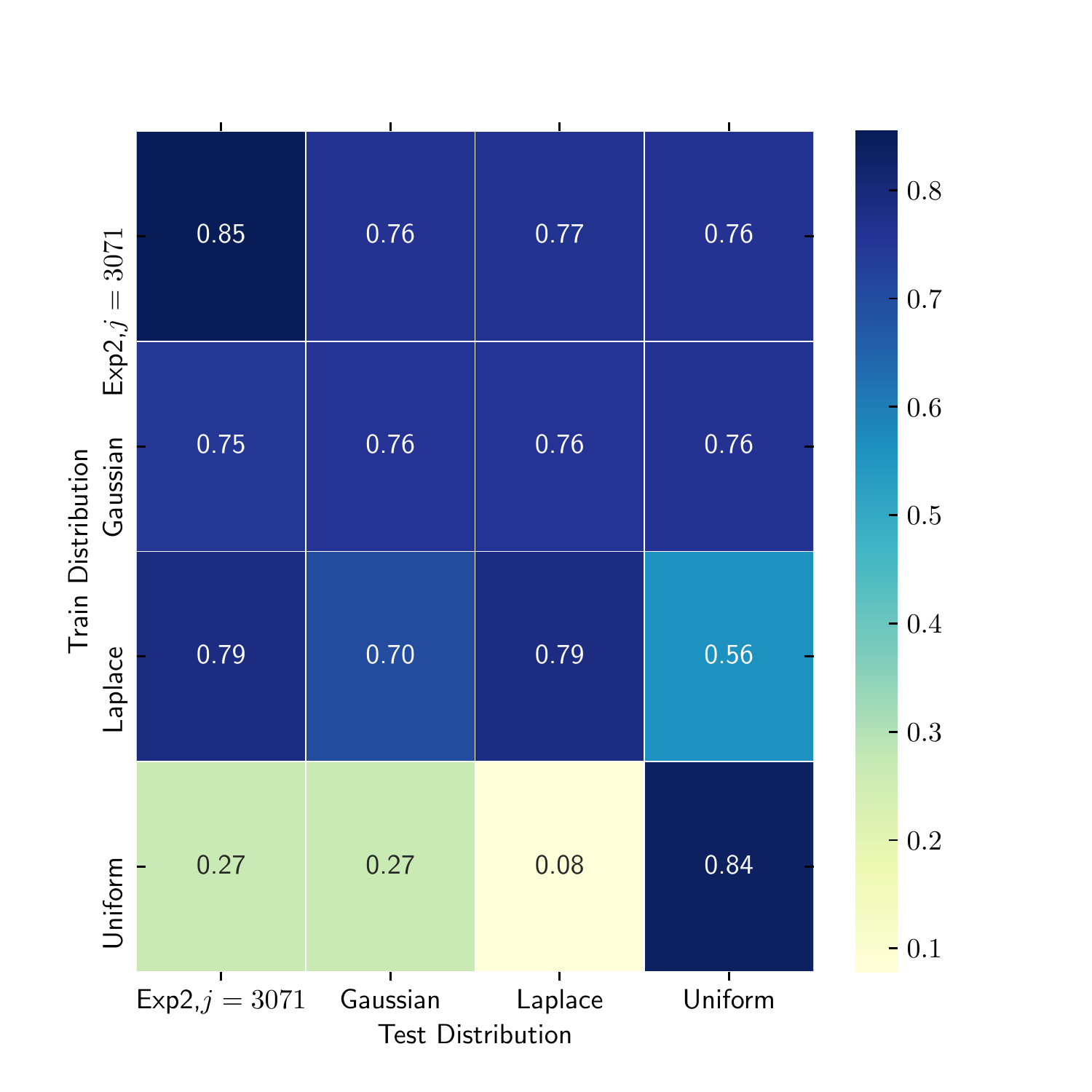}
    \caption{\textbf{Testing on a Different Noise than Trained For.} We compare clean testing accuracies of models (denoted by color) trained on one noise and tested on another, at fixed $\sigma=0.5$.
    We find a model performs best when tested with the same noise for which it was trained.}
    \label{fig:confusion}
\end{figure}

All results in this section are described for CIFAR-10.

\paragraph{$\ell_1$ Adversary}
In addition to the Gaussian, Laplace, and Uniform distributions, we considered an Exponential distribution with cubic level sets, an Exponential distribution with $\ell_1$ level sets, a power law distribution with cubic level sets, and an i.i.d Pareto distribution.
\begin{align*}
    q_{\mathrm{Exp}_\infty}(x) & \propto \exp(-(\|x/\lambda\|_\infty^k))\\
    q_{\mathrm{Exp}_1}(x) & \propto \exp(-(\|x/\lambda\|_1^k))\\
    q_{\mathrm{Power}_\infty}(x) & \propto (1+\|x/\lambda\|_\infty)^{-a}\\
    q_\mathrm{Pareto}(x) & \propto \prod_i\left(1+\frac{|x_i|}{\lambda}\right)^{-(a+1)}
\end{align*}
Results for these experiments are shown in \cref{fig:l1_additional}.
The suffix of the noises in the legend denotes the value of the shape parameter $k$ or $a$ that was chosen (whereas we fixed shape parameter $j=0$).
We note that results for distributions with cubical level sets match but do not exceed that of the Uniform distribution.
Meanwhile distributions without cubical level sets do not match performance of the Uniform distribution.
This suggests that the tail behavior of the noise does not matter as much as the shape of level sets.

\paragraph{Ablation of Our $\ell_1$ Improvement over Previous SOTA}
To understand how much of our $\ell_1$ results come from improved certification vs improved training performance, we repeated our Wide ResNet experiments with a multi-layer perceptron (MLP) and AlexNet.
We find that the Uniform distribution attains a higher upper envelope of certified accuracy than Gaussian or Laplace with this model (\cref{fig:mlp_additional}), but the improvement is less dramatic compared to \cref{tab:sota,fig:l1_certificates}.
Interestingly, the clean (i.e. $\epsilon=0$) training and testing accuracy of all three distributions are identical when fixed to the same level of $\sigma$ for the fully-connected model, but for AlexNet, the Uniform noise allows much higher accuracies (\cref{fig:conv_vs_mlp}), and for Wide ResNet, even more so.
This training improvement leads to substantial improvement in certified accuracies (\cref{fig:alexnet_additional}).

As an additional visualization, when we plot the certified accuracy at fixed $\epsilon$s versus the training accuracy of a Wide ResNet on noise-augmented CIFAR-10, the Uniform distribution can be seen to significantly outperform the Gaussian and Laplace distributions at all training accuracies except those very close to 1 (\cref{fig:zonotope_evidence_l1}).

So while some of the improvement in certified accuracy in \cref{fig:l1_certificates} is due to improved certified radius per $\rho$, it seems much more of it is due to the difference in how well a classifier trains when smoothed by noise.

\paragraph{Why Does Uniform Distribution Get Better Training Accuracy?}

Here we further investigate why improvement in architecture seems to amplify the advantage of uniform distribution over others, in terms of training accuracy for each level of $\sigma$.
Letting $W \in \mathbb{R}^{d,d}$ denote a pre-specified rotation matrix fixed throughout training/testing, we consider:
\begin{enumerate}[nosep]
    \item Smoothing with unmodified noise, rotated images:
    $$x \leftarrow Wx + \delta, \quad\delta\sim q.$$
    \item Smoothing with rotated noise, unmodified images:
    $$x \leftarrow x + W\delta, \quad\delta\sim q.$$
\end{enumerate}
Note that certification bounds are no longer necessarily applicable, so we only compare clean training accuracy i.e. whether $\arg\max_{\mathcal{Y}} g(x) = y$.
Results for Wide ResNet are shown in \cref{fig:rotation}.
We find that the difference in training performance still exists (but to a lesser degree) under alternative (1), smoothing with unmodified noise but rotated images.
On the other hand, we find this difference vanishes under alternative (2), smoothing with rotated noise and unmodified images.

This suggests that the improvement of training accuracy under Uniform noise is due to some \emph{synergy} of the model architecture with the data distribution and the smoothing noise.
The choice of Uniform distribution induces some improvement in training accuracy but this is greatly amplified by the interaction between convolution layers and the image dataset.
Thus, a good noise for randomized smoothing seems to be one that balances its robustness properties with its \emph{compatibility} with the architecture and the data.

\paragraph{$\ell_2$ Adversary}
In addition to the Gaussian distribution, we considered an Exponential distribution with spherical level sets and a power law distribution with spherical level sets.
\begin{align*}
    q_{\mathrm{Exp}_2}(x) & \propto (\|x\|_2/\lambda)^{-j}\exp(-(\|x\|_2^2/\lambda))\\
    q_{\mathrm{Power}_2}(x) & \propto (1+\|x\|_2^2/\lambda)^{-a}
\end{align*}
Results for these experiments are shown in \cref{fig:l2_additional}.
After appropriate hyperparameter search (of $j$ and $a$), performance for both distributions with spherical level sets matches that of the Gaussian.

\paragraph{Does Training and Testing on Different Noises Help?}

One may hope that certifying with a different noise than what a model was trained on may improve performance of the classifier.
For example, $\mathrm{Exp}_2$ noise with large $j$ has more mass concentrated around zero compared to Gaussian noise, and may therefore be easier to ``de-noise''.
In this section we find that training and testing with different noises does \emph{not} improve clean accuracy, when we compare noises at a fixed level of $\sigma=0.5$ (Figure \cref{fig:confusion}).
For all the noises we considered, testing a model with the same noise it was trained upon results in the best clean accuracy.
This suggests the classifier's de-noising process is quite reliant on the properties of the noise to which it is exposed in the training process.
\section{Experimental Details}

\label{sec:experiments_details}

\paragraph{Training Methods}
There are several methods of training a smoothed classifier.
Let $f$ denote the base classifier (up to the logit layer), $q$ denote the smoothing distribution, and consider an observation $(x,y)$.
\begin{enumerate}[nosep]
    \item Noise augmentation as in \citet{cohen_certified_2019},
    $$\mathcal{L}(x,y) = -\log f(x+\delta)_y,\quad \delta \sim q.$$
    \item Directly training the smoothed classifier as described in \citet{salman2019provably} (without adversarial attacks),
    $$\mathcal{L}(x,y) = -\log\mathbb{E}[f(x+\delta)]_y,\quad\delta\sim q.$$
    \item Adversarial training as in \citet{salman2019provably},
    $$\mathcal{L}(x,y) = -\log\mathbb{E}[f(\tilde{x} + \delta)]_y,\quad\delta\sim q.$$
    where $\tilde{x}$ is found via PGD on the smoothed classifier and $\delta$ noise samples are fixed
    throughout the PGD process.
    \item Stability training as in \citet{li_certified_2019},
    $$\mathcal{L}(x,y) = -\log f(x)_y + \gamma D_{\mathrm{KL}}( \sigma(f(x))\ \|\ \sigma(f(x+\delta))$$
    where $\delta \sim q$, $\sigma$ here denotes the softmax function, and $\gamma$ is a hyper-parameter.
\end{enumerate}
Unless otherwise noted, in all experiments we trained with the first option, appropriate noise augmentation.
We found that direct training was slower and did not yield superior performance in practice.
Of these four options we found that stability training with $\gamma=6$ tended to produce the best results (our choice of $\gamma$ follows \citet{carmon_unlabeled_2019}).
Therefore, we re-trained our SOTA models with stability training and list results in Table \ref{tab:sota} and Figure \ref{fig:more_data}.

\paragraph{Range of $\sigma$}
Recall that $\sigma^2 \defeq \mathbb{E}[\frac{1}{d}\|\delta\|_2^2]$.
This is a fairly consistent measurement of noise level across different noise distributions, and is a natural control variate for comparing the effect (e.g.\ training, testing, and certified accuracies) of different noises.
In addition, to obtain a good estimate of the upper envelope of certified accuracy, we need to take the pointwise maximum of the radius-vs-certified-accuracy curve (such as those in \cref{fig:l1_certificates}) for many $\sigma$s.
In this work, we swept over:
\[\sigma \in \{0.15, 0.25, 0.5, 0.75, 1.0, 1.25, 1.50\}.\]
For distributions with cubic level sets, we needed to sweep over larger $\sigma$s as well to estimate the large-radius portion of the upper envelope better:
\[\sigma \in \{1.75, 2.0, 2.25, 2.5, 2.75, 3.0, 3.25, 3.5\}.\]

\cref{tab:distinfo} lists for each distribution the conversion constant needed to obtain $\lambda$ from $\sigma = \sqrt{\EV_{\delta \sim q} \f 1 d \|\delta\|^2_2}$.

\paragraph{Certified Accuracy per $\sigma$}
In \cref{fig:per_sigma} we show the certified accuracies of Gaussian, Laplace and Uniform distributions, for each $\sigma$, for both ImageNet and CIFAR-10.
The upper envelopes reported in the main text are defined as the maximum certified accuracies over $\sigma$.

\paragraph{Experiment Hyperparameters}
For all experiments we trained with a cosine-annealed learning rate of 0.1, optimized by stochastic gradient descent with momentum of 0.9 and weight decay of 0.0001.

For ImageNet experiments we used a ResNet-50 model and trained with a batch size of 64 for 30 epochs.

For CIFAR-10 experiments we used a Wide ResNet 40-2 model and trained with a batch size of 128 for 120 epochs.

Ablation studies with a fully connected neural network employed two hidden layers of 2048 and 512 nodes followed by ReLU activations, trained with a learning rate of 0.01.

To compute the top categories for certification (which used $N=100,000$ samples), we used 64 samples.

Our code is publicly available at: \repo

\paragraph{More Data for Improved Robustness}
We explore using more data to improve the robustness of our SOTA smoothed classifiers for CIFAR-10 in two ways: using \emph{pre-training} as in \citet{hendrycks_using_2019}, and \emph{semi-supervised learning} as in \citet{carmon_unlabeled_2019}. Results are listed in Table \ref{tab:sota} and Figure \ref{fig:more_data}.

Pre-training is inspired by \citet{hendrycks_using_2019}, who showed that pre-training on the large downsampled ImageNet dataset can improve empirical $\ell_\infty$ robustness for CIFAR-10 and CIFAR-100 datasets.
Similarly, our pre-trained models are initially trained on the 1000-class downsampled ImageNet dataset \citep{chrabaszcz_downsampled_2017}.
We then re-initialize the final logit layers for the CIFAR-10 dataset and fine-tune with a learning rate of 0.001.

Semi-supervised learning is inspired by \citet{carmon_unlabeled_2019}, who showed that self-training on the unlabeled 80 Million Tiny Images dataset can improve robustness of CIFAR-10 classifiers.
We use their publicly released dataset of 500k images equipped with pseudo-labels generated by a network trained by CIFAR-10, and train on mini-batches from this dataset and CIFAR-10.

\begin{figure}[t]
    \centering
    \begin{subfigure}[t]{0.49\linewidth}
    \caption{ImageNet Stability}
    \includegraphics[width=\linewidth]{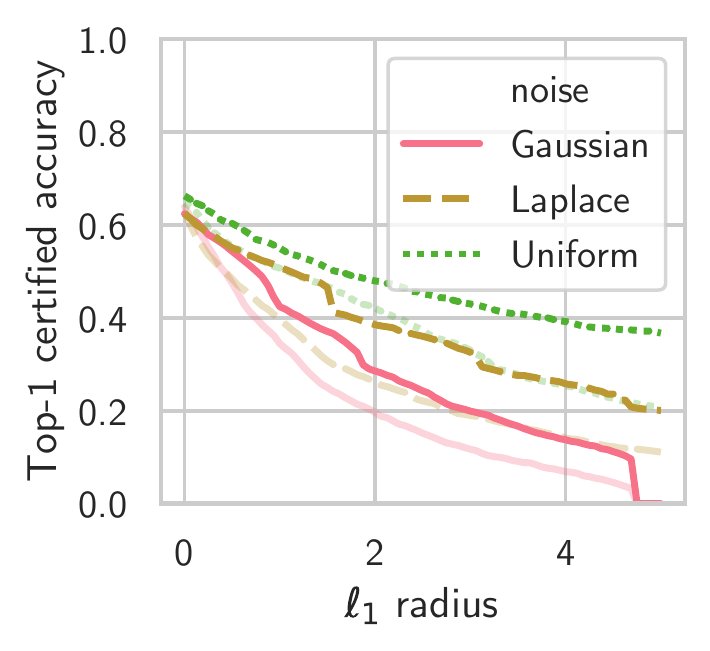}
    \end{subfigure}
    \begin{subfigure}[t]{0.49\linewidth}
    \caption{CIFAR-10 Stability}
    \includegraphics[width=\linewidth]{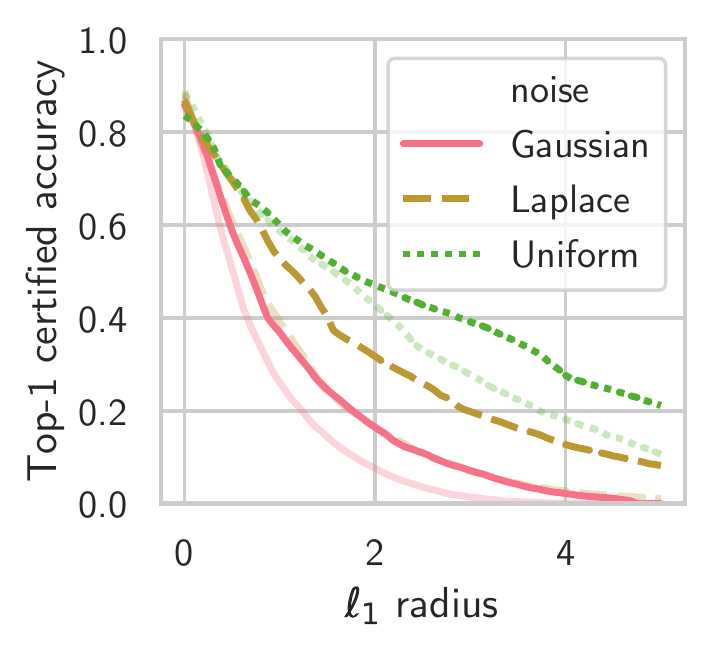}
    \end{subfigure}
    \begin{subfigure}[t]{0.49\linewidth}
    \caption{CIFAR-10 Stab + Semi-sup}
    \includegraphics[width=\linewidth]{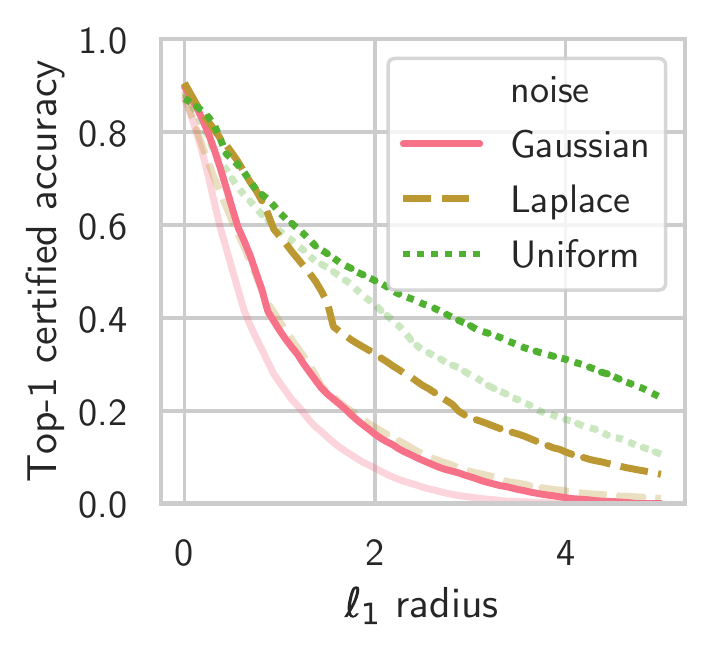}
    \end{subfigure}
    \begin{subfigure}[t]{0.49\linewidth}
    \caption{CIFAR-10 Stab + Pre-train}
    \includegraphics[width=\linewidth]{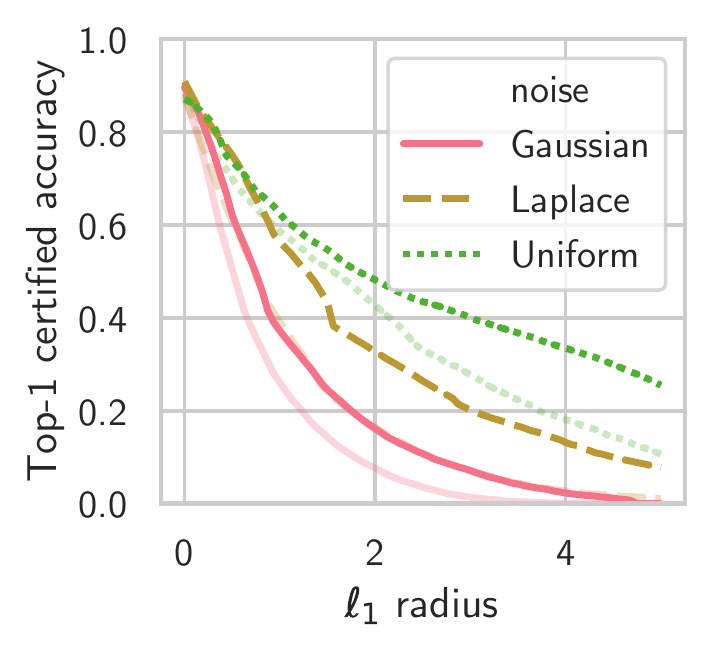}
    \end{subfigure}
    \caption{\textbf{Improved Results with Stability Training and More Data.} Certified top-1 accuracies against the $\ell_1$ adversary, showing improvements due to stability training, pre-training, and semi-supervised self-training.
    Certified accuracies yielded by the default noise augmentation are plotted in faint lines for ease of comparison.}
    \label{fig:more_data}
\end{figure}

\section{Mathematical Preliminaries}

In this section, we rigorously define several mathematical notions and their properties that will recurrent throughout what follows.
We will be brief here, but readers can skip this on first reading and refer back when necessary.
\paragraph{Note about Notation}
We will use $\Vol$ to denote measure, typically Hausdorff measure with the dimension implicit from context.
When integrating over a measurable set, the underlying measure is also typically the Hausdorff measure as well.
By $\pd U$ of a set $U$, we typically mean \emph{reduced boundary} (when $U$ has finite perimeter), especially in a measure-theoretic context.
Readers needing more background can consult \citet{evans2015measure}.

\paragraph{Sobolev Functions and Regular Functions}

While many distributions like Gaussian have continuously differentiable densities, many others, like Laplace, only have ``weak'' derivatives.
Thus, to cover all such distributions, we need to pin down a notion of ``weakly differentiable.''

\begin{defn}\label{defn:sobolev}
    Let $\Omega \sbe \R^d$, and $f: \Omega \to \R$, $g: \Omega \to \R^d$.
    We say $g$ \emph{is a weak derivative of} $f$ if for every smooth function $\phi: \Omega \to \R^d$ with compact support,
    \begin{align*}
        \int f \div \phi = - \int g \cdot \phi.
    \end{align*}
    We write $g = \nabla f$ in this case.

    For any open set $\Omega \sbe \R^d$, the \emph{Sobolev space} $W^{1,p}(\Omega)$ is defined as the functions $f \in L^p(\Omega)$ whose weak derivative exists and is in $L^p(\Omega; \R^d)$, i.e.\
    \begin{align*}
        W^{1, p}(\Omega) \defeq \{f \in L^p(\Omega): \nabla f \in L^p(\Omega; \R^d)\}.
    \end{align*}
\end{defn}

\begin{defn}\label{defn:regular}
Let $\Omega \sbe \R^d$ be an open set.
For the purpose of this paper, we say a function $f: \Omega \to \R$ is \emph{regular} if  $f \in W^{1,1}(\Omega)$.
\end{defn}
This means that $f$ has a weak derivative $\nabla f$ such that both $f$ and $\nabla f$ are integrable.
For example, ReLU is not a continuously differentiable function, but it is regular since it has the Heavyside step function as its weak derivative.

\paragraph{Coarea Formula and the Weak Sard's Theorem}

\begin{thm}[Coarea Formula \citep{federer2014geometric,evans2015measure}]\label{thm:coarea}
Let $\Omega \sbe\R^d$ be an open set, $g \in L^1(\Omega)$, and $f: \R^d \to \R$ regular in the sense of \cref{defn:regular}.
Define $U_t \defeq \{x: f(x) \ge t\}$ to be $f$'s superlevel sets.
Then
\begin{align*}
    \int_\Omega g(x) \|\nabla f(x)\|_2 \dd x
    = \int_{\R} \int_{\pd U_t} g(x) \dd x \dd t.
\end{align*}
Here the integral in $x$ in the RHS is over the $(d-1)$-dimensional Hausdorff measure of the \emph{reduced boundary} of $U_t$, which we abuse notation and denote as $\pd U_t$.
\end{thm}

The regularity of $f$ can be replaced by weaker conditions (see \citet{evans2015measure}), but the statement here suffices for our purposes.

By setting $g$ in \cref{thm:coarea} to be the indicator function over the set where the gradient of $f$ vanishes, we get
\begin{thm}[Weak Sard]\label{thm:regulargradvanish}
For any $f: \R^d \to \R$ regular in the sense of \cref{defn:regular}, let $Z \defeq \{x \in \R^d: \nabla f(x) = 0\}$.
Let $U_t \defeq \{x: f(x) \ge t\}$ denote the superlevel set of $f$ at level $t$.
Then
\begin{align*}
    \Vol(Z \cap \pd U_t) = 0
\end{align*}
for almost every $t \in \R$.
Here $\Vol$ denote the Hausdorff measure of dimension $d-1$, and again $\pd U_t$ denotes reduced boundary.
\end{thm}

\paragraph{Sobolev Functions and Absolute Continuity}

Recall the standard definition of \emph{absolute continuity}, which can be thought of as a more general notion of \emph{differentiability.}
\begin{defn}
A function $f: \R \to \R$ is called \emph{absolute continuous} if there exists a Lebesgue integrable function $g: \R \to \R$ and some $a \in \R$, such that
\[f(x) = f(a) + \int_a^x g(t) \dd t.\]
\end{defn}
Such an $f$ has derivative $f'$ almost everywhere, and $f'$ coincides with $g$ almost everywhere.

Sobolev functions are known to be absolutely continuous on every line, and this property roughly captures all Sobolev functions.
\begin{thm}[ACL Property of Sobolev Functions \citep{nikodym_sur_1933}]\label{thm:sobolevACL}
Let $\Omega \sbe \R^d$ be an open set.
The following statements hold.
\begin{itemize}
    \item Let $f: \Omega \to \R$ be Sobolev, $f \in W^{1,p}(\Omega)$.
    Then possibly after modifying $f$ on a set of measure 0, for every $u \in \R^d$, the function $t \mapsto f(x + tu)$ is absolutely continuous for almost every x.
    Furthermore, the (classical) directional derivative $D_u f$ is in $L^p(\Omega)$ for every $u$.
    \item Conversely, if the restriction of a function $f: \Omega \to \R$ on almost every line parallel to the coordinate axes is absolutely continuous, then pointwise gradient $\nabla f$ exists almost everywhere, and $f \in W^{1, p}(\Omega)$ as long as $f, \nabla f \in L^p(\Omega)$.
\end{itemize}

\end{thm}

The ACL property of Sobolev functions yields the differentiability of the convolution of a $L^\infty$ and a $W^{1, 1}$ function.
\begin{lemma}\label{lemma:convolutionDiff}
If a function $q$ is in $W^{1,1}(\R^d)$, then for every bounded measurable $F: \R^d \to \R$, the convolution $F * q$ is continuously differentiable, and
\[
\nabla (F * q) = F * (\nabla q).
\]
\end{lemma}
\begin{proof}
A function is differentiable if all of 1) its partial derivatives exist and 2) are continuous.
First we show that, for any vector $u$, $F * (D_u q) = D_u (F * q)$.
We can compute as follows.
\begin{align*}
    &\phantomeq
        F * (D_u q)(x)
        \\
    &=
        \int F(\hat x) D_u q(x - \hat x) \dd \hat x\\
    &=
        \left.
        \f d {d\tau} \int_0^\tau \int F(\hat x) D_u q(x - \hat x + t u) \dd \hat x \dd t
        \right|_{\tau = 0}
        \numberthis\label{_eqn:DInt}
        \\
    &=
        \left.
        \f d {d\tau} \int F(\hat x) \int_0^\tau D_u q(x - \hat x + t u) \dd t \dd \hat x
        \right|_{\tau = 0}
        \numberthis\label{_eqn:Fubini}
        \\
    &=
        \left.
        \f d {d\tau} \int F(\hat x) [q(x - \hat x + \tau u) - q(x - \hat x)] \dd \hat x
        \right|_{\tau = 0}
        \numberthis\label{_eqn:AC}
        \\
    &=
        D_u \int F(\hat x) q(x - \hat x) \dd \hat x = D_u (F * q)(x)
        \numberthis\label{_eqn:simpl}
        .
\end{align*}
In these equations, first note that
\begin{align*}
    &\phantomeq
        \int_0^\tau \int |F(\hat x) D_u q(x - \hat x + t u)| \dd \hat x \dd t
        \\
    &\le
        \tau \int |D_u q(\hat x)| \dd \hat x
        < \infty,
\end{align*}
by the ACL property of $q$ (\cref{thm:sobolevACL}).
Thus, in \cref{_eqn:DInt}, we introduced $\f d {d\tau} \int_0^\tau$ innocuously by the fundamental theorem of calculus, since the inner integral is absolutely integrable in $t$.
Then, in \cref{_eqn:Fubini}, we applied Fubini-Tonelli Theorem to swap the order of integration.
In \cref{_eqn:AC}, we integrated out the directional derivative $D_u q$ for almost every $\hat x$ where $t \mapsto q(x - \hat x + tu)$ is absolutely continuous.
Finally, in \cref{_eqn:simpl}, we simplified the integral by noting that $q(x - \hat x)$ does not depend on $\tau$, and $q(x - \hat x + \tau u)$ is absolutely integrable in $\hat x$.
This proves our claim that $F * (D_u q) = D_u (F * q)$ for any $u \in \R^d$.

Note additionally that, since $D_u q \in L^1$ (by assumption) and $F \in L^\infty$, their convolution $F * (D_u q)$ is bounded and continuous.

Then, taking $u$ to be the coordinate vectors, we see the partial derivatives of $F * q$ all exist and are continuous, proving our lemma.
\end{proof}

\section{The Differential Method}

We summarize the setup of this section in the following assumption.
Here we use a notion called \emph{regularity} introduced in \cref{defn:regular} that roughly says that a function needs to continuous almost everywhere and be ``weakly'' differentiable, and it and its gradient are both absolutely integrable.
All concrete density functions we work with in this paper will be regular, with the exception of the uniform distribution.

\begin{assm}\label{assm:FGqsmooth}
Let $F: \R^d \to [0, 1]$ be a measurable function and let $G: \R^d \to [0, 1]$ be the smoothing of $F$ by the distribution $q(x) \propto \exp(-\psi(x))$ for some $\psi: \R^d \to \R$, such that $q$ is \emph{regular} in the sense of \cref{defn:regular}.
Formally,
\begin{align*}
    G(x) = \EV_{\delta \sim q} F(x + \delta)
        &= \int q(\delta) F(x + \delta) \dd \delta\\
        &= \int q(\hat x - x) F(\hat x) \dd \hat x.
\end{align*}
Consider a norm $\| \cdot \|$ with unit ball $\Bb$ that is a convex body.
Let $\Vrt(\Bb)$ be the set of its extremal points.
\end{assm}

\begin{exmp}
If $\| \cdot \| = \| \cdot \|_1$ is the $\ell_1$-norm, then $\Bb$ is what is called the cross-polytope, defined as the convex hull of the unit vectors and their negations.
If $\| \cdot \| = \| \cdot \|_\infty$ is the $\ell_\infty$-norm, then $\Bb$ is the cube with vertices $\{\pm1\}^d$.
If $\| \cdot \| = \| \cdot \|_2$ is the $\ell_2$-norm, then $\Bb$ is the unit sphere, and $\Vrt(\Bb)$ is its entire boundary.
\end{exmp}

\begin{exmp}
If $\psi(x) = \|x\|^2_2$, then $q$ is the standard Gaussian distribution.
If $\psi(x) = \|x\|_1$, then $q$ is the Laplace distribution.
\end{exmp}

The following definition of $\Phi$ turns out to be equivalent to \cref{eqn:PhiMaintext}, which will be apparent in the proof of \cref{thm:differentialMethodMasterTheorem}.
It gives a systematic way of computing $\Phi$.
\begin{defn}\label{defn:Phi}
Let $q(x) \propto \exp(-\psi(x))$ be a distribution over $\R^d$ as in \cref{assm:FGqsmooth}.
For any vector $u \in \R^d$, let $\gamma_u$ be the random variable $\la u, \nabla \psi(\delta) \ra \in \R$ with $\delta \sim q$.
Define $\varphi_u$ to be the complementary CDF of $\gamma_u$,
\begin{align*}
    \varphi_u(c) \defeq \Pr[\gamma_u > c],
\end{align*}
and define the inverse complementary CDF $\inv\varphi_u(p)$ of $\gamma_u$ to be
\begin{align*}
    \inv\varphi_u(p) \defeq \inf\{c: \Pr[\gamma_u > c] \le p\}.
\end{align*}
For any $p \in [0, 1]$, define a new random variable $\mygamma{u}{p}$ by
\begin{align*}
    \mygamma{u}{p} =
    \begin{cases}
    \gamma_u|_{(c, \infty)}
        &   \text{with probability $\varphi_u(c)$}\\
    c
        &   \text{with probability $p - \varphi_u(c)$}\\
    0   &   \text{with probability $1 - p$,}
    \end{cases}
\end{align*}
where $c \defeq \varphi_u^{-1}(p)$ and $\gamma_u|_{(c, \infty)}$ is the random variable $\gamma_u$ conditioned on $\gamma_u > c$.
Roughly speaking, the PDF of $\mygamma{u}{p}$ allocates probability $p$ to the right portion of $\gamma_u$'s PDF,
and puts the rest $1-p$ probability on 0.
One just needs to be careful when $\gamma_u$'s measure has a singular point at $\inv \varphi_u(p)$, which is dealt with in the middle line above.

Let $\Bb$ be the unit ball of $\|\cdot \|$ as in \cref{assm:FGqsmooth}.
Then we define $\Phi: [0, 1] \to \R$ by
\begin{align*}
    \Phi(p) \defeq
        \max_{u \in \Vrt(\Bb)}
            \EV \mygamma{u}{p}
            .
\end{align*}
\end{defn}

\begin{remk}
The function $p \mapsto \bar \EV \mygamma{u}{p}$ in \cref{defn:Phi} is increasing on $[0, \varphi_u(0)]$ and nonincreasing on $[\varphi_u(0), 1]$.
Thus $\Phi(p)$ is also increasing on $[0, \inf_{u \in \Vrt(\Bb)} \varphi_u(0)]$.
\end{remk}

The following theorem is the master theorem for applying the differential method.
We illustrate its usage to recover the known Gaussian \cite{cohen_certified_2019} and Laplace \cite{teng_ell_1_2019} bounds as warmups in \cref{sec:GaussianWarmup,sec:LaplaceWarmup} before applying the technique at scale.
\begin{thm}[The Differential Method]
\label{thm:differentialMethodMasterTheorem}
As in \cref{assm:FGqsmooth}, fix any norm $\| \cdot \|$ and
let $G: \R^d \to [0, 1]$ be the smoothing of any measurable $F: \R^d \to [0, 1]$ by $q(x) \propto \exp(-\psi(x))$, such that $q$ is regular in the sense of \cref{defn:regular}.
Let $\Phi: [0, 1] \to \R$ be given as in \cref{defn:Phi}.

Then for any $x$, if $G(x) < 1/2$, then $G(x + \delta) < 1/2$ for any
\begin{align}
    \| \delta \| < \int_{G(x)}^{1/2} \f 1 {\Phi(p)} \dd p
    .
    \label{eqn:robustrad}
    \tag{$\star$}
\end{align}
\end{thm}

In \cref{thm:differentialMethodMasterTheorem}, one should think of $G(x)$ as the probability that the smoothed classifier assigns to \emph{any class other than the} correct one.
So \cref{thm:differentialMethodMasterTheorem} says that, if the smoothed classifier predicts the correct class ($G(x) < 1/2$), then it continues to do so even when the input is perturbed by a noise with magnitude bounded by \cref{eqn:robustrad}.

Sometimes, when $\varphi_u$ is continuous for all $u$, for $p \in [0, 1/2]$, we can factor
\[\Phi(p) = \bar \varphi_u (\inv \varphi_u(p)),\quad\text{where}\quad
\bar \varphi_u(c) \defeq \EV_{\gamma_u} \gamma_u \ind(\gamma_u > c),\]
for some specific $u \in \Vrt(\Bb)$, either due to symmetry in the vertices of $\Bb$ (so that it doesn't matter which $u$ it is) or because a specific $u$ maximizes the expression for all $p \in [0, 1/2]$.
Then the following lemma is very useful for simplifying the integral in \cref{thm:differentialMethodMasterTheorem}.
It can be proved easily using change of coordinates.
\begin{lemma}\label{lemma:reparam}
Suppose $\Phi(p) = \bar \varphi(\inv \varphi(p))$ on $p \in [0, 1/2]$, where $\varphi(p)$ is differentiable and both $\varphi$ and $\bar \varphi$ are nonincreasing.
Then for any $0 \le p_0 \le 1/2$,
\begin{align*}
    \int_{p_0}^{1/2} \f 1 {\Phi(p)} \dd p
    &=
        \int^{\varphi^{-1}(p_0)}_{\varphi^{-1}(1/2)} \f {|\varphi'(c)|} {\bar \varphi(c)} \dd c
\end{align*}
\end{lemma}

Finally, as mentioned before, the proof of \cref{thm:differentialMethodMasterTheorem} will show that
\begin{prop}\label{prop:PhiDefnsSame}
The definition of $\Phi$ in \cref{defn:Phi} coincides with the definition \cref{eqn:PhiMaintext} for any smoothing distribution $q$ with regular density function supported everywhere in $\R^d$.
\end{prop}

\begin{proof}[Proof of \cref{thm:differentialMethodMasterTheorem}]
Consider a path $\xi_t: [0, \|\delta\|] \to \R^d$ given by $\xi_0 = x$, $\xi_{\|\delta\|} = x + \delta$, and $\xi'_t = d\xi_t/dt= \delta/\|\delta\|$.
We will show
\begin{align*}
    d G(\xi_t)/dt \le \Phi(G(\xi_t))
\end{align*}
and apply \cref{lemma:ODErobustness} to yield the desired result.

By chain rule,
\begin{align*}
    d G(\xi_t)/dt = \xi'_t \cdot \nabla G(\xi_t) = \f \delta {\|\delta\|} \cdot \nabla G(\xi_t).
\end{align*}
To upper bound this quantity, we relax
\begin{align*}
    \f \delta {\|\delta\|} \cdot \nabla G(\xi_t)
    &\le
        \max_{u \in \Bb} u \cdot \nabla G(\xi_t)
        = \max_{u \in \Vrt(\Bb)} u \cdot \nabla G(\xi_t)
\end{align*}
where $\Bb$ is the unit ball of the norm $\| \cdot \|$, and the equality is because $u \cdot \nabla G(x)$ is linear in $u$, so optima are achieved on vertices.
Therefore, it suffices to show that,
\begin{equation}
\text{$\forall u \in \Vrt(\Bb), x \in \R^d$,}\qquad u \cdot \nabla G(x) \le \Phi(G(x)).
\label{eqn:differentialBound}
\end{equation}

Below, we let $x$ be any vector in $\R^d$ (not just those satisfyiing $G(x) \le 1/2$ as in the theorem statement).
In general, for any vector $u$ and any $x \in \R^d$, the directional derivative $u \cdot \nabla G(x)$ of $G(x)$ in the direction of $u$ is given by
\begin{align*}
    u \cdot \nabla G(x)
        &=
            u \cdot \int \nabla_x q(\hat x - x) F(\hat x) \dd \hat x
            \\
        &=
            \int \la u, \nabla \psi(\hat x - x) \ra
                q(\hat x - x) F(\hat x) \dd \hat x
\end{align*}
where we used \cref{lemma:convolutionDiff} and the assumption that $q$ is regular.
Then
\begin{align*}
    u \cdot \nabla G(x)
        &=
            \EV_{\delta \sim q} F(x + \delta)\la u, \nabla \psi(\delta)\ra
            \\
        &\le
            \sup_{\hat F: \hat G(x) = G(x)}
            \EV_{\delta \sim q} \hat F(x + \delta)\la u, \nabla \psi(\delta)\ra
            ,
\end{align*}
where we vary over all $\hat F: \R^d \to [0, 1]$ such that its smoothing $\hat G$ has the same value as $G$ at $x$.
While at first glance, this seems like a unwieldy quantity to maximize, there's a simple intuition to find the maximizing $\hat F$:

Imagine $\hat F(x + \cdot)$ as some allocation of mass in $\R^d$ that amounts to $G(x)$ under the measure $q$.
When we vary $\hat F$, we are allowed to shuffle this mass around while keeping its $q$-measure equal to $G(x)$, as long as $0 \le \hat F \le 1$.
To maximize $\EV_{\delta \sim q} \hat F(x + \delta)\la u, \nabla \psi(\delta)\ra$, we then need to allocate as much $q$-measure as possible toward regions where $\la u, \nabla \psi(\cdot)\ra$ is large.

In other words, the maximizing $\hat F$, which we denote as $\hat F^*$, is
\[\hat F^*(x + \delta) =
\begin{cases}
1 & \text{if $\la u, \nabla \psi(\delta)\ra > \varphi_u^{-1}(G(x))$}\\
0 & \text{else,}
\end{cases}
\]
if $\Pr[\la u, \nabla \psi(\delta)\ra = \inv \varphi_u(G(x))] = 0$,
where $\inv \varphi_u$ is the inverse complementary CDF of the random variable $\gamma_u = \la u, \nabla \psi(\delta)\ra$ (with randomness induced by $\delta \sim q$), as defined in \cref{defn:Phi}.
If there is a singular point at $\inv \varphi_u(G(x))$, i.e.\ $\Pr[\la u, \nabla \psi(\delta)\ra = \inv \varphi_u(G(x))] > 0$, then we choose a subset of $U \sbe \{ \delta: \la u, \nabla \psi(\delta) = \inv \varphi_u(G(x)) \}$ with $q$-measure $\Pr[\delta \in U] = G(x) - \varphi_u(\inv \varphi_u(G(x))$, and define $\hat F^*$ as
\[\hat F^*(x + \delta) =
\begin{cases}
1 & \text{if $\la u, \nabla \psi(\delta)\ra > \varphi_u^{-1}(G(x))$ or $\delta \in U$}\\
0 & \text{else.}
\end{cases}
\]
Then
\begin{align*}
    &\phantomeq\EV_{\delta \sim q}
        \hat F(x + \delta)\la u, \nabla \psi(\delta)\ra
        \\
    &\le
        \EV_{\delta \sim q}
            \hat F^*(x + \delta)\la u, \nabla \psi(\delta)\ra
    =
        \EV \gamma_u^{G(x)}
        ,
\end{align*}
where $\mygamma{u}{p}$ is the random variable defined in \cref{defn:Phi}.
Finally, putting everything together,
\begin{align*}
    &\phantomeq
        \max_{u \in \Vrt(\Bb)} u \cdot \nabla G(x)
        \\
    &\le
        \max_{u \in \Vrt(\Bb)}
        \sup_{\hat F: \hat G(x) = G(x)}
            \EV_{\delta \sim q} \hat F(x + \delta)\la u, \nabla \psi(\delta)\ra
        \\
    &=
        \max_{u \in \Vrt(\Bb)}
        \EV \gamma^{G(x)}_u
        \\
    &=
        \Phi(G(x))
\end{align*}
by the definition of $\Phi$ in \cref{defn:Phi}.
This shows \cref{eqn:differentialBound} and consequently the theorem as well.
\end{proof}

\begin{lemma}\label{lemma:ODErobustness}
Consider a function $p_t$ differentiable in $t \in [0, \infty)$.
Suppose $0 < p_0 \le 1/2$, and
\begin{align*}
    dp_t/dt \le \Phi(p_t)
\end{align*}
for some function $\Phi: (0, \infty) \to \R^+$ taking only positive values,
Then $p_T < 1/2$ as well for any
\begin{align*}
    T < \int_{p_0}^{1/2} \f 1 {\Phi(p)} \dd p
    .
\end{align*}
\end{lemma}
\begin{proof}
WLOG, we can assume that $dp_t/dt > 0$ for all $t \in [0, \infty)$.
Thus, $p_t$ is increasing in $t$, and there exists a differentiable inverse function $t(p)$ that expresses the time $t$ that $p_t = p$.
We then have $dt(p)/dp = \f{1}{dp_t/dt} = \f1{\Phi(p)}$, and for any $\epsilon \ge 0$,
\begin{align*}
    t(1/2 - \epsilon) &= t(1/2 - \epsilon) - t(p_0)\\
        &= \int_{p_0}^{1/2-\epsilon} \f{dt(p)}{dp} \dd p
        = \int_{p_0}^{1/2-\epsilon} \f1{\Phi(p)} \dd p
        .
\end{align*}
Since this integral is continuous in $\epsilon$, there is an $\epsilon^* > 0$ such that
\begin{align*}
    t(1/2 - \epsilon^*) = \int_{p_0}^{1/2-\epsilon^*} \f1{\Phi(p)} \dd p = T
    .
\end{align*}
Therefore $p_T = 1/2 - \epsilon^* < 1/2$, as desired.
\end{proof}

\subsection{Example: Gaussian against \texorpdfstring{$\ell_2$}{L2} Adversary}
\label{sec:GaussianWarmup}

We give a quick example of recovering the tight Gaussian bound of \citet{cohen_certified_2019} using the differential method.

In this section, we set the norm $\| \cdot \|$ to be the $\ell_2$ norm $\|x\|_2 = \sqrt{\sum_{i=1}^d x_i^2}$.
Then $\Bb$ is just the unit ball, and its ``vertices'' are just the points on the unit sphere.
Additionally, we let $q$ be the Gaussian measure
\begin{gather*}
    q(x) \propto \exp(-\|x\|_2^2/2)\\
    \text{so that}
    \quad
    \psi(x) = \|x\|_2^2/2
    \quad\text{and}\quad
    \nabla \psi(x) = x.
\end{gather*}

Below, let $\GaussianCDF$ be the CDF of the standard Gaussian in 1D.
\begin{thm}\label{thm:gaussianl2}
Suppose $H$ is a smoothed classifier smoothed by the Gaussian distribution
\[q(x) \propto \exp(-\|x\|_2^2/2\sigma^2),\]
such that $H(x) = (H(x)_1, \ldots, H(x)_C)$ is a vector of probabilities that $H$ assigns to each class $1, \ldots, C$.
If $H$ correctly predicts the class $y$ on input $x$, and the probability of the correct class is $\rho \defeq H(x)_y > 1/2$, then $H$ continues to predict the correct class when $x$ is perturbed by any $\eta$ with
\begin{align*}
    \|\eta\|_2 < \sigma \GaussianCDF^{-1}(\rho).
\end{align*}
\end{thm}

\begin{proof}

By linearity in $\sigma$, it suffices to show this for $\sigma = 1$.
For brevity, let us denote $\GaussianCDF$ in this proof by $\Psi$.

\renewcommand{\GaussianCDF}{\Psi}

We seek to apply \cref{thm:differentialMethodMasterTheorem} to $G(x) = 1 - H(x)_y$, for which we need to derive random variables $\gamma_u$ and $\mygamma{u}{p}$, and most importantly, the function $\Phi$.

For any $u \in \Vrt(\Bb)$ (i.e.\ any unit vector $u$), $\gamma_u = \la u, \nabla \psi(\delta) \ra = \la u, \delta \ra, \delta \sim q,$ is a standard Gaussian random variable (in $\R$).
Therefore, for $p \in [0, 1]$, the random varible $\mygamma{u}{p}$ defined in \cref{defn:Phi} is just
\begin{align*}
    \mygamma{u}{p} =
    \begin{cases}
        0   &   \text{with prob.\ $1 - p$}\\
        \Gaus(0, 1)|_{[c, \infty)} & \text{with prob.\ $p$,}
    \end{cases}
\end{align*}
where $c \defeq \GaussianCDF^{-1}(1-p)$, and $\Gaus(0, 1)|_{[c, \infty)}$ is a standard Gaussian $z$ conditioned on $z \ge c$.
Thus, for any $u \in \Vrt(\Bb)$,
\begin{align*}
    \Phi(p) &= \EV \mygamma{u}{p}
    = \EV_{z \sim \Gaus(0, 1)} z \ind(z \ge c)
    = \f 1 {\sqrt{2\pi}} e^{-z^2/2}\bigg|^c_\infty
        \\
    &= \GaussianCDF'(c)
    = \GaussianCDF'(\GaussianCDF^{-1}(1-p))
        .
\end{align*}
Then, by setting $G(x)$ in \cref{thm:differentialMethodMasterTheorem} to be $1 - H(x)_y = 1 - \rho$, we get the provably robust radius of
\begin{align*}
        \int_{1-\rho}^{1/2} \f 1 {\Phi(p)} \dd p
    &=
        \int_{1-\rho}^{1/2} \f 1 {\GaussianCDF'(\GaussianCDF^{-1}(1-p))} \dd p
        \\
    &=
        \int^{\GaussianCDF^{-1}(\rho)}_{0}
            \dd c
    =
        \GaussianCDF^{-1}(\rho),
\end{align*}
as desired.
\end{proof}

\subsection{Example: Laplace against \texorpdfstring{$\ell_1$}{L1} Adversary}
\label{sec:LaplaceWarmup}

Let us give another quick example of recovering the tight Laplace bound of \citet{teng_ell_1_2019} using the differential method.

In this section, we set the norm $\| \cdot \|$ to be the $\ell_1$ norm $\|x\|_1 = \sum_{i=1}^d |x_i|$.
Then the unit ball $\Bb$ is the convex hull of its vertices which are the coordinates vectors and their negations:
\[\Vrt(\Bb) = \{\pm e_i: i \in [d]\}.\]

Consider the Laplace distribution
\begin{gather*}
    q(x) \propto \exp(-\|x\|_1)
    \quad
    \text{so that}\\
    \psi(x) = \|x\|_1
    \quad\text{and}\quad
    \nabla \psi(x) = (\sgn(x_1), \ldots, \sgn(x_d)),
\end{gather*}
with $\nabla \psi(x)$ defined whenever all $x_i$s are nonzero.

\begin{thm}\label{thm:laplaceL1}
Suppose $H$ is a smoothed classifier smoothed by the Laplace distribution
\[q(x) \propto \exp(-\|x\|_1/\lambda),\]
such that $H(x) = (H(x)_1, \ldots, H(x)_C)$ is a vector of probabilities that $H$ assigns to each class $1, \ldots, C$.
If $H$ correctly predicts the class $y$ on input $x$, and the probability of the correct class is $\rho \defeq H(x)_y > 1/2$, then $H$ continues to predict the correct class when $x$ is perturbed by any $\eta$ with
\begin{align*}
    \|\eta\|_1 < \lambda \log \f 1 {2(1 - \rho)}
                    .
\end{align*}
\end{thm}

\begin{proof}
By linearity in $\lambda$, it suffices to show this for $\lambda = 1$.

We seek to apply \cref{thm:differentialMethodMasterTheorem} to $G(x) = 1 - H(x)_y$, for which we need to derive random variables $\gamma_u$ and $\mygamma{u}{p}$, and most importantly, the function $\Phi$.

For any $u \in \Vrt(\Bb)$ (i.e.\ $u = \pm e_i$), $\gamma_u = \la u, \nabla \psi(\delta) \ra, \delta \sim q,$ is a Rademacher random variable that takes values 1 and $-1$ with equal probability.
Therefore, for $p \in [0, 1/2]$, the random variable $\mygamma{u}{p}$ defined in \cref{defn:Phi} is
\begin{align*}
    \mygamma{u}{p} = \begin{cases}
                1 & \text{with prob.\ $p$}\\
                0 & \text{with prob.\ $1 - p$.}
            \end{cases}
\end{align*}
Thus, for any $u \in \Vrt(\Bb)$,
\begin{align*}
    \Phi(p) = \EV \mygamma{u}{p} = p.
\end{align*}
Then, by setting $G(x)$ in \cref{thm:differentialMethodMasterTheorem} to be $1 - H(x)_y = 1 - \rho$, we get the provably robust radius of
\begin{align*}
    \int_{1-\rho}^{1/2} \f 1 {\Phi(p)} \dd p =
    \int_{1-\rho}^{1/2} \f 1 {p} \dd p
    = \log \f{1}{2(1-\rho)}
    .
\end{align*}
\end{proof}

\section{Wulff Crystal}

The following is an intuitive statement of the main isoperimetric property of Wulff Crystals.
\begin{thm}[Isoperimetric property of Wulff Crystals \citep{brothersIsoperimetricTheoremGeneral1994DOI.orgCrossref}, informal statement]\label{thm:wulff}
Let $\|\cdot\|$ be any norm on $\R^n$.
Let $Z$ be the Wulff Crystal of $\|\cdot\|$, i.e.\ the unit ball of the norm $\|\cdot\|_*$ dual to $\|\cdot\|$.
\begin{align*}
    Z = \{x: \|x\|_* \le 1\}.
\end{align*}
Let $\Omega$ be any measurable subset of $\R^n$ with finite perimeter and of the same volume as $Z$.
Then with $\mathbf n(\Omega, x)$ denoting the normal vector at $x$ with respect to $\Omega$, normalized to have $\ell_2$ norm 1,
\begin{align*}
    \int_{\pd \Omega} \|\mathbf n(\Omega, x)\| \dd x
    \ge
    \int_{\pd Z} \|\mathbf n(Z, x)\| \dd x
\end{align*}
with equality holding if and only if $\Omega$ differs from a translate of $Z$ by a set of volume zero.
\end{thm}

This statement carries across the core essence of the isoperimetry, and is a rigorous statement if $\Omega$ is restricted to have smooth boundary, but care needs to be taken to explain the concept of ``finite perimeter,'' ``normal vector,'' the ``boundary $\pd \Omega$,'' and the boundary measure on $\pd \Omega$, in the context of general, measurable $\Omega$.
These quantities are defined in \cref{sec:BV}, but we also refer the interested reader to \cite{brothersIsoperimetricTheoremGeneral1994DOI.orgCrossref} for more mathematical details.

\subsection{Wulff Crystals are Zonotopes}

In this paper, the norm $\|\cdot \|$ referred to in \cref{thm:wulff} will take the form of $\|x\| = \EV_v |\la x, v\ra|$ where $v$ is sampled from some distribution, as in \cref{defn:wulff}.
When $v$ is sampled uniformly from some finite set of vectors $\mathcal S$, then the Wulff Crystal $B = \{x: \|x\|_* \le 1\}$ is proportional to the \emph{Zonotope} of $\mathcal S$ \citep{mcmullen1971zonotopes}.
\begin{defn}[Zonotope]
Given a finite collection of vectors $\mathcal S$, the zonotope $\Zon(\mathcal S)$ is defined as the Minkowski sum of the vectors of $\mathcal S$, i.e.\
\begin{align*}
    \Zon(\mathcal S) \defeq \left\{\sum_{v \in \mathcal S} a_v v : a_v \in [0, 1], \forall v \in \mathcal S\right\}.
\end{align*}
\end{defn}

The zonotope can be viewed as a linear projection of the cube $[0, 1]^{\mathcal S}$ sending each unit vector to a vector of $\mathcal S$.

\begin{exmp}
If $\mathcal B$ is the $\ell_1$ unit ball, then $\Zon(\Vrt(\mathcal B))$ is a cube.
If $\mathcal B$ is the $\ell_\infty$ unit ball, then in 2 dimensions, $\Zon(\Vrt(\mathcal B))$ is a rhombus; in 3 dimensions, it is rhombic dodecahedron; in higher dimensions, there is no simpler description of the resulting polytope.
\end{exmp}

\begin{prop}\label{prop:wulffzonotope}
The Wulff Crystal w.r.t.\ $\mathcal B$ is equal to the zonotope $\f{2}{|\Vrt(\mathcal B)|} \Zon(\Vrt(\mathcal B))$.
\end{prop}

The volume of a zonotope, and thus of Wulff Crystals, can be computed easily using the following formula.

\begin{prop}\label{prop:zonvol}
Let $\mathcal S$ be a finite set of vectors in $\R^d$.
Then the $d$-dimensional volume of $\Zon(\mathcal S)$ is given by \[\sum_{\mathcal T \sbe \mathcal S: |\mathcal T| = d} |\Vol(\Zon(\mathcal T))| = \sum_{\mathcal T\sbe \mathcal S: |\mathcal T| = d} |\det \mathcal T|,\]
where $\det \mathcal T$ is the determinant of the square matrix with vectors of $\mathcal T$ as columns.
\end{prop}

\subsection{Wulff Crystals Yield Optimal Uniform Distributions for Randomized Smoothing}

In this section, we will formulate \cref{thm:wulffOptimalMainText} rigorously and prove it.

\begin{defn}
Let $\mathcal S$ be a finite set of vectors in $\R^d$ and let $G$ be the group of linear transformations that permute $\mathcal S$ (i.e.\ $G$ is $\mathcal S$'s linear symmetry group).
We say $\mathcal S$ is \emph{symmetric} if $G$ acts on $\mathcal S$ transitively, i.e.\ for any two elements $v,w \in \mathcal S$, there is a group element $g \in G$ such that $g \cdot v = w$.
\end{defn}

For example, the boolean cube $\{\pm1\}^d$ is symmetric, and so is the set of coordinate vectors and their negations.
The following is the main theorme of this section, stating the optimality of uniform distributions suppoorted Wulff Crystals.

\begin{thm}\label{thm:wulffOptimalFull}
If $\mathcal B$ is a full-dimensional polytope in $\R^d$ symmetric around the origin, and whose vertices form a symmetric set, then the Wulff Crystal w.r.t.\ $\mathcal B$ minimizes
\[\sup_{v \in \mathcal B} \lim_{r \to 0} \inv r \Vol((S + rv) \setminus S)\]
among all measurable, not necessarily convex, sets $S \sbe \R^d$ of the same volume and of finite perimeter.
In other words, among uniform distributions supported on measurable sets of volume 1i and finite perimeter, the one supported on the Wulff Crystal minimizes the maximal instantaneous growth $\Phi(p)$ in the measure of a set due to an instantaneous perturbation from $\mathcal B$.
\end{thm}
The condition ``finite perimeter'' can be interpreted intuitvely here, but a formal definition is given in \cref{defn:finitePerimeter}.
This condition is necessary because otherwise the limit in question does not exist.

\begin{proof}
By \cref{lemma:SetGrowthFormula}, we have
\begin{align}
    \lim_{r \to 0} \inv r \Vol((S + rv) \setminus S)
    = \int_{\pd S} \Theta(\la \mathbf n(x), v \ra) \dd x
    \label{eqn:leakage}
\end{align}
where $\mathbf n(x)$ is the normal at $x$ w.r.t.\ $S$, and $\Theta(x) = \max(0, x)$.
Note this quantity is convex in $v$ because $\Theta$ is convex.
Then
\begin{align*}
    &\phantomeq
        \sup_{v \in \mathcal B} \lim_{r \to 0} \inv r \Vol((S + rv) \setminus S)\\
    &=
        \sup_{v \in \Vrt(\mathcal B)} \lim_{r \to 0} \inv r \Vol((S + rv) \setminus S)\\
        \numberthis\label{eqn:sup2ev}
    &\ge
        \EV_{v \sim \Vrt(\mathcal B)} \lim_{r \to 0} \inv r \Vol((S + rv) \setminus S)\\
    &=
        \EV_{v \sim \Vrt(\mathcal B)}
        \int_{\pd S} \Theta(\la \mathbf n(x), v \ra) \dd x
        \\
    &=
        \int_{\pd S} \EV_{v \sim \Vrt(\mathcal B)}
        \Theta(\la \mathbf n(x), v \ra) \dd x
\end{align*}
Since $\mathcal B = -\mathcal B$ and thus $\Vrt(\mathcal B) = -\Vrt(\mathcal B)$,
\[
\|w\| \defeq
    \EV_{v \sim \Vrt(\mathcal B)}
        \Theta(\la w, v \ra)
=
    \f 1 2 \EV_{v \sim \Vrt(\mathcal B)}
        |\la w, v \ra|
\]
is a seminorm.
This is in fact a norm because $\Vrt(\mathcal B)$ spans $\R^d$, by the assumption that $\mathcal B$ is full-dimensional.
Then, plugging $\| \cdot \|$ into $\| \cdot \|$ in \cref{thm:wulff}, we get that the Wulff Crystal $Z$ w.r.t.\ $\mathcal B$ minimizes
\[Z = \argmin_{S: \Vol(S) = \Vol(Z)} \EV_{v \sim \Vrt(\mathcal B)} \lim_{r \to 0} \inv r \Vol((S + rv) \setminus S).\]

Now note that the norm above is invariant under the transpose of $\mathcal B$'s symmetry group:
for any linear symmetry $g$ of $\Vrt(\mathcal B)$,
\[
\|g^\trsp w\| = \f 1 2 \EV_v |\la g^\trsp w, v\ra| = \f 1 2 \EV_v |\la w, g v\ra| = \|w\|.
\]
This invariance translates to the dual norm $\|\cdot\|_*$'s invariance under the symmetry group itself.
Thus the Wulff Crystal $Z$, being the unit ball of the dual norm, is itself invariant under the symmetry group of $\Vrt(\mathcal B)$.
By assumption, this symmetry group acts transitively on $\Vrt(\mathcal B)$, so $Z$ ``looks the same'' from the angle of every $v \in \Vrt(\mathcal B)$, i.e.
\begin{align*}
    \Vol((Z + rw) \setminus Z)
    = \Vol((Z + rv) \setminus Z)
\end{align*}
for any $w, v \in \Vrt(\mathcal B)$.
Consequently, \cref{eqn:sup2ev} holds with equality, and $Z$ minimizes the supremum in question as well.

\end{proof}

\subsubsection{Growth Calculations for Standard Shapes}

Using the fact that the volume of the $d$-dimensional unit ball is $\pi^{d/2} \Gamma(d/2+1)^{-1}$, and the volume of the standard $d$-dimensional cross polytope is $2^d/d!$, as well as the identity
\[\lim_{r \to 0} \inv r \Vol((S + r v) \setminus S) = \|v\|_2 \Vol(\Pi_{v} S)\]
if $S$ is convex,
we can derive the following facts easily.

\begin{thm}
If $S \sbe \R^d$ is an axis-parallel unit cube and $e_1$ is the first unit vector, then
\[\lim_{r \to 0} \inv r \Vol((S + re_1) \setminus S) = 1.\]
\end{thm}
\begin{thm}\label{thm:ell2ballleakage}
If $S \sbe \R^d$ is a ($\ell_2$-) ball of volume 1 and $v$ is any ($\ell_2$-)unit vector, then
\[\lim_{d \to \infty}\lim_{r \to 0} \inv r \Vol((S + rv) \setminus S) = \sqrt e.\]
\end{thm}

\begin{thm}
If $S \sbe R^d$ is the cross polytope (i.e.\ $\ell_1$ ball) of volume 1, and $e_1$ is the first unit vector, then
\[\lim_{d \to \infty}\lim_{r \to 0} \inv r \Vol((S + re_1) \setminus S) = e.\]
\end{thm}

\begin{thm}
If $S \sbe \R^d$ is an axis-parallel unit cube and $v = (1, \ldots, 1)$, then
\[\lim_{r \to 0} \inv r \Vol((S + rv) \setminus S) = d.\]
\end{thm}

\begin{thm}
If $S \sbe R^d$ is the cross polytope (i.e.\ $\ell_1$ ball) of volume 1, and $v = (1, \ldots, 1)$, then
\[\lim_{d \to \infty}d^{-1/2} \lim_{r \to 0} \inv r \Vol((S + rv) \setminus S) = e \sqrt{2/\pi}.\]
\end{thm}
\begin{proof}
It is equivalent to take $S$ to be the standard $\ell_1$ ball, and to calculate
\begin{align}
    \lim_{d \to \infty} \f{\lim_{r \to 0} \inv r \Vol((S + rv) \setminus S)}{\Vol(S)^{\f{d-1}{d}} \sqrt d},
    \label{eqn:leakageratio}
\end{align}
and confirm it equals $e \sqrt{2/\pi}$.
Note that the unit surface normals of $S$ are $\{\pm1\}^d/\sqrt d$, occurring with equal probability over the surface measure of $S$.
Using \cref{eqn:leakage}, we then see that
\begin{align*}
    \lim_{d \to \infty}\lim_{r \to 0} \inv r \Vol((S + rv) \setminus S)
    &=
        W
        \sum_{i=0}^{\lfloor d/2 \rfloor} \binom d i \f{d - 2i}{\sqrt d}
\end{align*}
where $W = \f{\sqrt d}{(d-1)!}$ is the volume of the simplex $\{x: \sum_i x_i = 1, x \ge 0\}$.
This evaluates to
\begin{align*}
    \f 1 {(d-1)!} \times
    \begin{cases}
    \frac{d+2}{2} \binom{d}{\frac{d}{2}+1} & \text{if $d$ is even}\\
    \f{d+1}2 \binom d {\f{d+1}2} & \text{if $d$ is odd}.
    \end{cases}
\end{align*}
Finally, since the volume of $S$ is $2^d/d! $, we can calculate \cref{eqn:leakageratio} directly and obtain the desired result.
\end{proof}

\subsubsection{Wulff Crystal of the \texorpdfstring{$\ell_\infty$}{Linf} Ball}

In this section, let $Z$ be the Wulff Crystal (\cref{defn:wulff}) w.r.t.\ $\mathcal B = \{x: \|x\|_\infty \le 1\}$, i.e.\ $Z$ is the unit ball of the norm dual to $\|x\|_* \defeq \EV_{v \sim \{\pm1\}^d} |\la x, v\ra|$.
By \cref{prop:wulffzonotope}, $Z$ can also be described as the zonotope of $2^{-d+1}\{\pm1\}^d$.
From these descriptions, we can straightforwardly see the following properties of $Z$.
\begin{prop}\label{prop:linfWulffCrystalBasic}
The vertices of $Z$ farthest from the origin are coordinate vectors and their negations.
The facets of $Z$ closest to the origin are of the form $\{x: \pm x_i \pm x_j \le 1\}$.
Therefore, with $B$ denoting the $\ell_2$ unit ball,
\[\f 1 {\sqrt 2} B \sbe Z \sbe B.\]
\end{prop}
In general, the properties of $Z$ are elusive, and tied to many open problems in combinatorics and polytope theory \citep{ziegler_lectures_1995}.
But we may heuristically compute $\lim_{d \to \infty} \lim_{r \to 0} \inv r \Vol((Z + rv) \setminus Z) = \lim_{d \to \infty}  \|v\|_2 \Pi_{v} Z$ when $v = (1, \ldots, 1)$, as follows.
(Because our computation is heuristic, we phrase the following as a claim, and not a theorem)

\begin{claim}
If $S \sbe \R^d$ is the Wulff Crystal w.r.t.\ the $\ell_\infty$ unit ball, scaled to have volume 1, and $v = (1, \ldots, 1)$, then
\[\lim_{d \to \infty} d^{-1/2} \lim_{r \to 0} \inv r \Vol((S + rv) \setminus S) = \sqrt e.\]
\end{claim}
\begin{proof}[Derivation]
Since $2^{d-1}Z$ is $\Zon(\{\pm1\}^d)$, we have $\Pi_{v} 2^{d-1} Z = \Zon(\Pi_{v}\{\pm1\}^d)$, the zonotope of the set of vectors $\left\{x - \f{\la x, v \ra}{\|v\|_2} v: v \in \{\pm1\}^d\right\}$.
By \cref{eqn:differentialUniform}, we then have
\begin{align}
    \lim_{r \to 0} \f{\Vol((S + rv) \setminus S)}{r \sqrt d}
    =
    \f{\Vol(\Zon(\Pi_{v}\{\pm1\}^d))}{\Vol(\Zon(\{\pm1\}^d))^{\f{d-1}d}}.
    \label{eqn:target}
\end{align}
Now \cref{lemma:zoncubevol} tells us that the $\Vol(\Zon(\{\pm1\}^d))$ is a multiple of the expected determinants of all $d \times d$ matrices with entries $\pm 1$.
By a result of \citet{nguyenvu} (\cref{thm:detCLT}), the determinant of a random $d \times d$ matrix with iid $\pm 1$ entries is distributed in high dimension $d$ roughly as $\sqrt{(d-1)!} e^{z\sqrt{\f 1 2 \log d}}$ where $z \sim \Gaus(0, 1)$.
Thus, by \cref{lemma:zoncubevol}, we should expect (this is the first place where we argue heuristically)
\begin{align*}
    \Vol(\Zon(\{\pm1\}^d))
        &\approx \f 1 {d!} 2^{d^2} \EV_z \sqrt{(d-1)!} e^{z\sqrt{\f 1 2 \log d}}
            \\
        &=
            \f 1 {d!} 2^{d^2} \sqrt{(d-1)!} d^{\f 1 4}
            .
            \numberthis\label{eqn:linfzonvol}
\end{align*}
We verify this approximation to be correct numerically for moderately large $d$.
Similarly, the uniform distribution over $\{\pm1\}^d$ is close to a standard Gaussian when $d \gg 1$, so that $\Pi_{v} \{\pm1\}^d$ is close to a $(d-1)$-dimensional standard Gaussian.
Therefore, we should expect that
\begin{align*}
    &\phantomeq
        \Vol(\Zon(\Pi_{v}\{\pm1\}^d))
        \\
    &\approx
        \f 1 {(d-1)!} (2^d-1)^{d-1} \EV | \det Y|
        \\
    &\approx
        \f 1 {(d-1)!} 2^{d(d-1)} \sqrt{(d-2)!} (d-1)^{\f 1 4}
        ,
        \numberthis\label{eqn:projectedlinfzonvol}
\end{align*}
where $Y$ is a $(d-1)\times (d-1)$ Gaussian matrix, and where in \cref{eqn:projectedlinfzonvol}, we used the heuristic that for large $d$, $|\det Y|$ is lognormal (\cref{thm:detCLT}).
Again, we verify these approximations numerically.
Plugging in \cref{eqn:linfzonvol,eqn:projectedlinfzonvol} into \cref{eqn:target} and taking the $d \to \infty$ limit yields the desired result.

\end{proof}

Since the sphere has the same large $d$ limit (\cref{thm:ell2ballleakage}), we can say that
\begin{claim}\label{claim:linfSphereOptimal}
For every $\epsilon > 0$, $S =$ the $\ell_2$-ball of volume 1 achieves within $\epsilon$ of
\[\min_{\Vol(S)=1} \sup_{\|v\|_\infty \le 1} d^{-1/2} \lim_{r \to 0} \inv r \Vol((S + rv) \setminus S),\]
for sufficiently large $d$.
This is not true for $S =$ the $\ell_\infty$- or the $\ell_1$-ball.
\end{claim}

\begin{lemma}\label{lemma:zoncubevol}
The volume of $\Zon(\{\pm1\}^d)$ is
\begin{align*}
    \f 1 {d!} 2^{d^2} \EV_X |\det X|
\end{align*}
where $X \in \{\pm1\}^{d\times d}$ is a random $d\times d$ matrix whose coordinates are iid Rademacher variables (i.e.\ 1 or $-1$ with equal probability).
\end{lemma}
\begin{proof}
The above expression can be rewritten as
\begin{align*}
    \f 1 {d!} \sum_{X \in \{\pm1\}^{d\times d}} |\det X|.
\end{align*}
Because $\det X = 0$ if any two columns are equal, so this is equivalent to summing over all $X$ with distinct columns.
\begin{align*}
    \f 1 {d!} \sum_{\substack{X \in \{\pm1\}^{d\times d}\\ X \text{ has distinct columns}}} |\det X|.
\end{align*}
Finally, any given set of $d$ distinct column vectors is represented $d!$ times in the sum through its $d!$ permutations, so this is equal to
\begin{align*}
    \sum_{T \sbe \{\pm1\}^{d}: |T| = d} |\det T|,
\end{align*}
which by \cref{prop:zonvol} is the volume of the zonotope in question.
\end{proof}

\begin{thm}[\citet{nguyenvu}]\label{thm:detCLT}
Let $A_n$ be an $n\times n$ random matrix whose entries are independent real random variables with mean zero, variance one and with subexponential tail.
Then with $\mu_n = \f 1 2 \log (n-1)!$ and $\sigma_n = \sqrt{\f 1 2 \log n}$,
\begin{align*}
    \sup_{x\in {\R}}
    \biggl|\Pr
        \biggl(\frac
            {\log(|\det A_n|)-\mu_n}
            {\sigma_n}
        \le x\biggr)
        -
        \Pr\bigl(\Gaus(0,1)\le x\bigr)
    \biggr|
    \\
\le\log^{-{1}/{3}+o(1)}n.
\end{align*}
\end{thm}
In other words, $\det A_n \approx \sqrt{(n-1)!} e^{z\sqrt{\f 1 2 \log n}}$ where $z \sim \Gaus(0, 1)$.

\subsubsection{Growth Formula of a Set}

\begin{lemma}\label{lemma:SetGrowthFormula}
    Let $S \sbe \R^d$ be a set of finite perimeter and $v \in \R^d$ be any vector.
    Then
    \begin{align*}
        \lim_{r \to 0} \inv r \Vol((S + rv) \setminus S)
        = \int_{\pd S} \Theta(\la \mathbf n(x), v \ra) \dd x,
    \end{align*}
    where $\mathbf n(x)$ is the normal at $x$ w.r.t.\ $S$, and $\Theta(x) = \max(0, x)$.
\end{lemma}

\begin{proof}
    Let $\pd S_v \defeq \{x \in \pd S: \la \mathbf n(x), v \ra > 0\}$ be the part of $S$'s boundary whose surface normal \emph{aligns} with $v$.
    For any vector $w$, let $\pd S_v + [0, w] = \{x + rw: x \in \pd S_v, r \in [0, 1]\}$ be the Minkowski sum of $\pd S_v$ and the segment $[0, w]$.
    Then it's clear that $\Vol(\pd S_v + [0, rv]) \le r \int_{\pd S} \Theta(\la \mathbf n(x), v \ra) \dd x$, and that
    \begin{align*}
        (S + rv) \setminus S \sbe \pd S_v + [0, rv].
    \end{align*}
    Thus,
    \begin{align*}
        &\phantomeq
            \lim_{r \to 0} \inv r \Vol((S + rv) \setminus S)
            \\
        &\le
            \lim_{r \to 0}
            \inv r r \int_{\pd S} \Theta(\la \mathbf n(x), v \ra) \dd x\\
        &=
            \int_{\pd S} \Theta(\la \mathbf n(x), v \ra) \dd x.
    \end{align*}

    Now for the other direction, observe that the signed measure $\inv r(\ind(x \in S + rv) - \ind(x \in S))$ converges weakly to the (singular) signed measure $\la \mathbf n(x), v\ra$ supported on $\pd S$.
    Indeed, for any compactly supported $C^1$ function $f$, we have
    \begin{align*}
        &\phantomeq
            \lim_{r\to 0} \inv r \int f(x) (\ind(x \in S + rv) - \ind(x \in S)) \dd x\\
        &=
            \lim_{r\to 0} \inv r \int (f(x+rv) - f(x)) \ind(x \in S) \dd x\\
        &=
            \int_S  D_v f(x) \dd x
        =
            \int_{\pd S} \la \mathbf n(x), v f(x)\ra \dd x.
    \end{align*}
    Now, taking the supremum of the RHS over all compactly supported $C^1$ function $|f| \le 1$, we get
    \begin{align*}
        &\phantomeq
            \int_{\pd S} \Theta(\la \mathbf n(x), v \ra) \dd x
            \\
        &=
            \f 1 2 \int_{\pd S} |\la \mathbf n(x), v \ra| \dd x
            \\
        &=
            \f 1 2 \sup_{f} \int_{\pd S} \la \mathbf n(x), v f(x)\ra \dd x\\
        &=
            \f 1 2 \sup_f \lim_{r\to 0} \inv r \int f(x) (\ind(x \in S + rv) - \ind(x \in S)) \dd x\\
        &\le
            \f 1 2 \liminf_{r\to 0} \inv r \sup_f \int f(x) (\ind(x \in S + rv) - \ind(x \in S)) \dd x\\
        &=
            \liminf_{r\to 0} \inv r \Vol((S + rv) \setminus S)
    \end{align*}
    as desired.
\end{proof}

\subsection{Optimal Smoothing Distributions Have Wulff Crystal Level Sets}

\begin{defn}[Level-Equivalence]
Call two distribution $q_1$ and $q_2$ \emph{level-equivalent} if their superlevel sets have the same volumes:
\begin{align*}
    \Vol\{x: q_1(x) \ge t\} = \Vol\{x: q_2(x) \ge t\},\quad
    \forall t \in (0, \infty).
\end{align*}
\end{defn}

\begin{thm}\label{thm:wulffOptimalGeneral}
Let $\mathcal B$ be a full-dimensional polytope in $\R^d$ symmetric around the origin, and whose vertices form a symmetric set.
Let $Z$ be the Wulff Crystal w.r.t. $\mathcal B$.
Let $q_0$ be a regular (\cref{defn:regular}) probability density function.
Among all probability distributions $q$ with regular (\cref{defn:regular}) and even density function that is level equivalent to $q_0$, the probability density function with concentric superlevel sets proportional to $Z$ minimize
\begin{align*}
    \Phi(1/2) = \sup_{\|v\|=1} \sup_{q(U) = 1/2} \lim_{r \searrow 0} \f{q(U-rv) - 1/2}r
    ,
\end{align*}
where $\|\cdot\|$ is the norm defined by $\mathcal B$.
\end{thm}

Note that this theorem does not imply \cref{thm:wulffOptimalFull} since uniform distributions do not have regular densities.
However, this can be generalized to bounded-variation densities (\cref{thm:wulffOptimalGeneralBV}) which subsume both \cref{thm:wulffOptimalFull} and \cref{thm:wulffOptimalGeneral}.

\begin{proof}

Consider any distribution $q$ level-equivalent to $q_0$.
Let $U_t$ be its superlevel sets.

Expanding the definition of $\Phi$ in terms of $\mygamma{u}{p}$ (see \cref{defn:Phi}), and exchanging maximization and expectation, we get
\begin{align*}
\Phi(1/2)
    &=
        \max_{u \in \Vrt(\mathcal B)} \EV \mygamma{u}{1/2} \ge \EV_u \EV \mygamma{u}{1/2}
        \\
    &=
        \EV_u \int \max(\nabla q(x) \cdot u, 0) \dd x
        \\
    &=
        \int \EV_u \max(\nabla q(x) \cdot u, 0) \dd x
        .
\end{align*}
Since $\mathcal B = -\mathcal B$ and thus $\Vrt(\mathcal B) = -\Vrt(\mathcal B)$,
\[
\mynorm w \defeq
    \EV_{u \sim \Vrt(\mathcal B)}
        \max(0, \la w, u \ra)
=
    \f 1 2 \EV_{u \sim \Vrt(\mathcal B)}
        |\la w, u \ra|
\]
is a seminorm.
This is in fact a norm because $\Vrt(\mathcal B)$ spans $\R^d$, by the assumption that $\mathcal B$ is full-dimensional.
Thus
\begin{align*}
    \Phi(1/2) \ge \int \mynorm{\nabla q(x)} \dd x.
\end{align*}
Define $g(x) \defeq \f{\mynorm{\nabla q(x)}}{\|\nabla q(x)\|_2}$ if $\nabla q(x) \ne 0$, and $g(x) = 0$ otherwise.
Then by \cref{thm:coarea},
\begin{align*}
    \int \mynorm{\nabla q(x)} \dd x
    &=
        \int g(x) \|\nabla q(x)\|_2 \dd x
        \\
    &=
        \int_0^\infty \int_{\pd U_t} g(x) \dd x \dd t
        .
\end{align*}
By the Weak Sard's theorem (\cref{thm:regulargradvanish}), we may ignore the places where $\nabla q(x) = 0$, and this integral is the same as
\begin{align*}
    \int_0^\infty \int_{\pd U_t} g(x) \dd x \dd t
    &=
        \int_0^\infty \int_{\pd U_t} \f{\mynorm{\nabla q(x)}}{\|\nabla q(x)\|_2} \dd x \dd t.
        \numberthis\label{eqn:PhiLowerbound}
\end{align*}

Now, the surface normal at $x$ w.r.t.\ $U_t$ is proportional to $\nabla q(x)$.
Thus, the ($\ell_2$-)unit normal $\mathbf n(x)$ at $x$ w.r.t.\ $U_t$ is given by $\f{-\nabla q(x)}{\|\nabla q(x)\|_2}$, so
$\f{\mynorm{\nabla q(x)}}{\|\nabla q(x)\|_2}$ is the $\mynorm \cdot$-norm of $\mathbf n(x)$.
Therefore, the inner integral is
\[\int_{\pd U_t} \f{\mynorm{\nabla q(x)}}{\|\nabla q(x)\|_2} \dd x
= \int_{\pd U_t} \mynorm{\mathbf n(x)} \dd x.\]
By \cref{thm:wulffOptimalFull}, this is minimized for fixed $\Vol(U_t)$ by $U_t \propto$ the Wulff Crystal w.r.t.\ $\mathcal B$.
Thus, the unique distribution $q^*$ level equivalent to $q_0$ and with concentric Wulff Crystal superlevel sets (all centered at 0) minimizes \cref{eqn:PhiLowerbound}.
But since
\[\Phi(1/2) = \max_u \int \max(\nabla q(x) \cdot u, 0) \dd x\]
and the inner integral here is invariant in $u$ when $q = q^*$ by the symmetry of the Wulff Crystal, as in the proof of \cref{thm:wulffOptimalFull}, \cref{eqn:PhiLowerbound} in fact holds with equality for $q^*$, so that $q^*$ minimizes $\Phi(1/2)$ as well.
\end{proof}

\subsection{Optimality among Wulff Crystal Distributions}

Given the optimality of Wulff Crystal distributions among level-equivalent distributions, one may wonder, among Wulff Crystal distributions themselves, which one minimizes $\Phi(1/2)$?
Because no two such distributions are level-equivalent, we need to fix another notion of the spread of the distribution.
The below theorem answers this question, controlling for the expected Wulff Crystal norm.
\begin{thm}\label{thm:WulffCrystalDistributionOptimization}
Let $\mathcal B$ be a full-dimensional polytope in $\R^d$ symmetric around the origin, and whose vertices form a symmetric set.
Let $Z$ be the Wulff Crystal w.r.t. $\mathcal B$, and let $\mynorm{\cdot}$ denote the norm with $Z$ as its unit ball.
Consider a probability distribution $q(x) \propto \exp(-\psi(\mynorm{x}))$ on $\R^d, d \ge 2$, for some regular, even $\psi$.
Then with $\Phi$ defined against the adversary $\mathcal B$, we have, for any $k > 0$,
\begin{align*}
    \Phi(1/2) \ge \f{(d-1)C}{\sqrt[k]{\EV_{\delta \sim q}\mynorm{\delta}^k}},
\end{align*}
where $C \defeq \f 1 2 \EV_{x \sim \pd Z} |\la \nabla \mynorm{x}, u \ra|$, for any vertex $u$ of $\mathcal B$, is a constant that depends only on $Z$.
\end{thm}

\begin{remk}
Note that for any $p > 0, k > 0$, there is a constant $T_{p, k}$ depending only on $Z$ such that
\[\sqrt[k]{\EV_{\delta \sim q}\mynorm{\delta}^k} = T_p \sqrt[k]{\EV_{\delta \sim q}\|\delta\|_p^k}\]
for any $q$ of the form in \cref{thm:WulffCrystalDistributionOptimization}.
So this theorem applies when we want to fix most measures of spread.
\end{remk}

\newcommand{\radial}{\mathrm{r}}
\begin{proof}

    By the symmetry of the Wulff Crystal w.r.t.\ symmetry group of $\Vrt(\mathcal B)$, we have $\Phi(1/2) = \EV \mygamma{u}{1/2}$ for any $u \in \Vrt(\mathcal B)$.
    Henceforth, we fix $u$ to be one such vertex of $\mathcal B$.
    Note that $\nabla q(x) = -q(x) \psi'(\mynorm{x}) \nabla \mynorm{x}$.
    Thus $\nabla q(x) = -\nabla q(-x)$, and
    \begin{align*}
        \Phi(1/2) =
            \EV \mygamma{u}{1/2}
        &=
            \f 1 2 \int q(x) |\psi'(\mynorm{x}) \la \nabla \mynorm{x}, u \ra| \dd x
            \\
        &=
            \f 1 2 \EV_{x \sim q} |\psi'(\mynorm{x}) \la \nabla \mynorm{x}, u \ra|
    \end{align*}
    Note that a sample from $q$ can be obtained by first sampling $v \sim \pd Z$ from the (uniform distribution on the) boundary of $Z$ and $r \sim q_\radial$, where $q_\radial(r) \propto r^{d-1} e^{-\psi(r)}$, and finally returning their product $r v$.
    Therefore, because $\nabla \mynorm{x}$ doesn't dependent on $\mynorm x$, we can continue the above equations as follows.
    \begin{align*}
        \Phi(1/2)
            &=
                \f 1 2 \EV_{x \sim \pd Z} |\la \nabla \mynorm{x}, u \ra|
                    \times \EV_{r \sim q_\radial} |\psi'(r)|
                \\
            &=
                C \EV_{r \sim q_\radial} |\psi'(r)|
                .
    \end{align*}
    Now notice that, because $d \ge 2$, with $R \defeq \int_0^\infty r^{d-1} e^{-\psi(r)} \dd r$,
    \begin{align*}
        &\phantomeq
            R \EV_{r \sim q_\radial} |\psi'(r)|
            \\
        &=
            \int_0^\infty r^{d-1} e^{-\psi(r)} |\psi'(r)| \dd r
            \\
        &\ge
            \int_0^\infty r^{d-1} e^{-\psi(r)} \psi'(r) \dd r
            \numberthis\label{eqn:_rabs}
            \\
        &=
            -r^{d-1}e^{-\psi(r)}\bigg|_0^\infty
            (d-1)\int_0^\infty r^{d-2} e^{-\psi(r)}\dd r
            \\
        &=
            (d-1)\int_0^\infty r^{d-2} e^{-\psi(r)}\dd r
            \\
        &=
            (d-1) R \EV_{r \sim q_\radial} \inv r
            .
            \numberthis\label{eqn:_radialIntegrationbyparts}
    \end{align*}
    Then, by Holder's inequality, for any $k > 0$,
    \begin{align*}
        \Phi(1/2) \sqrt[k]{\EV_{x \sim q} \mynorm{x}^k}
        &=
            (d-1) C \EV_{r \sim q_\radial} \inv r \sqrt[k]{\EV_{r \sim q_\radial} r^k}
            \\
        &\ge
            (d-1) C \lp \EV_{r \sim q_\radial} 1\rp^{1 + \f 1 k}
            \numberthis\label{eqn:_radialCauchyschwarz}
            \\
        &=
            (d-1) C.
    \end{align*}
\end{proof}

\begin{remk}
    Let us comment briefly on the equality case of \cref{thm:WulffCrystalDistributionOptimization}, or the lack thereof.
    There are two inequalities used in the proof above, namely \cref{eqn:_rabs,eqn:_radialCauchyschwarz}.
    For \cref{eqn:_rabs} to hold with equality, we just need $\psi'(r) \ge 0$ for all $r \ge 0$.
    On the other hand, it is impossible for \cref{eqn:_radialCauchyschwarz} to hold with equality when $\psi$ is not allowed to be a delta function on $r = 1$ (and if that were the case, then \cref{eqn:_rabs} cannot hold with equality).
    However, as long as the radial distribution $q_\radial$ concentrates around its mean value sufficiently well, the inequality should be approximately tight.
    This is typically the case for high dimensional distributions.

    For example, in the Gaussian case with $\psi(r) = e^{-r^2}$, we have $\int_0^\infty r^c e^{-\psi(r)} \dd r = \f 1 2 \Gamma\lp \f{c+1}2 \rp$ for any $c > -1$, so that
    \begin{align*}
            \EV_{r \sim q_\radial} \inv r \EV_{r \sim q_\radial} r
        &=
            \f{\Gamma\lp \f{d-1}2 \rp\Gamma\lp \f{d+1}2 \rp}{\Gamma\lp \f{d}2 \rp^2}
            \\
        &=
            1 + \f 1 {2d} + O(d^{-3/2}),
            \quad
            \text{as $d \to \infty$.}
    \end{align*}
    Concretely, when $d = 3 \times 1024$ as in the case of CIFAR10, this quantity is 1.00016, so \cref{eqn:_radialCauchyschwarz} is quite close to being tight here.
    \end{remk}

\section{Generalization of Differential Method and Wulff Crystal Optimality Results to Bounded Variation Densities}

\label{sec:BV}

While the regularity condition \cref{defn:regular} covers most distributions we care about, we still have to reason separately about, e.g.\ uniform distributions on sets, or mixture of such distributions and regular distribution.
However, regularity can be weakened to the notion of \emph{bounded variation} to cover all such cases.
Bounded variation (BV) is ``essentially the weakest measure theoretic sense in which a function can be differentiable'' \cite{evans2015measure}.
BV functions include the usual continuously differentiable functions as well as indicator functions of ``finite perimeter'' sets.
More generally, the notion of BV allows a ``controlled'' amount of jump-type discontinuities.
Our differential method and our Wulff Crystal optimality results can be generalized to distributions with BV densities.

Readers exposed to the notion of bounded variation for the first time might find it helpful to mentally substitute ``BV'' in our results below with ``differentiable'' or with ``indicator function'' on the first read through.
All probability distribution densities we work with concretely in this paper have bounded variation.
\begin{defn}[Bounded Variation]\label{defn:BV}
Let $\Omega \sbe \R^d$ be an open set.
A function $f \in L^1(\Omega)$ is said be of \emph{bounded variation} (or BV), written $f \in \BV(\Omega)$, if there exists a finite Radon measure $|Df|$ on $\R^d$ along with a vector function $\mathbf n: \R^d \to \R^d$ with $\|\mathbf n(x)\|_2 = 1$ almost everywhere w.r.t.\ $|Df|$, such that, for every compactly supported, continuous differentiable $\phi: \Omega \to \R^d$, we have
\begin{align}
    \int_\Omega f(x) \div \phi(x) \dd x = -\int_\Omega \la \phi, \mathbf n(x) \ra |Df|(x).
        \label{eqn:BVintbypart}
\end{align}
We denote by $Df$ the vector measure $\mathbf n |Df|$.
\end{defn}

\begin{exmp}
If $u$ is differentiable, then $Du(x)$ is just the vector measure $\nabla u(x) \dd x$ and $|Du|(x)$ is $\|\nabla u(x)\|_2 \dd x$, and \cref{eqn:BVintbypart} follows just by ordinary integration by parts.
Same thing holds for Sobolev (i.e. weakly differentiable) functions.
\end{exmp}

\begin{exmp}\label{exmp:indicatorFunctionBV}
If $u$ is the indicator function of, for example, a ball, then $|Du|$ is the $(d-1)$-dimensional Hausdorff measure supported on its boundary (a sphere), and $\mathbf n(x)$ is the unit normal at $x$ pointing inward.
More generally, this characterization of $Du$ as the boundary measure with unit normals holds when $u$ is the indicator function of a set of \emph{finite perimeter}.
\end{exmp}

\begin{defn}[Sets of Finite Perimeter]\label{defn:finitePerimeter}
A set $U$ has \emph{finite perimeter} if its indicator function $\chi$ is a BV function.
In this case, we write
\begin{align}
    \int_{\pd U} g(x) \dd x \defeq \int g(x) |D \chi|(x)
    \label{eqn:finitePerimeter}
\end{align}
for any Borel function $g$.
\end{defn}
For sufficently smooth sets $U$ (like a sphere), the LHS of \cref{eqn:finitePerimeter} can be interpreted as an integral over the Hausdorff measure of the topological boundary $\pd U$, and \cref{eqn:finitePerimeter} still holds.
More generally, there is a subset of the topological boundary, called the \emph{reduced boundary} of $U$, containing ``almost every point'' of $\pd U$, such that the LHS of \cref{eqn:finitePerimeter} can be interpreted as an integral over the Hausdorff measure of the reduced boundary.
See \cite{evans2015measure} for more details.

\paragraph{Coarea Formula and Weak Sard for BV Functions}

Coarea formula also holds for BV functions.

\begin{thm}[Coarea Formula \citep{federer2014geometric,morgan2016geometric}]\label{thm:coareaBV}
Let $\Omega \sbe \R^d$ be an open set, $g \in L^1(\Omega)$ be Borel, and $f: \R^d \to \R$ have bounded variation.
Let $U_t \defeq \{x: f(x) \ge t\}$ denote the superlevel set of $f$ at level $t$.
Then for almost every $t$, $U_t$ has finite perimeter, and we have
\begin{align}
    \int g(x) |Df|(x) = \int_\R \int_{\pd U_t} g(x) \dd x \dd t.
    \label{eqn:coareaBV}
\end{align}
\end{thm}

\begin{exmp}
If $f$ is differentiable, then \cref{eqn:coareaBV} reduces to
\begin{align}
    \int g(x) \|\nabla f(x)\|_2 \dd x = \int_\R \int_{\pd U_t} g(x) \dd x \dd t.
\end{align}
\end{exmp}

\begin{exmp}
    If $f$ is the indicator function of a set $U$ of finite perimeter, then both sides of \cref{eqn:coareaBV} reduce to $\int_{\pd U} g(x) \dd x$.
\end{exmp}

We also have a converse that tells us a function is BV if almost all of its superlevel sets have finite perimeter.
\begin{thm}[c.f.\ \citet{evans2015measure}]
Let $\Omega \sbe \R^d$ be an open set and $f \in L^1(\Omega)$.
Let $U_t \defeq \{x: f(x) \ge t\}$ denote the superlevel set of $f$ at level $t$.
If almost every $U_t$ has finite perimeter and
\begin{align*}
    \int_{-\infty}^\infty \int_{\pd U_t} \dd x \dd t
    = \int_{-\infty}^\infty \Vol(\pd U_t) \dd t < \infty,
\end{align*}
then $f$ is BV (where $\Vol(\pd U_t)$ denotes $(d-1)$-Hausdorff measure of the \emph{reduced boundary} of $U_t$).
\end{thm}

By setting $g$ in \cref{thm:coareaBV} to be the indicator function over the complement of the support of $|Df|$, we get
\begin{thm}[Weak Sard]\label{thm:regulargradvanishBV}
For any BV $f: \R^d \to \R$, let $Z$ denote the complement of the support of $|Df|$.
Let $U_t \defeq \{x: f(x) \ge t\}$ denote the superlevel set of $f$ at level $t$.
Then
\begin{align*}
    \Vol(Z \cap \pd U_t) = 0
\end{align*}
for almost every $t \in \R$.
Here $\Vol$ denote the Hausdorff measure of dimension $d-1$, and again $\pd U_t$ denotes reduced boundary.
\end{thm}

\paragraph{Bounded Variation on Almost Every Line}

Like how Sobolev functions has the ACL property, a BV function on $\R^d$ is also BV on almost every line parallel to a given direction.
\begin{thm}[c.f.\ Thm 5.22 of \citet{evans2015measure}]\label{thm:BVL}
Let $f: \R^d \to \R$ be BV, and let $u \in \R^d$ be some vector.
Then for almost every line parallel to $u$, the restriction of $f$ to that line is BV, possibly after changing values on a Lebesgue measure 0 (on that line).
\end{thm}

This allows to show convolution with BV functions yields a.e. differentiability.
\begin{lemma}\label{lemma:convolutionDiffBV}
    If a function $q$ is in $\BV(\R^d)$, then for every bounded measurable $F: \R^d \to \R$, the convolution $F * q$ is absolutely continuous on every line, and for every vector $u \ne 0$,
    \[
    D_u (F * q) = F * (D_u q),\quad \text{a.e.,}
    \]
    where $D_u$ on the LHS denotes ordinary directional derivative, and $D_u q$ on the RHS denotes the measure $\la \mathbf n, u\ra|Dq|$, with $\mathbf n$ as in \cref{defn:BV}.
\end{lemma}
Note that $F * q$ is already bounded and continuous as the convolution of a $L^\infty$ and a $L^1$ function.

\begin{proof}
It suffices to show that $(F * q)(x + \tau u) - (F * q)(x) = \int_{0}^\tau (F * D_u q)(x + t u) \dd t$ for every $x$ and every unit vector $u$.
\begin{align*}
    &\phantomeq
        \int_{0}^\tau (F * D_u q)(x + t u) \dd t
        \\
    &=
        \int_0^\tau \int F(\hat x) D_u q(x - \hat x + t u) \dd \hat x \dd t
        \\
    &=
        \int F(\hat x) \int_0^\tau D_u q(x - \hat x + t u) \dd t \dd \hat x
        \numberthis\label{_eqn:FubiniBV}
        \\
    &=
        \int F(\hat x) [q(x - \hat x + \tau u) - q(x - \hat x)] \dd \hat x
        \numberthis\label{_eqn:BV}
        \\
    &=
        F*q(x + \tau u) - F * q(x)
        .
\end{align*}
In \cref{_eqn:FubiniBV}, we applied Fubini's theorem after noting that
\begin{align*}
    &\phantomeq
        \int_0^\tau \int |F(\hat x)| |D_u q|(x - \hat x + t u) \dd \hat x \dd t
        \\
    &\le
        \|F\|_{L^\infty} \int_0^\tau \int |D q|(x - \hat x + t u) \dd \hat x \dd t\\
    &\le
        \tau \|F\|_{L^\infty} TV(|Dq|) < \infty,
\end{align*}
where $TV$ denotes total variation.
In \cref{_eqn:BV}, we applied the linewise BV property (\cref{thm:BVL}) of $q$ to modify $q$ on a null set to obtain a version $\tilde q$ that is BV and right continuous on almost every line parallel to $u$.
Then \cref{_eqn:BV} expands as
\begin{align*}
    &\phantomeq
        \int F(\hat x) \int_0^\tau D_u q(x - \hat x + t u) \dd t \dd \hat x
        \\
    &=
        \int F(\hat x) \int_0^\tau D_u \tilde q(x - \hat x + t u) \dd t \dd \hat x
        \numberthis\label{_eqn:prestieltjes}
        \\
    &=
        \int F(\hat x) [\tilde q(x - \hat x + \tau u) - \tilde q_-(x - \hat x)] \dd \hat x
        \numberthis\label{_eqn:stieltjes}
        \\
    &\pushright{\text{where $\tilde q_-(x) = \lim_{t \nearrow 0} \tilde q(x + tu)$}}\\
    &=
        \int F(\hat x) [\tilde q(x - \hat x + \tau u) - \tilde q(x - \hat x)] \dd \hat x
        \numberthis\label{_eqn:stieltjesCorrection}
        \\
    &=
        \int F(\hat x) [q(x - \hat x + \tau u) - q(x - \hat x)] \dd \hat x
        \numberthis\label{_eqn:stieltjesCorrection2}
        .
\end{align*}
Here in \cref{_eqn:stieltjes}, we use the fact that the Lebesgue integral in $t$ in \cref{_eqn:prestieltjes} is equal to the Lebesgue-Stietjes integral with integrator $t \mapsto \tilde q(x - \hat x + t u)$.
Because on almost every line, $\tilde q_-$ differs from $\tilde q$ only at the points of discontinuity, of which there are only countably many, $\tilde q_-$ differs from $\tilde q$ on $\R^d$ in a null set; this is \cref{_eqn:stieltjesCorrection}.
Finally, $q$ differs from $\tilde q$ on a null set, so \cref{_eqn:stieltjesCorrection2} holds.
\end{proof}

\subsection{Differential Method for BV Densities}

To generalize the differential method to distribution with BV densities, we need to to define $\Phi$ differently.
\begin{defn}\label{defn:PhiBV}
Let $q(x)$ be a distribution with BV density, which we also denote as $q$.
Then $|Dq|$ is a finite Radon measure.
By Lebesgue Decomposition Theorem, $|Dq|$ is the sum of two measures $|Dq|_{ac}$ and $|Dq|_s$ which are resp.\ absolutely continuous and singular w.r.t.\ $q$.
Thus, there is some set of $q$-measure 0 that has full measure under $|Dq|_s$.

For any vector $u \in \R^d$, let $\gamma_u$ be the random variable given by
\[\gamma_u \defeq \la u, \mathbf n(\delta) \ra \f{\dd |Dq|_{ac}(\delta)}{\dd q(\delta)}, \delta \sim q,\]
where $\f{\dd |Dq|_{ac}(\delta)}{\dd q(\delta)}$ is the Radon-Nikodym derivative of $|Dq|_{ac}$ against $q$, and $\mathbf n$ is the vector component of $Dq$ as in \cref{defn:BV}.
Define $\varphi_u$ to be the complementary CDF of $\gamma_u$,
\begin{align*}
    \varphi_u(c) \defeq \Pr[\gamma_u > c],
\end{align*}
and define the inverse complementary CDF $\inv\varphi_u(p)$ of $\gamma_u$ to be
\begin{align*}
    \inv\varphi_u(p) \defeq \inf\{c: \Pr[\gamma_u > c] \le p\}.
\end{align*}
For any $p \in [0, 1]$, define a new random variable $\mygamma{u}{p}$ by
\begin{align*}
    \mygamma{u}{p} =
    \begin{cases}
    \gamma_u|_{(c, \infty)}
        &   \text{with probability $\varphi_u(c)$}\\
    c
        &   \text{with probability $p - \varphi_u(c)$}\\
    0   &   \text{with probability $1 - p$,}
    \end{cases}
\end{align*}
where $c \defeq \varphi_u^{-1}(p)$ and $\gamma_u|_{(c, \infty)}$ is the random variable $\gamma_u$ conditioned on $\gamma_u > c$.

Define
\begin{align*}
\vartheta_u \defeq
    \int \max(0, \la\mathbf n(x), u\ra) |Dq|_s(x).
\end{align*}
Let $\Bb$ be the unit ball of $\|\cdot \|$ as in \cref{assm:FGqsmooth}.
Then we define $\Phi: [0, 1] \to \R$ by
\begin{align*}
    \Phi(p) \defeq
        \max_{u \in \Vrt(\Bb)}
            \EV \mygamma{u}{p}
            + \vartheta_u
            .
\end{align*}
\end{defn}
Here, $\vartheta_u$ represents the instantaneous growth in measure when the set has maximal allocation toward the support of $|Dq|_s$.

\begin{exmp}
Let $q$ be the uniform distribution on $[0, 1]^d$.
Then $|Dq|$ is the Hausdorff measure on the surface of the cube, which is purely singular w.r.t.\ $q$.
Thus, $|Dq| = |Dq|_s$, $|Dq|_{ac} = 0$, and $\gamma_u = 0$.
On the other hand, for $x$ on the boundary of the cube, $\mathbf n(x)$ is the unit normal pointing into the cube, and
\begin{align*}
    \vartheta_u = \int \max(0, \la \mathbf n(x), u\ra) |Dq|_s(x)
        = \Vol(\Pi_u [0, 1]^d),
\end{align*}
for any $\ell_2$ unit vector $u$.
\end{exmp}
This example generalizes to any uniform distribution on any set $S$ of finite perimeter, except that the last equality needs not hold if $S$ is not convex.

With this definition of $\Phi$, the proof of the differential method goes through if we swap usage of \cref{lemma:convolutionDiff} with \cref{lemma:convolutionDiffBV}.
\begin{thm}[The Differential Method for BV Densities]
\label{thm:differentialMethodMasterTheoremBV}
Fix any norm $\| \cdot \|$ and
let $G: \R^d \to [0, 1]$ be the smoothing of any measurable $F: \R^d \to [0, 1]$ by $q(x)$ with BV density.
Let $\Phi: [0, 1] \to \R$ be given as in \cref{defn:PhiBV}.

Then for any $x$, if $G(x) < 1/2$, then $G(x + \delta) < 1/2$ for any
\begin{align*}
    \| \delta \| < \int_{G(x)}^{1/2} \f 1 {\Phi(p)} \dd p
    .
\end{align*}
\end{thm}

\subsection{Wulff Crystal Optimality for BV Densities}

Similarly, if we swap out usage of \cref{thm:coarea} with \cref{thm:coareaBV} and the usage of \cref{thm:regulargradvanish} with \cref{thm:regulargradvanishBV}, then we generalize \cref{thm:wulffOptimalGeneral} to distributions with BV densities.

\begin{thm}\label{thm:wulffOptimalGeneralBV}
Let $\mathcal B$ be a full-dimensional polytope in $\R^d$ symmetric around the origin, and whose vertices form a symmetric set.
Let $Z$ be the Wulff Crystal w.r.t. $\mathcal B$.
Let $q_0$ be a BV probability density function.
Among all probability distributions $q$ with BV and even density function that is level equivalent to $q_0$, the probability density function with concentric superlevel sets proportional to $Z$ minimize
\begin{align*}
    \Phi(1/2) = \sup_{\|v\|=1} \sup_{q(U) = 1/2} \lim_{r \searrow 0} \f{q(U-rv) - 1/2}r
    ,
\end{align*}
where $\|\cdot\|$ is the norm defined by $\mathcal B$.
\end{thm}

Likewise, \cref{thm:WulffCrystalDistributionOptimization} generalizes similarly to BV densities.

\begin{thm}\label{thm:WulffCrystalDistributionOptimizationBV}
    Let $\mathcal B$ be a full-dimensional polytope in $\R^d$ symmetric around the origin, and whose vertices form a symmetric set.
    Let $Z$ be the Wulff Crystal w.r.t. $\mathcal B$, and let $\mynorm{\cdot}$ denote the norm with $Z$ as its unit ball.
    Consider a probability distribution $q(x)$  on $\R^d, d \ge 2$ with even, BV density that depends only on the norm $\mynorm{x}$.
    Then with $\Phi$ defined against the adversary $\mathcal B$, we have, for any $k > 0$,
    \begin{align*}
        \Phi(1/2) \ge \f{(d-1)C}{\sqrt[k]{\EV_{\delta \sim q}\mynorm{\delta}^k}},
    \end{align*}
    where $C \defeq \f 1 2 \EV_{x \sim \pd Z} |\la \nabla \mynorm{x}, u \ra|$, for any vertex $u$ of $\mathcal B$, is a constant that depends only on $Z$.
    \end{thm}

\section{Robust Radii Derivations}
\subsection{IID Distributions}

In this section, we study smoothing distributions that have i.i.d.\ coordinates.

\subsubsection{\texorpdfstring{$\ell_1$ Adversary}{L1 Adversary}}

\subparagraph{IID Log Concave Distributions}

\begin{thm}\label{thm:iidlogconcaveL1}
Let $\phi: \R \to \R$ be absolutely continuous, even, and convex such that $\exp(-\phi(x))$ is integrable.
Suppose $H$ is a smoothed classifier smoothed by
\[q(x) \propto \prod_{i=1}^d e^{-\phi(x_i)},\]
such that $H(x) = (H(x)_1, \ldots, H(x)_C)$ is a vector of probabilities that $H$ assigns to each class $1, \ldots, C$.
If $H$ correctly predicts the class $y$ on input $x$, and the probability of the correct class is $\rho \defeq H(x)_y > 1/2$, then $H$ continues to predict the correct class when $x$ is perturbed by any $\eta$ with
\begin{align*}
    \|\eta\|_1 < \mathrm{CDF}_\phi^{-1}(\rho),
\end{align*}
where $\mathrm{CDF}_\phi^{-1}$ is the inverse CDF of the 1D random variable with density $\propto \exp(-\phi(x))$.
This robust radius is tight.
\end{thm}

\begin{proof}
We seek to apply \cref{thm:differentialMethodMasterTheorem} to $G(x) = 1 - H(x)_y$, for which we need to derive random variables $\gamma_u$ and $\mygamma{u}{p}$, and most importantly, the function $\Phi$.

WLOG, assume $u \in \Vrt(\Bb)$ is $e_1$.
Then $\gamma_u = \la u, \nabla \log q(\delta) \ra = \phi'(\delta_1)$, for $\delta \sim q$.
Let $X$ be the random variable whose density function is $\propto e^{-\phi(x)}$, and denote $\varphi(c) \defeq \Pr[X > c]$.
Thus $\gamma_u$ is distributed as $\phi'(X)$.
Since $\phi$ is convex, $\phi'$ is nondecreasing, so that $\mygamma{u}{p}$ is distributed like $\phi'(X) \ind(X > \varphi^{-1}(p))$ (using the fact that $X$ has an atomless measure).
Then with $C = \int_{-\infty}^\infty e^{-\phi(t)} \dd t,$ we have
\begin{align*}
    \Phi(p) =
    \EV \mygamma{u}{p}
        &=
        C^{-1}\int_{\varphi^{-1}(p)}^\infty
            e^{-\phi(t)} \phi'(t)
            \dd t\\
        &=
            C^{-1} e^{-\phi(\varphi^{-1}(p))}
        .
\end{align*}
Then with $p_0 = 1-\rho$, and by reparametrization the integral (\cref{lemma:reparam}), the certified radius is
\begin{align*}
    \int_{p_0}^{1/2} \f 1 {\Phi(p)} \dd p
    &= \int_{\varphi^{-1}(1/2)}^{\varphi^{-1}(p_0)}
        \f{|\varphi'(c)|}{C^{-1} e^{-\phi(c)}} \dd c
        \\
    &=
        \int_{\varphi^{-1}(1/2)}^{\varphi^{-1}(p_0)} \dd c
        \\
    &=
        \varphi^{-1}(p_0) - \varphi^{-1}(1/2)
        .
\end{align*}
Simplifying this in terms of the CDF, and noting that $\varphi^{-1}(1/2) = 0$ because $\phi$ is even, yields the expression in the theorem statement.

This robust radius is tight, as can be seen from the case when a half-plane $\{x: x_1 \ge s\}$ is the set of inputs that the base classifier assigns the label $y$.

\end{proof}

The same proof can be generalized straightforwardly to distributions with BV densities by using \cref{thm:differentialMethodMasterTheoremBV} (this, for example, yields another proof of the robust radii of uniform distribution against $\ell_1$ adversary).

\begin{thm}\label{thm:iidlogconcaveL1BV}
Let $q_1: \R \to \R$ be an even and convex function and assume it has bounded variations.
Suppose $H$ is a smoothed classifier smoothed by
\[q(x) \propto \prod_{i=1}^d q_1(x_i),\]
such that $H(x) = (H(x)_1, \ldots, H(x)_C)$ is a vector of probabilities that $H$ assigns to each class $1, \ldots, C$.
If $H$ correctly predicts the class $y$ on input $x$, and the probability of the correct class is $\rho \defeq H(x)_y > 1/2$, then $H$ continues to predict the correct class when $x$ is perturbed by any $\eta$ with
\begin{align*}
    \|\eta\|_1 < \mathrm{CDF}_{q_1}^{-1}(\rho),
\end{align*}
where $\mathrm{CDF}_{q_1}^{-1}$ is the inverse CDF of the 1D random variable with density $\propto q_1(x)$.
This robust radius is tight.
\end{thm}

\subparagraph{IID Log Convex* Distributions}

\begin{thm}\label{thm:iidlogconvexL1}
Let $\phi: [0, \infty) \to \R$ be absolutely continuous, concave, and nondecreasing, such that $\exp(-\phi(|x|))$ is integrable.
Suppose $H$ is a smoothed classifier smoothed by
\[q(x) \propto \prod_{i=1}^d e^{-\phi(|x_i|)}\]
such that $H(x) = (H(x)_1, \ldots, H(x)_C)$ is a vector of probabilities that $H$ assigns to each class $1, \ldots, C$.
If $H$ correctly predicts the class $y$ on input $x$, and the probability of the correct class is $\rho \defeq H(x)_y > 1/2$, then $H$ continues to predict the correct class when $x$ is perturbed by any $\eta$ with
\begin{align*}
    \|\eta\|_1
        &<
            \int^{\infty}_{\varphi^{-1}(1-\rho)}
            \f{\dd c}{e^{\phi(c) - \phi(0)} - 1}
            \\
        &=
            \int_{1-\rho}^{1/2} \f{C \dd p}{e^{-\phi(0)} - e^{-\phi(\varphi^{-1}(p))}}
\end{align*}
where $\varphi^{-1}$ is the inverse function of
\[\varphi(c) \defeq \Pr_{z \sim q}[0 \le z_1 \le c],\]
and
\begin{align*}
    C = \int_{-\infty}^\infty e^{-\phi(|t|)} \dd t.
\end{align*}
\end{thm}

\begin{proof}
We seek to apply \cref{thm:differentialMethodMasterTheorem} to $G(x) = 1 - H(x)_y$, for which we need to derive random variables $\gamma_u$ and $\mygamma{u}{p}$, and most importantly, the function $\Phi$.

WLOG, assume $u \in \Vrt(\Bb)$ is $e_1$.
Then $\gamma_u$ is the random variable $\la u, \nabla \log q(\delta) \ra = \phi'(|\delta_1|)\sgn(\delta_1)$ where $\delta \sim q$.
Let $X \in \R$ be the random variable whose density function is $\propto e^{-\phi(|x|)}$, and so $\varphi(c) = \Pr[0 \le X \le c]$.
Thus $\gamma_u$ is distributed as $\phi'(X)\sgn(X)$.
Since $\phi$ is concave, $\phi'$ is nonincreasing, so that for $p < 1/2$, $\mygamma{u}{p}$ is distributed as $\phi'(X) \ind(\varphi^{-1}(p) \ge X \ge 0)$ (using the fact that $X$ has an atomless measure).
Then with $C = \int_{-\infty}^\infty e^{-\phi(|t|)} \dd t,$
\begin{align*}
    \Phi(p) =
    \EV \mygamma{u}{p}
        &=
        C^{-1}\int_0^{\varphi^{-1}(p)}
            e^{-\phi(t)} \phi'(x)
            \dd t\\
        &=
            C^{-1} (e^{-\phi(0)} - e^{-\phi(\varphi^{-1}(p))})
    .
\end{align*}
Then by change of variables $c = \varphi^{-1}(p)$ and with $p_0 = 1-\rho$, the certified radius is
\begin{align*}
    \int_{p_0}^{1/2} \f 1 {\Phi(p)} \dd p
    &=
        \int^{\varphi^{-1}(1/2)}_{\varphi^{-1}(p_0)}
        \f{|\varphi'(c)|}{C^{-1} (e^{-\phi(0)} - e^{-\phi(c)})} \dd c
        \\
    &=
        \int^{\varphi^{-1}(1/2)}_{\varphi^{-1}(p_0)}
        \f{e^{-\phi(c)}}{e^{-\phi(0)} - e^{-\phi(c)}}
        \dd c
        \\
    &=
        \int^{\varphi^{-1}(1/2)}_{\varphi^{-1}(p_0)}
        \f{1}{e^{\phi(c) - \phi(0)} - 1}
        \dd c
        .
\end{align*}
Since $\phi$ is even, $\varphi^{-1}(1/2) = \infty$, yielding the desired result.

\end{proof}

\subparagraph{Corollaries}
The $\ell_p$ based exponential distribution $\propto e^{-\|x\|_p^p}$ has each coordinate is distributed as $\Rademacher \cdot \sqrt[p]{\Gammadist(1/p)}$.
When $p \ge 1$, it satisfies \cref{thm:iidlogconcaveL1}, so we obtain
\begin{cor}\label{cor:lpExponentialL1Adv}
    Suppose $H$ is a smoothed classifier smoothed by
    \[q(x) \propto e^{-\|x/\lambda\|_p^p}, p \ge 1,\]
    such that $H(x) = (H(x)_1, \ldots, H(x)_C)$ is a vector of probabilities that $H$ assigns to each class $1, \ldots, C$.
    If $H$ correctly predicts the class $y$ on input $x$, and the probability of the correct class is $\rho \defeq H(x)_y > 1/2$, then $H$ continues to predict the correct class when $x$ is perturbed by any $\eta$ with
    \begin{align*}
        \|\eta\|_1 < \lambda\sqrt[p]{\GammaCDF^{-1}(2\rho - 1; 1/p)},
    \end{align*}
    where $\mathrm{CDF}_\phi^{-1}$ is the inverse CDF of the 1D random variable with density $\propto \exp(-\phi(x))$.
    This robust radius is tight.
\end{cor}

\renewcommand{\polylog}{\operatorname{polylog}}

When $p < 1$, it satisfies \cref{thm:iidlogconvexL1}, so we obtain
\begin{cor}\label{cor:lpExponentialLogconvexL1Adv}
    Suppose $H$ is a smoothed classifier smoothed by
    \[q(x) \propto e^{-\|x/\lambda\|_p^p}, p < 1,\]
    such that $H(x) = (H(x)_1, \ldots, H(x)_C)$ is a vector of probabilities that $H$ assigns to each class $1, \ldots, C$.
    If $H$ correctly predicts the class $y$ on input $x$, and the probability of the correct class is $\rho \defeq H(x)_y > 1/2$, then $H$ continues to predict the correct class when $x$ is perturbed by any $\eta$ with
    \begin{align*}
        \|\eta\|_1 &< \lambda \int^{\infty}_{\varphi^{-1}(1-\rho)}
        \f{\dd c}{e^{c^p} - 1}\\
            &=
                \lambda \int_{1-\rho}^{1/2} \f{2 \Gamma(1 + \f 1 p) \dd p_0}{1 - e^{-|\varphi^{-1}(p_0)|^p}},
    \end{align*}
    where $\varphi^{-1}(p_0) \defeq \GammaCDF^{-1}(2p_0; 1/p)^{1/p}.$
\end{cor}
The integral above can be evaluated explicitly for inverse integer $p = 1/k$.
We show a few examples below:
\begin{align*}
    p&=1/2: & R&= 2 \lambda (-c \log (1 - e^{-c}) + \polylog (2, e^{-c}))\\
    p&=1/3: & R&= 3\lambda \bigg(- c^2 \log (1 - e^{-c}) \\
    &&&\pushright{{}+ 2 c \polylog (2, e^{-c}) + 2 \polylog(3, e^{-c})\bigg)}\\
    p&=1/4: & R&= 4\lambda\bigg(-c^3 \log(1 - e^{-c})\\
    &&&\pushright{{}+ 3 c^2 \polylog(2, e^{-c}) + 6 c \polylog(3, e^{-c})}\ \\
    &&&\pushright{{}+ 6 \polylog(4, e^{-c})\bigg)}
\end{align*}
where $c = \GammaCDF^{-1}(2(1-\rho); 1/p)$ for each $p$, and $\polylog$ is the function defined as
\begin{align*}
    \polylog(n, z) = \sum_{k=1}^\infty z^k / k^n.
\end{align*}
\subsection{\texorpdfstring{$\ell_\infty$}{Linf} Norm-Based Exponential Law}
\label{sec:cubicalExponential}
In this section, we derive robustness guarantees for distributions of the form $q(x) \propto \|x\|_\infty^{-j} \exp(-\|x\|_\infty^k)$.

\subsubsection{\texorpdfstring{$\ell_1$ Adversary}{L1 Adversary}}
\label{sec:cubicalExponentialL1}
In this section, we set the norm $\| \cdot \|$ in \cref{assm:FGqsmooth} to be the $\ell_1$ norm $\|x\|_1 = \sum_{i=1}^d |x_i|$.
Then the unit ball $\Bb$ in \cref{assm:FGqsmooth} is the convex hull of its vertices which are the coordinates vectors and their negations:
\[\Vrt(\Bb) = \{\pm e_i: i \in [d]\}.\]

\subparagraph{Overview} $\ell_\infty$ norm-based distributions will in general have certified radius that is linear in $\rho - 1/2$, where $\rho$ is the probability that the \emph{smoothed} classifier assigns to the correct class.

We first demonstrate the differential method on $q(x) \propto \exp(-\|x\|_\infty)$ as a warmup before stating the more general result.

\begin{thm}\label{thm:ExpInfL1}
Suppose $H$ is a smoothed classifier smoothed by
\[q(x) \propto \exp(-\|x\|_\infty/\lambda),\]
such that $H(x) = (H(x)_1, \ldots, H(x)_C)$ is a vector of probabilities that $H$ assigns to each class $1, \ldots, C$.
If $H$ correctly predicts the class $y$ on input $x$, and the probability of the correct class is $\rho \defeq H(x)_y > 1/2$, then $H$ continues to predict the correct class when $x$ is perturbed by any $\eta$ with
\begin{align*}
    \|\eta\|_1 <
    \begin{cases}
        2d\lambda (\rho - \f 1 2)
            & \text{if $\rho \le 1 - \f 1 {2d}$}
                \\
        \lambda \log \f 1 {2d(1 - \rho)}
        + \lambda (d-1)
            & \text{if $\rho > 1 - \f 1 {2d}$.}
    \end{cases}
\end{align*}
\end{thm}

\begin{proof}
By linearity in $\lambda$, it suffices to show this for $\lambda = 1$.
Here, we have $q(x) \propto \exp(\psi(x))$ with
\begin{align*}
    \psi(x) = \|x\|_\infty
    \quad\text{and}\quad
    \nabla \psi(x) = \sgn(x_{i^*}) e_{i^*},
\end{align*}
where $i^* = \argmax_i |x_i|$, and $e_{i^*}$ is the $i^*$th coordinate vector, with $\nabla \psi(x)$ defined whenever $i^*$ is the unique argmax.

We seek to apply \cref{thm:differentialMethodMasterTheorem} to $G(x) = 1 - H(x)_y$, for which we need to derive random variables $\gamma_u$ and $\mygamma{u}{p}$, and most importantly, the function $\Phi$.

For any $u \in \Vrt(\Bb)$ (i.e.\ $u = \pm e_i$), the random variable $\gamma_u = \la u, \nabla \psi(\delta) \ra = \la u, \sgn(\delta_{i^*})e_{i^*}\ra, \delta \sim q,$ is given by
\begin{align*}
    \gamma_u =
    \begin{cases}
    0 & \text{with prob.\ $1 - \f 1 d$}\\
    1 & \text{with prob.\ $\f 1 {2d}$}\\
    -1 & \text{with prob.\ $\f 1 {2d}$.}\\
    \end{cases}
\end{align*}
Therefore, for $p \in [0, 1/2]$, the random variable $\mygamma{u}{p}$ defined in \cref{defn:Phi} is
\begin{align*}
    \begin{cases}
    \mygamma{u}{p} = \begin{cases}
                1 & \text{with prob.\ $\f 1 {2d}$}\\
                0 & \text{with prob.\ $1 - \f 1 {2d}$}
            \end{cases}
        & \text{if $p \in [\f 1 {2d},\f 1 2]$,}
            \\
    \mygamma{u}{p} = \begin{cases}
                1 & \text{with prob.\ $p$}\\
                0 & \text{with prob.\ $1 - p$}
            \end{cases}
        & \text{if $p \in [0, \f 1 {2d}]$.}
    \end{cases}
\end{align*}
Thus, for any $u \in \Vrt(\Bb)$,
\begin{align*}
    \Phi(p) = \EV \mygamma{u}{p} =
    \begin{cases}
        \f 1 {2d} & \text{if $p \in [ \f 1 {2d}, \f 1 2]$}\\
        p & \text{if $p \in [0, \f 1 {2d}]$.}
    \end{cases}
\end{align*}
Then, by setting $G(x)$ in \cref{thm:differentialMethodMasterTheorem} to be $1 - H(x)_y = 1 - \rho \defeq p_0$, we get the provably robust radius of
\begin{align*}
    &\phantomeq
        \int_{p_0}^{1/2} \f 1 {\Phi(p)} \dd p
        \\
    &=
    \begin{cases}
    \int_{p_0}^{1/2} 2d \dd p
        = 2d(\f 1 2 - p_0)
        & \text{if $p_0 \ge \f 1 {2d}$}\\
    \int_{p_0}^{1/2d} \inv p \dd p + \int_{1/2d}^{1/2} 2d \dd p
        & \text{if $p_0 \le \f 1 {2d}$.}
    \end{cases}
\end{align*}
Simplifying the arithmetics yields the desired claim.
\end{proof}

Now we tackle the general case.

\begin{thm}\label{thm:ExpInfkjL1}
Suppose $H$ is a smoothed classifier smoothed by
\[q(x) \propto (\|x\|_\infty/\lambda)^{-j}\exp(-(\|x\|_\infty/\lambda)^k), k \ge 1,\]
such that $H(x) = (H(x)_1, \ldots, H(x)_C)$ is a vector of probabilities that $H$ assigns to each class $1, \ldots, C$.
If $H$ correctly predicts the class $y$ on input $x$, and the probability of the correct class is $\rho \defeq H(x)_y \in (1/2, 1- \f 1 {2d}]$, then $H$ continues to predict the correct class when $x$ is perturbed by any $\eta$ with
\begin{align*}
    \|\eta\|_1 <
        \f {2d\lambda} {d-1}
        \f{\Gamma\lp \f {d-j}k\rp}
          {\Gamma\lp \f {d-1-j}k \rp}
        \lp \rho - \f 1 2\rp.
\end{align*}
\end{thm}

\begin{proof}
By linearity in $\lambda$, it suffices to show this for $\lambda = 1$.
Here, we have
\begin{align*}
    q(x) &\propto \exp(-\|x\|_\infty^k-j \log \|x\|_\infty)
    \quad
    &\text{so that}\\
    \psi(x) &= \|x\|_\infty^k + j \log \|x\|_\infty\\
    \nabla \psi(x) &= (k\|x\|^{k-1}_\infty + j \|x\|_\infty^{-1}) \sgn(x_{i^*}) e_{i^*},
\end{align*}
where $i^* = \argmax_i |x_i|$, and $e_{i^*}$ is the $i^*$th coordinate vector, with $\nabla \psi(x)$ defined whenever $i^*$ is the unique argmax.

We seek to apply \cref{thm:differentialMethodMasterTheorem} to $G(x) = 1 - H(x)_y$, for which we need to derive random variables $\gamma_u$ and $\mygamma{u}{p}$, and most importantly, the function $\Phi$.

WLOG among $\Vrt(\Bb)$, let's assume $u = e_1$.
Then the random variable $\gamma_u = \la u, \nabla \psi(\delta) \ra, \delta \sim q,$ is 0 with probability $1 - \f 1 d$, when $i^* \ne 1$.
When $i^* = 1$ and $\sgn(x_{i^*}) = 1$ (which happens with probability $\f 1 {2d}$), $\gamma_u$ is $k \|x\|_\infty^{k-1} + j \|x\|_\infty^{-1}$, where $x \sim q$.
By \cref{lemma:expNormSampling}, this is just the random variable $k z^{\f{k-1}k} + j z^{-1}$, where $z \sim \Gammadist(d/k)$.
Likewise, with probability $\f 1 {2d}$, $\gamma_u$ is the random variable $-k z^{\f{k-1}k} - j z^{-1}$.
This can be summarized below.
\begin{align*}
    \gamma_u =
    \begin{cases}
    0 & \text{with prob.\ $1 - \f 1 d$}\\
    k z^{\f{k-1}k} + j z^{-1} & \text{with prob.\ $\f 1 {2d}$}\\
    - k z^{\f{k-1}k} - j z^{-1} & \text{with prob.\ $\f 1 {2d}$,}\\
    \end{cases}
\end{align*}
where $z \sim \Gammadist(d/k)$.

Therefore, for $p \in [\f 1 {2d}, \f 1 2]$, the random variable $\mygamma{u}{p}$ defined in \cref{defn:Phi} is
\begin{align*}
    \mygamma{u}{p} = \begin{cases}
                kz^{\f{k-1}k} + j z^{-1} & \text{with prob.\ $\f 1 {2d}$}\\
                0 & \text{with prob.\ $1 - \f 1 {2d}$}
            \end{cases}
\end{align*}
where $z$ is sampled from $\Gammadist(d/k)$.

Thus, for any $u \in \Vrt(\Bb)$, by \cref{lemma:GammaExpect},
\begin{align*}
    \Phi(p) = \EV \mygamma{u}{p} =
        \f 1 {2d}\EV kz^{\f{k-1}k}
        = \f {d-1} {2d} \f
                {\Gamma\lp \f {d-1-j}k \rp}
                {\Gamma\lp \f {d-j}k\rp}
\end{align*}
which does not depend on $p$.

Then, by setting $G(x)$ in \cref{thm:differentialMethodMasterTheorem} to be $1 - H(x)_y = 1 - \rho$, we get the provably robust radius of
\begin{align*}
    \int_{1-\rho}^{1/2} \f 1 {\Phi(p)} \dd p =
    \f {2d} {d-1} \f
        {\Gamma\lp \f {d-j}k\rp}
        {\Gamma\lp \f {d-1-j}k \rp}
    \lp \rho - \f 1 2 \rp
\end{align*}
as desired.

\end{proof}

As $j = 0$ and $k \to \infty$, the distribution above converges to the uniform distribution, and the robust certificate converges likewise to the one computed previous for uniform distribution.

\begin{thm}[\citet{lee_tight_2019}]\label{thm:UniformL1}
Suppose $H$ is a smoothed classifier smoothed by the uniform distribution on the cube $[-\lambda, \lambda]^d$,
such that $H(x) = (H(x)_1, \ldots, H(x)_C)$ is a vector of probabilities that $H$ assigns to each class $1, \ldots, C$.
If $H$ correctly predicts the class $y$ on input $x$, and the probability of the correct class is $\rho \defeq H(x)_y > 1/2$, then $H$ continues to predict the correct class when $x$ is perturbed by any $\eta$ with
\begin{align*}
    \|\eta\|_1 <
        2\lambda \lp \rho - \f 1 2 \rp.
\end{align*}
\end{thm}

\subsubsection{\texorpdfstring{$\ell_\infty$ Adversary}{Linf Adversary}}

In this section, we set the norm $\| \cdot \|$ in \cref{assm:FGqsmooth} to be the $\ell_\infty$ norm $\|x\|_\infty = \max_{i=1}^d |x_i|$.
Then the unit ball $\Bb$ in \cref{assm:FGqsmooth} is the convex hull of its vertices which are points in the Boolean cube:
\[\Vrt(\Bb) = \{\pm1\}^d.\]

\begin{thm}\label{thm:ExpInfLinfty}
Suppose $H$ is a smoothed classifier smoothed by
\[q(x) \propto \exp(-\|x\|_\infty/\lambda),\]
such that $H(x) = (H(x)_1, \ldots, H(x)_C)$ is a vector of probabilities that $H$ assigns to each class $1, \ldots, C$.
If $H$ correctly predicts the class $y$ on input $x$, and the probability of the correct class is $\rho \defeq H(x)_y > 1/2$, then $H$ continues to predict the correct class when $x$ is perturbed by any $\eta$ with
\begin{align*}
    \|\eta\|_\infty < \lambda \log \f{1}{2(1- \rho)}.
\end{align*}
\end{thm}

\begin{proof}
By linearity in $\lambda$, it suffices to show this for $\lambda = 1$.

We seek to apply \cref{thm:differentialMethodMasterTheorem} to $G(x) = 1 - H(x)_y$, for which we need to derive random variables $\gamma_u$ and $\mygamma{u}{p}$, and most importantly, the function $\Phi$.

For any $u \in \Vrt(\Bb)$, the random variable $\gamma_u = \la u, \nabla \psi(\delta) \ra = \la u, \sgn(\delta_{i^*})e_{i^*}\ra, \delta \sim q,$ is a Rademacher random variable taking values $\pm 1$ with equal probability.
Therefore, for $p \in [0, 1/2]$, the random variable $\mygamma{u}{p}$ defined in \cref{defn:Phi} is
\begin{align*}
    \mygamma{u}{p} = \begin{cases}
                1 & \text{with prob.\ $p$}\\
                0 & \text{with prob.\ $1 - p$.}
            \end{cases}
\end{align*}
Thus, for any $u \in \Vrt(\Bb)$,
\begin{align*}
    \Phi(p) = \EV \mygamma{u}{p} = p.
\end{align*}
Then, by setting $G(x)$ in \cref{thm:differentialMethodMasterTheorem} to be $1 - H(x)_y = 1 - \rho$, we get the provably robust radius of
\begin{align*}
    \int_{1-\rho}^{1/2} \f 1 {\Phi(p)} \dd p = \int_{1-\rho}^{1/2} \f 1 {p} \dd p = \log \f{1}{2(1- \rho)}.
\end{align*}
\end{proof}

\begin{thm}\label{thm:ExpInfjkLinfty}
Suppose $H$ is a smoothed classifier smoothed by
\[q(x) \propto \exp(-\|x/\lambda\|_\infty^k), k \ge 1\]
such that $H(x) = (H(x)_1, \ldots, H(x)_C)$ is a vector of probabilities that $H$ assigns to each class $1, \ldots, C$.
If $H$ correctly predicts the class $y$ on input $x$, and the probability of the correct class is $\rho \defeq H(x)_y > 1/2$, then $H$ continues to predict the correct class when $x$ is perturbed by any $\eta$ with
\begin{align}
    \|\eta\|_\infty < \lambda \int_{1-\rho}^{1/2} \f 1 {\Phi(p)} \dd p,
    \label{eqn:ExpInfkLinftyRadius}
\end{align}
in which
\begin{align*}
    \Phi(p) &\defeq C \lp 1 - \GammaCDF\left(c^*(p); \f{d+k-1}k\right) \rp,\\
        &\pushright{\text{where $c^*(p) \defeq \GammaCDF^{-1}\left(1-2p; \f d k\right)$,}}\\
        &\pushright{C \defeq \f {k} 2 \f{\Gamma\left(\f{d+k-1}k\right)}{\Gamma\left(\f d k\right)}.}
\end{align*}
More generally, if the smoothing distribution is
\[q(x) \propto \|x/\lambda\|_\infty^{-j}\exp(-\|x/\lambda\|_\infty^k), k \ge 1, j < d - 1,\]
then $H$ is robust against $\ell_\infty$ perturbation
\begin{align}
    \|\eta\|_\infty < \lambda \int_{1-\rho}^{1/2} \f 1 {\Phi(p)} \dd p,
    \label{eqn:ExpInfjkLinftyRadius}
\end{align}
where
\begin{align*}
    \Phi(p) &\defeq \f 1 2 \bar \phi(\phi^{-1}(2p)), \quad \text{where}\\
    \phi(c) &\defeq \Pr[\gamma > c]\\
    \bar \phi(c) &\defeq \EV \gamma \ind( \gamma > c)
\end{align*}
and $\gamma \defeq (k-1)\xi^{\f{k-1}k} + j \xi^{-\f 1 k}, \xi \sim \Gammadist\lp \f d k - \f j k \rp$.
\end{thm}

\begin{proof}
By linearity in $\lambda$, it suffices to show this for $\lambda = 1$.

We seek to apply \cref{thm:differentialMethodMasterTheorem} to $G(x) = 1 - H(x)_y$, for which we need to derive random variables $\gamma_u$ and $\mygamma{u}{p}$, and most importantly, the function $\Phi$.

For any $u \in \Vrt(\Bb)$, we have
\begin{align*}
    \gamma_u
        &=
            \la u, -\nabla \log q(\delta) \ra\\
        &=
            \la u, (k\|\delta\|_\infty^{k-1} + j \|\delta\|_\infty^{-1})\sgn(\delta_{i^*})e_{i^*}\ra,
            \delta \sim q.
\end{align*}
Since $\|\delta\|_\infty$ is distributed as $\sqrt[k]{\Gammadist(\f d k - \f j k)}$ and $\la u, \sgn(\delta_{i^*})e_{i^*}\ra$ is $\pm1$ with equal probability, $\gamma_u$ is distributed as the product of random variables
\begin{align*}
    \gamma_u &= \zeta (k\xi^{\f{k-1}k} + j \xi^{-\f 1 k}),\\
    &\quad \zeta \sim \Rademacher, \xi \sim \Gammadist\left(\f d k - \f j k\right).
\end{align*}
Let $\varphi(c) \defeq \Pr[\gamma_u > c]$.
Then for $p < 1/2$,  $\phi^{-1}(2p) = \varphi^{-1}(p)$.
Since $\gamma_u$ has an absolutely continuous distribution, the variable $\mygamma{u}{p} = \gamma_u|_{(\varphi^{-1}(p), \infty)}$ with probability $p$, and $0$ otherwise.
Thus
\begin{align*}
    \Phi(p) &= \EV \mygamma{u}{p} = \bar\varphi(\varphi^{-1}(p)), \quad\text{where}\\
    \bar \varphi(c) &= \EV \gamma_u \ind(\gamma_u > c).
\end{align*}
Note that $\bar \varphi(c) = \f 1 2 \bar \phi(c)$.
Plugging into \cref{thm:differentialMethodMasterTheorem} yields \cref{eqn:ExpInfjkLinftyRadius}.

\subparagraph{Assuming $j =0$}
If $j = 0$, $\gamma_u = \zeta k\xi^{\f{k-1}k}, \zeta \sim \Rademacher, \xi \sim \Gammadist\left(\f d k \right)$.
Then for $p < 1/2$,
\begin{align*}
    \bar \varphi(c)
        &=
            \f {k} 2 \EV \xi^{\f{k-1}k} \ind(\xi > c^*),\\
        &\quad
            \text{where $c^* = \GammaCDF^{-1}\left(1-2p; \f d k\right)$}\\
        &=
            C \lp 1 - \GammaCDF\left(c^*; \f{d+k-1}k\right)\rp,
\end{align*}
where $C = \f {k} 2 \f{\Gamma\left(\f{d+k-1}k\right)}{\Gamma\left(\f d k\right)}$.
Plugging into \cref{thm:differentialMethodMasterTheorem} yields \cref{eqn:ExpInfkLinftyRadius}.
\end{proof}

Compare this with the uniform case below.

\begin{thm}[\citet{lee_tight_2019}]\label{thm:UniformLinfty}
    Suppose $H$ is a smoothed classifier smoothed by the uniform distribution on the cube $[-\lambda, \lambda]^d$,
    such that $H(x) = (H(x)_1, \ldots, H(x)_C)$ is a vector of probabilities that $H$ assigns to each class $1, \ldots, C$.
    If $H$ correctly predicts the class $y$ on input $x$, and the probability of the correct class is $\rho \defeq H(x)_y > 1/2$, then $H$ continues to predict the correct class when $x$ is perturbed by any $\eta$ with
    \begin{align*}
        \|\eta\|_\infty <
            2\lambda \lp 1 - \sqrt[d]{\f 3 2 - \rho} \rp.
    \end{align*}
\end{thm}
When $d \to \infty$, this robust radius is roughly
\begin{align*}
    &\phantomeq 2\lambda \lp 1 - e^{\f 1 d \log\left[1 - \left(\rho - \f 1 2\right)\right]}\rp\\
    &\approx 2\lambda \lp 1 - \lp 1 + \f 1 d \log\left[1 - \left(\rho - \f 1 2\right)\right]\rp\rp\\
    &\approx \f {2\lambda} d \lp\rho - \f 1 2\rp.
\end{align*}

On the other hand, when $k \to \infty$ in \cref{eqn:ExpInfkLinftyRadius}, we see that
\begin{enumerate}
    \item $d/k \to 0$ while $\f{d+k-1}k \to 1$
    \item $c^* \to 0$ for any $p < 1/2$
    \item $\GammaCDF(c^*; \f{d+k-1}k) \to \GammaCDF(0; 1) = 1$ consequently
    \item by simple calculation $k \f{\Gamma(\f{d+k-1}k)}{\Gamma(\f d k)} \to d$
    \item so $\Phi(p) \to d/2$ for any $p < 1/2$.
\end{enumerate}
Therefore, when $k \to \infty$, the $\ell_\infty$ robust radius in \cref{eqn:ExpInfkLinftyRadius} converges to
\[\f {2\lambda} d \lp\rho - \f 1 2\rp\]
as well.

\subsection{\texorpdfstring{$\ell_\infty$}{Linf} Norm-Based Power Law}

Now consider a power law of the $\ell_\infty$ norm:
For $a > d$,
\begin{align*}
    q(x) &\propto (1 + \|x\|_\infty)^{-a}
    &\text{so that}\\
    \psi(x) &= a \log (1 + \|x\|_\infty)\\
    \nabla \psi(x) &= a(1 + \|x\|_\infty)^{-1} \sgn(x_{i^*}) e_{i^*},
\end{align*}
where $i^* = \argmax_i |x_i|$, and $e_{i^*}$ is the $i^*$th coordinate vector, with $\nabla \psi(x)$ defined whenever $i^*$ is the unique argmax.
Note that the $\ell_\infty$ norm of vector sampled from $q$ has distribution with CDF
\begin{align}
    \Pr_{\delta \sim q}[\|\delta\|_\infty \le c] =
    \f{\Gamma(a)}{\Gamma(a-d)\Gamma(d)}\int_0^c \f{r^{d-1}}{(1 + r)^a} \dd r
    .
    \label{eqn:powerlawRdist}
\end{align}
This is known as the \emph{Beta prime or Beta distribution of the second kind, with shape parameters $\alpha = d, \beta = a - d$,} which we will denote by $\BetaPrime(d, a-d)$.
If $a > d+1$, this distribution has mean
\begin{align*}
    \EV_{\delta \sim q}\|\delta\|_\infty =
    \f{d}{a-d-1}
        .
\end{align*}

\subsubsection{\texorpdfstring{$\ell_1$ Adversary}{L1 Adversary}}
\label{sec:cubicalPowerL1}

\begin{thm}\label{thm:ExpInfpolyL1}
Suppose $H$ is a smoothed classifier smoothed by
\[q(x) \propto (1 + \|x\|_\infty/\lambda)^{-a}, a > d,\]
such that $H(x) = (H(x)_1, \ldots, H(x)_C)$ is a vector of probabilities that $H$ assigns to each class $1, \ldots, C$.
If $H$ correctly predicts the class $y$ on input $x$, and the probability of the correct class is $\rho \defeq H(x)_y > 1/2$, then $H$ continues to predict the correct class when $x$ is perturbed by any $\eta$ with
\begin{align*}
    \|\eta\|_1 < \lambda \f{2d}{a-d}
                        \lp \rho - \f 1 2 \rp
                        .
\end{align*}
\end{thm}

\begin{proof}
By linearity in $\lambda$, it suffices to show this for $\lambda = 1$.

We seek to apply \cref{thm:differentialMethodMasterTheorem} to $G(x) = 1 - H(x)_y$, for which we need to derive random variables $\gamma_u$ and $\mygamma{u}{p}$, and most importantly, the function $\Phi$.

WLOG among $\Vrt(\Bb)$, let's assume $u = e_1$.
Then the random variable $\gamma_u = \la u, \nabla \psi(\delta) \ra, \delta \sim q,$ is 0 with probability $1 - \f 1 d$, when $i^* \ne 1$.
When $i^* = 1$ and $\sgn(x_{i^*}) = 1$ (which happens with probability $\f 1 {2d}$), $\gamma_u$ is distributed as  $a(1 + r)^{-1}$, where $r$ has CDF \cref{eqn:powerlawRdist}.
Likewise, with probability $\f 1 {2d}$, $\gamma_u$ is distributed as $-a (1 + r)^{-1}$.

This can be summarized below.
\begin{align*}
    \gamma_u =
    \begin{cases}
    0 & \text{with prob.\ $1 - \f 1 d$}\\
    a (1 + r)^{-1} & \text{with prob.\ $\f 1 {2d}$}\\
    -a (1 + r)^{-1} & \text{with prob.\ $\f 1 {2d}$,}\\
    \end{cases}
\end{align*}
where $r$ is a random variable with CDF \cref{eqn:powerlawRdist}.

Therefore, for $p \in [\f 1 {2d}, \f 1 2]$, the random variable $\mygamma{u}{p}$ defined in \cref{defn:Phi} is
\begin{align*}
    \mygamma{u}{p} = \begin{cases}
                a (1 + r)^{-1} & \text{with prob.\ $\f 1 {2d}$}\\
                0 & \text{with prob.\ $1 - \f 1 {2d}$.}
            \end{cases}
\end{align*}

Thus, for any $u \in \Vrt(\Bb)$, by \cref{lemma:GammaExpect},
\begin{align*}
    \Phi(p) &= \EV \mygamma{u}{p} =
        \f 1 {2d}\EV_r a (1 + r)^{-1}
            \\
    &=
        \f 1 {2d}\f{a\Gamma(a)}{\Gamma(a-d)\Gamma(d)}\int_0^\infty a\f{r^{d-1}}{(1 + r)^{a+1}} \dd r
        \\
    &=
        \f 1 {2d}\f{\Gamma(a+1)}{\Gamma(a-d)\Gamma(d)}\f{\Gamma(a+1-d)\Gamma(d)}{\Gamma(a+1)}
        \\
    &=
        \f {a-d}{2d}.
\end{align*}
which does not depend on $p$.

Then, by setting $G(x)$ in \cref{thm:differentialMethodMasterTheorem} to be $1 - H(x)_y = 1 - \rho$, we get the provably robust radius of
\begin{align*}
    \int_{1-\rho}^{1/2} \f 1 {\Phi(p)} \dd p =
    \f{2d}{a-d}
    \lp \rho - \f 1 2 \rp
\end{align*}
as desired.
\end{proof}

\subsubsection{\texorpdfstring{$\ell_\infty$ Adversary}{Linf Adversary}}

\begin{thm}\label{thm:ExpInfpolyLinf}
    Suppose $H$ is a smoothed classifier smoothed by
    \[q(x) \propto (1 + \|x\|_\infty/\lambda)^{-a}, a > d,\]
    such that $H(x) = (H(x)_1, \ldots, H(x)_C)$ is a vector of probabilities that $H$ assigns to each class $1, \ldots, C$.
    If $H$ correctly predicts the class $y$ on input $x$, and the probability of the correct class is $\rho \defeq H(x)_y > 1/2$, then $H$ continues to predict the correct class when $x$ is perturbed by any $\eta$ with
    \begin{align*}
        \|\eta\|_1 < \f{2\lambda}{a-d} \int_{1-\rho}^{1/2} \f {\dd p} {\Upsilon(\Upsilon^{-1}(2p; d, a-d); d, a+1-d)}
                            ,
    \end{align*}
    where $\Upsilon = \BetaPrimeCDF$.
\end{thm}

\begin{proof}
    By linearity in $\lambda$, it suffices to show this for $\lambda = 1$.

    We seek to apply \cref{thm:differentialMethodMasterTheorem} to $G(x) = 1 - H(x)_y$, for which we need to derive random variables $\gamma_u$ and $\mygamma{u}{p}$, and most importantly, the function $\Phi$.

    For any $u \in \Vrt(\Bb)$, we have
    \begin{align*}
        \gamma_u
            &=
                \la u, -\nabla \log q(\delta) \ra\\
            &=
                \la u, a(1 + \|\delta\|_\infty)^{-1} \sgn(\delta_{i^*}) e_{i^*}\ra,
                \delta \sim q.
    \end{align*}

    Since $\|\delta\|_\infty$ is distributed as $\BetaPrime(d, a-d)$ and $\la u, \sgn(\delta_{i^*})e_{i^*}\ra$ is $\pm1$ with equal probability, $\gamma_u$ is distributed as the product of random variables
    \begin{align*}
        \gamma_u &= \zeta a (1 + \xi)^{-1},\\
        &\quad \zeta \sim \Rademacher, \xi \sim \BetaPrime(d, a-d).
    \end{align*}
    Since $r \mapsto (1 + r)^{-1}$ is a decreasing function on $r \in [0, \infty)$, we have, for $p < 1/2$,
    \begin{align*}
        \Phi(p)
            &=
                \EV \mygamma{u}{p} = \f 1 2 \EV_\xi a (1 + \xi)^{-1} \ind(\xi < c(p)),
                \\
            &\pushright{\text{where }
                c(p) = \BetaPrimeCDF^{-1}(2p; d, a-d).}
    \end{align*}
    Of course, we can simplify
    \begin{align*}
        &\phantomeq
            \EV_\xi (1 + \xi)^{-1} \ind(\xi < c)\\
        &=
            \f{\Gamma(a)}{\Gamma(a-d)\Gamma(d)}
            \int_0^c \f{r^{d-1}}{(1 + r)^{a+1}} \dd r\\
        &=
            \f{\Gamma(a)}{\Gamma(a-d)\Gamma(d)}
            \f{\Gamma(a+1-d)\Gamma(d)}{\Gamma(a+1)}\times
            \\
        &\pushright{
            \BetaPrimeCDF(c; d, a+1-d)}
            \\
        &=
            \f{(a-d)}{a}
            \BetaPrimeCDF(c; d, a+1-d)
            .
    \end{align*}
    Therefore,
    \begin{align*}
        \Phi(p) = \f {a-d} 2 \BetaPrimeCDF(c(p); d, a+1-d).
    \end{align*}
    Plugging into \cref{thm:differentialMethodMasterTheorem} yields the desired result.
\end{proof}

\subsection{\texorpdfstring{$\ell_1$}{L1} Norm-Based Exponential Law}

Consider the following generalization of the Laplace distribution
\begin{align*}
    q(x) &\propto \exp(-\|x\|_1^k)
    \quad
    \text{so that}\\
    \psi(x) &= \|x\|_1^k\\
    \nabla \psi(x) &= k\|x\|_1^{k-1} (\sgn(x_1), \ldots, \sgn(x_d)),
\end{align*}
with $\nabla \psi(x)$ defined whenever all $x_i$s are nonzero.

\subsubsection{\texorpdfstring{$\ell_1$ Adversary}{L1 Adversary}}

\begin{thm}\label{thm:Exp1kL1}
\renewcommand{\GammaCDF}{\Upsilon}
Suppose $H$ is a smoothed classifier smoothed by
\[q(x) \propto \exp(-(\|x\|_1/\lambda)^k),\]
such that $H(x) = (H(x)_1, \ldots, H(x)_C)$ is a vector of probabilities that $H$ assigns to each class $1, \ldots, C$.
If $H$ correctly predicts the class $y$ on input $x$, and the probability of the correct class is $\rho \defeq H(x)_y > 1/2$, then $H$ continues to predict the correct class when $x$ is perturbed by any $\eta$ with
\begin{align*}
    \|\eta\|_1 < \lambda
        \int_{1-\rho}^{1/2} \f R {\Psi(p)} \dd p
                    .
\end{align*}
Here $R = \f
{2\Gamma\lp \f d k \rp}
{k \Gamma\lp \f{d+k-1}k \rp}$, and
\begin{align*}
    \Psi(p) \defeq
    \begin{cases}
    1 - \GammaCDF\lp \GammaCDF^{-1}(1 - 2 p; \f d k); \f{d+k-1}k\rp
        &   \text{if $k \ge 1$}\\
    \GammaCDF\lp \GammaCDF^{-1}(2 p; \f d k); \f{d+k-1}k\rp
        &   \text{if $k \in (0, 1)$.}
    \end{cases}
\end{align*}
where $\GammaCDF = \mathrm{GammaCDF}$.
\end{thm}

\begin{proof}
\renewcommand{\GammaCDF}{\Upsilon}

By linearity in $\lambda$, it suffices to show this for $\lambda = 1$.

We seek to apply \cref{thm:differentialMethodMasterTheorem} to $G(x) = 1 - H(x)_y$, for which we need to derive random variables $\gamma_u$ and $\mygamma{u}{p}$, and most importantly, the function $\Phi$.

For any $u \in \Vrt(\Bb)$ (i.e.\ $u = \pm e_i$), $\gamma_u = \la u, \nabla \psi(\delta) \ra = \pm k \|\delta\|_1^{k-1}, \delta \sim q,$ takes positive or negative sign with equal probability.
By \cref{lemma:expNormSampling}, $\|\delta\|_1$ is the random variable $\Gamma(d/k)^{1/k}$.
Therefore, $\gamma_u$ is distributed as $ k\Gammadist(d/k)^{\f{k-1}k}\Rademacher(1/2)$.

Therefore, for $p \in [0, 1/2]$, the random variable $\mygamma{u}{p}$ defined in \cref{defn:Phi} is
\begin{align*}
    \mygamma{u}{p} = \begin{cases}
                k z_p^{\f{k-1}k} & \text{with prob.\ $p$}\\
                0 & \text{with prob.\ $1 - p$,}
            \end{cases}
\end{align*}
where,
if $k \ge 1$,
$z_p$ is sampled from $\Gammadist(d/k)$ conditioned on $z_p > \GammaCDF^{-1}(1 - 2 p; d/k)$ (because $z^{\f{k-1}k}$ is increasing in $z$), but if $k \in (0, 1)$, then $z_p$ is sampled from $\Gammadist(d/k)$ conditioned on $z_p < \GammaCDF^{-1}(2 p; d/k)$ (because $z^{\f{k-1}k}$ is decreasing in $z$).

Thus, for any $u \in \Vrt(\Bb)$,
\begin{align*}
    \Phi(p) = \EV \mygamma{u}{p} =
    \begin{cases}
    \f k 2 \EV z^{\f{k-1}k} \ind(z > \GammaCDF^{-1}(1-2p)) & \text{if $k \ge 1$}\\
    \f k 2 \EV z^{\f{k-1}k} \ind(z < \GammaCDF^{-1}(2p)) & \text{if $k < 1$}
    \end{cases}
\end{align*}
where $z \sim \Gammadist(d/k)$.
This integral simplifies to $R^{-1} \Psi(p)$ (with $\Psi$ taking different forms depending on $k$) by \cref{lemma:GammaExpect}.

Then, by setting $G(x)$ in \cref{thm:differentialMethodMasterTheorem} to be $1 - H(x)_y = 1 - \rho$, we get the provably robust radius of
\begin{align*}
    \int_{1-\rho}^{1/2} \f 1 {\Phi(p)} \dd p =
    \int_{1-\rho}^{1/2} \f R {\Psi(p)} \dd p
    .
\end{align*}

\end{proof}

\subsubsection{\texorpdfstring{$\ell_\infty$ Adversary}{Linf Adversary}}

We first start with the Laplace distribution to highlight the basic logic behind the $\ell_\infty$ radius derivation.

\begin{thm}\label{thm:laplacelinf}
Suppose $H$ is a smoothed classifier smoothed by the Laplace distribution
\[q(x) \propto \exp(-\|x\|_1/\lambda),\]
such that $H(x) = (H(x)_1, \ldots, H(x)_C)$ is a vector of probabilities that $H$ assigns to each class $1, \ldots, C$.
If $H$ correctly predicts the class $y$ on input $x$, and the probability of the correct class is $\rho \defeq H(x)_y > 1/2$, then $H$ continues to predict the correct class when $x$ is perturbed by any $\eta$ with
\begin{align*}
    \|\eta\|_\infty < \lambda \int_{1-\rho}^{1/2}
        \f 1 {\Phi(p)} \dd p,
\end{align*}
where
\begin{align*}
    \Phi(p) &= c (p - \phi_d(c))
            + d \phi_{d-1}\left(c - \f 1 2\right)
            - d \phi_{d}(c),\\
        &\pushright{\text{in which $c = \phi_d^{-1}(p)$, and}}\\
    \phi_d(c)
        &\defeq
            2^{-d} \sum_{i=\f{c+d}2 + 1}^d \binom{d}{i}\\
        &=
            1 - \BinomCDF\lp \f{c+d}2; d\rp
            \\
        &\pushright{\text{for any $c \equiv d \mod 2$}}
            \\
    \phi_d^{-1}(p)
        &\defeq
            \inf \{c: \phi_d(c) \le p\}\\
        &=
            2 \BinomCDF^{-1}(1 - p) - d
            .
\end{align*}
\end{thm}

Note that when $d$ is large,
\[\Phi(p) \approx \GaussianCDF'(\GaussianCDF^{-1}(p)) \sqrt d,\]
so that the the bound above is roughly
\[\|\delta\|_\infty < \lambda \GaussianCDF^{-1}(\rho) /\sqrt d.
\]

\begin{proof}
By linearity in $\lambda$, it suffices to show this for $\lambda = 1$.

We seek to apply \cref{thm:differentialMethodMasterTheorem} to $G(x) = 1 - H(x)_y$, for which we need to derive random variables $\gamma_u$ and $\mygamma{u}{p}$, and most importantly, the function $\Phi$.

WLOG, let $u \in \Vrt(\Bb)$ be $u = (1, \ldots, 1)$; arguments for other $u \in\Vrt(\Bb)$ proceeds similarly.
For this $u$, $\gamma_u = \la u, -\nabla \log q(\delta) \ra, \delta \sim q,$ is a sum of independent Rademacher random variable $\gamma_u = \sum_{i=1}^d R_i$, where each $R_i$ independently takes values 1 and $-1$ with equal probability.
Thus $\gamma_u$ is distributed like $2 B_d - d$, where $B_d$ is the binomial random variable corresponding to the number of heads in $d$ coin tosses.
Then for any integer $c$ with the same parity as $d$, $\phi_d(c)$ is the complementary CDF of $\gamma_u$ and $\phi_d^{-1}$ is the corresponding inverse CDF.
Then, for $p \in [0, 1/2]$, we have
\begin{align*}
    \mygamma{u}{p} =
    \begin{cases}
    \gamma_u|_{(c, \infty)}
        &   \text{with probability $\phi_d(c)$}\\
    c
        &   \text{with probability $p - \phi_d(c)$}\\
    0   &   \text{with probability $1 - p$,}
    \end{cases}
\end{align*}
where $c \defeq \phi_d^{-1}(p)$ and $\gamma_u|_{(c, \infty)}$ is the random variable $\gamma_u$ conditioned on $\gamma_u > c$.
Therefore,
\begin{align*}
    \Phi(p) &= \EV \mygamma{u}{p} = c (p - \phi_d(c)) + 2^{-d} \sum_{i=\f{c+d}2 + 1}^d (2i - d) \binom{d}{i}\\
    &=
        c (p - \phi_d(c)) + 2^{-d} \sum_{i=\f{c+d}2 + 1}^d 2d\binom{d-1}{i-1} - d \binom d i\\
    &=
        c (p - \phi_d(c)) + d \phi_{d-1}\left(c - \f 1 2\right) - d \phi_{d}(c)
\end{align*}

Then, by setting $G(x)$ in \cref{thm:differentialMethodMasterTheorem} to be $1 - H(x)_y = 1 - \rho$, we get the desired robust radius.
\end{proof}

\begin{thm}\label{thm:exp1kLinf}
    Suppose $H$ is a smoothed classifier smoothed by
    \[q(x) \propto \exp(-\|x/\lambda\|_1^k), k > 1,\]
    such that $H(x) = (H(x)_1, \ldots, H(x)_C)$ is a vector of probabilities that $H$ assigns to each class $1, \ldots, C$.
    If $H$ correctly predicts the class $y$ on input $x$, and the probability of the correct class is $\rho \defeq H(x)_y > 1/2$, then $H$ continues to predict the correct class when $x$ is perturbed by any $\eta$ with
    \begin{align*}
        \|\eta\|_\infty < \lambda \int_{1-\rho}^{1/2}
            \f 1 {\Phi(p)} \dd p,
    \end{align*}
    where
    \begin{align*}
        \Phi(p)
            &=
                \EV \gamma \ind(\gamma > \varphi^{-1}(p)), \quad \text{with}\\
        \gamma
            &=
                \lp\sum_{i=1}^d \zeta_i\rp k\xi^{\f{k-1}k}, \\
            &\pushright{\zeta_i \sim \Rademacher, \xi \sim \Gammadist(d/k)}\\
        \varphi(c)
            &=
                \Pr[\gamma > c].
    \end{align*}
\end{thm}

\begin{proof}
    By linearity in $\lambda$, it suffices to show this for $\lambda = 1$.

    WLOG, let $u \in \Vrt(\Bb)$ be $u = (1, \ldots, 1)$; arguments for other $u \in\Vrt(\Bb)$ proceeds similarly.
    As in the proof of \cref{thm:laplacelinf}, we find $\gamma_u = \la u, -\nabla \log q(\delta) \ra = \la u, k\|\delta\|_1^{k-1} \sgn(\delta) \ra, \delta \sim q,$ is distributed like $\gamma$ in the theorem statement --- a product of sum of Rademacher variables (coming from $\la u, \sgn(\delta))$) and $k \xi^{\f{k-1}k}, \xi \sim \Gamma(d/k)$ (coming from $k \|\delta\|_1^{k-1}$).
    Because $k>1$, $\gamma$'s distribution is absolutely continuous (as it's a mixture of scaled versions of $\xi^{\f{k-1}k}$'s distribution, which is absolutely continuous).
    Therefore, the random variable $\mygamma{u}{p} = \gamma \ind(\gamma > \varphi^{-1}(p))$.
    Then the theorem statement follows straightforwardly from \cref{thm:differentialMethodMasterTheorem}.
\end{proof}

\subsection{Pareto Distribution}

\newcommand{\Pareto}{\mathrm{Pareto}}
\newcommand{\ParetoCDF}{\mathrm{ParetoCDF}}
For $a, \lambda > 0$ and $u \in \R$, define the 0-centered, symmetric Pareto distribution by its PDF
\begin{align*}
    \Pareto(x; a, \lambda)
        &=
            \f a {2\lambda} \lp 1 + \left| \f{x}{\lambda} \right|\rp^{-1-a}
            \\
        &=
            \f a {2\lambda} \exp\left[-(1+a) \log\lp 1 + \left| \f{x}{\lambda} \right|\rp\right]
            .
\end{align*}
Its CDF is given by
\begin{align*}
    \ParetoCDF(x; a, \lambda)
        &=
            \begin{cases}
            1 - \f 1 2 \lp 1 + \left| \f{x}{\lambda} \right|\rp^{-a}
                &   \text{if $x > 0$}\\
            \f 1 2 \lp 1 + \left| \f{x}{\lambda} \right|\rp^{-a}
                &   \text{else.}
            \end{cases}
\end{align*}

Consider smoothing distributions of the form
\begin{align*}
    q(x) &= \prod_{i=1}^d \Pareto(x_i; a, 1),
    \quad\text{so that}\quad\\
    \psi(x) &= (1+a) \sum_{i=1}^d \log(1 + |x_i|)\\
    \nabla \psi(x) &= \left\{(1+a)\f{\sgn(x_i)}{1 + \left| {x_i} \right|}\right\}_{i=1}^d
    ,
\end{align*}
with $\nabla \psi(x)$ defined when all coordinates $x_i$s are nonzero.

\subsubsection{\texorpdfstring{$\ell_1$ Adversary}{L1 Adversary}}
\begin{thm}\label{thm:paretoL1}
Suppose $H$ is a smoothed classifier smoothed by
\[q(x) \propto \prod_{i=1}^d \Pareto(x_i; a, \lambda),\]
such that $H(x) = (H(x)_1, \ldots, H(x)_C)$ is a vector of probabilities that $H$ assigns to each class $1, \ldots, C$.
If $H$ correctly predicts the class $y$ on input $x$, and the probability of the correct class is $\rho \defeq H(x)_y > 1/2$, then $H$ continues to predict the correct class when $x$ is perturbed by any $\eta$ with
\begin{align*}
    \|\eta\|_1 <
    \lambda\frac{2 \rho - 1}{a}
    \, _2F_1\left(1,\frac{a}{a+1};\frac{a}{a+1}+1;(2 \rho - 1)^{1+\frac{1}{a}}\right),
\end{align*}
where $\, _2F_1$ is the hypergeometric function.
\end{thm}

\begin{proof}
By linearity in $\lambda$, it suffices to show this for $\lambda = 1$.

We seek to apply \cref{thm:differentialMethodMasterTheorem} to $G(x) = 1 - H(x)_y$, for which we need to derive random variables $\gamma_u$ and $\mygamma{u}{p}$, and most importantly, the function $\Phi$.

WLOG, assume $u \in \Vrt(\Bb)$ is $e_1$.
Then $\gamma_u = \la u, \nabla \psi(\delta) \ra = (1+a)\f{\sgn(\delta_1)}{1 + \left| {\delta_1} \right|}, \delta \sim q,$ is distributed as $(1+a)\f{\sgn(z)}{1 + |z|}$ where $z \sim \Pareto(a, 1)$.
Therefore, for $p \in [0, 1/2]$, the random variable $\mygamma{u}{p}$ defined in \cref{defn:Phi} is
\begin{align*}
    \mygamma{u}{p} = \begin{cases}
                \f{1+a}{1+z_p} & \text{with prob.\ $p$}\\
                0 & \text{with prob.\ $1 - p$,}
            \end{cases}
\end{align*}
where $z_p$ is sampled from $\Pareto(a, 1)$ conditioned on the interval $[0, \ParetoCDF^{-1}(p + 1/2; a, 1)]$.

Thus, for any $u \in \Vrt(\Bb)$,
\begin{align*}
    \Phi(p) = \EV \mygamma{u}{p} = \EV \f{1+a}{1+z} \ind(z \in [0, c]),
\end{align*}
where $z \sim \Pareto(a, 1)$ and $c = \ParetoCDF^{-1}(p + 1/2; a, 1)$.
This can be simplified as follows:
\begin{align*}
\Phi(p)
    &=
        \int_0^c \Pareto(r; a, 1)
            \f{1+a}{1 + r}
                \dd r
            \\
    &=
        -\int_0^c \Pareto'(r; a, 1)
                \dd r
            \\
    &=
        \Pareto(0; a, 1) - \Pareto(c; a, 1)
            \\
    &=
        \f a {2} \lp 1 - \lp 1 + c \rp^{-1-a} \rp.
\end{align*}
Note that for $p \in [1/2, 1]$,
\begin{align*}
    \ParetoCDF^{-1}(p + 1/2; a, 1)
    =
        (1 - 2p)^{-1/a}-1,
\end{align*}
so $\Phi(p)$ can be further simplified:
\begin{align*}
    \Phi(p)
        &=
            \f a {2} \lp 1 - (1-2p)^{\f{a+1}a} \rp.
\end{align*}

Then, by setting $G(x)$ in \cref{thm:differentialMethodMasterTheorem} to be $1 - H(x)_y = 1 - \rho$, we get the provably robust radius of
\begin{align*}
    &\phantomeq
        \int_{1-\rho}^{1/2} \f 1 {\Phi(p)} \dd p
    =
        \int_{1-\rho}^{1/2}  \f {2} a \lp 1 - (1-2p)^{\f{a+1}a}\rp^{-1} \dd p
        \\
    &=
    \frac{2 p-1}{a}
    \, _2F_1\left(1,\frac{a}{a+1};\frac{a}{a+1}+1;(1-2 p)^{1+\frac{1}{a}}\right)
    \bigg|_{1-\rho}^{1/2}
        \\
    &=
    \frac{2 \rho - 1}{a}
    \, _2F_1\left(1,\frac{a}{a+1};\frac{a}{a+1}+1;(2 \rho - 1)^{1+\frac{1}{a}}\right)
        .
\end{align*}
\end{proof}

\subsection{\texorpdfstring{$\ell_2$}{L2}-Norm Based Exponential Law}

In this section we consider
\begin{align*}
    q(x) &\propto \exp(-\|x\|_2)\quad\text{so that}\\
    \psi(x) &= \|x\|_2
    \quad\text{and}\quad
    \nabla \psi(x) = x/\|x\|_2,
\end{align*}
defined as long as $x \ne 0$.

\subsubsection{\texorpdfstring{$\ell_2$ Adversary}{L2 Adversary}}

\begin{thm}\label{eqn:exp2l2}
Suppose $H$ is a smoothed classifier smoothed by
\[q(x) \propto \exp(-\|x\|_2/\lambda),\]
such that $H(x) = (H(x)_1, \ldots, H(x)_C)$ is a vector of probabilities that $H$ assigns to each class $1, \ldots, C$.
If $H$ correctly predicts the class $y$ on input $x$, and the probability of the correct class is $\rho \defeq H(x)_y > 1/2$, then $H$ continues to predict the correct class when $x$ is perturbed by any $\eta$ with
\begin{align*}
    \|\eta\|_2
        &<
            \lambda (d-1) \arctanh\bigg(\\
        &\phantomeq\quad
            1 - 2\BetaCDF^{-1}\lp
                1 - \rho
                ;
                \f{d-1}2, \f{d-1}2
            \rp
        \bigg).
\end{align*}
\end{thm}

\newcommand{\PDF}{\mathrm{PDF}}
\begin{proof}

By linearity in $\lambda$, it suffices to show this for $\lambda = 1$.

We seek to apply \cref{thm:differentialMethodMasterTheorem} to $G(x) = 1 - H(x)_y$, for which we need to derive random variables $\gamma_u$ and $\mygamma{u}{p}$, and most importantly, the function $\Phi$.

For any $u \in \Vrt(\Bb)$ (i.e.\ any unit vector $u$), $\gamma_u = \la u, \nabla \psi(\delta) \ra = \la u, \f\delta{\|\delta\|_2} \ra, \delta \sim q,$ is distributed like $2 \BetaVar\left(\f{d-1}2, \f{d-1}2\right) - 1$ by \cref{lemma:sphericalCDF}.
Its complementary CDF is given by
\begin{align*}
    \Pr[\gamma_u > c]
        &=
            R \int_{c}^1 (1 - t^2)^{\f{d-3}{2}} \dd t\\
        &=
            \BetaCDF\lp \f{1-c}2; \f{d-1}2, \f{d-1}2\rp
        \defeq
            \varphi(c),
\end{align*}
where $R \defeq \frac{\Gamma \left(\frac{d}{2}\right)}{\sqrt{\pi } \Gamma \left(\frac{d-1}{2}\right)}$.
Therefore, for $p \in [0, 1]$, the random varible $\mygamma{u}{p}$ defined in \cref{defn:Phi} is given by
\begin{align*}
    \mygamma{u}{p} =
    \begin{cases}
    \gamma_u|_{(\inv \varphi(p), \infty)}
        &   \text{with probability $p$}\\
    0   &   \text{with probability $1 - p$,}
    \end{cases}
\end{align*}
Thus, for any $u \in \Vrt(\Bb)$,
\begin{align*}
    \Phi(p) &= \EV \mygamma{u}{p}
    = R \int_{\inv \varphi(p)}^1 t (1 - t^2)^{\f{d-3}2} \dd t\\
    &= \frac{R}{d-1}
        \left(1-\inv \varphi(p)^2\right)^{\frac{d-1}{2}}
        .
\end{align*}
Then, by setting $G(x)$ in \cref{thm:differentialMethodMasterTheorem} to be $1 - H(x)_y = 1 - \rho$, we get the provably robust radius of
\begin{align*}
    &\phantomeq
        \int_{1-\rho}^{1/2} \f 1 {\Phi(p)} \dd p
        \\
    &=
        \f{d-1}R
        \int_{1-\rho}^{1/2}
        \left(1-\inv \varphi(p)^2\right)^{-\frac{d-1}{2}}
        \dd p
        \\
    &=
        (d-1)
        \int^{\inv \varphi(1-\rho)}_{0}
        (1 - c^2)^{-\frac{d-1}{2}}(1 - c^2)^{\f{d-3}2}
            \dd c
        \\
    &=
        (d-1)
        \int^{\inv \varphi(1-\rho)}_{0}
        (1 - c^2)^{-1}
            \dd c
        \\
    &=
        (d-1) \arctanh(\inv \varphi (1- \rho))
        .
\end{align*}
Unpacking the definition of $\varphi$ yields the result.
\end{proof}

\subsection{Uniform Distribution over a Sphere}

\subsubsection{\texorpdfstring{$\ell_2$ Adversary}{L2 Adversary}}
\begin{figure}
    \centering
    \includegraphics[width=0.4\textwidth]{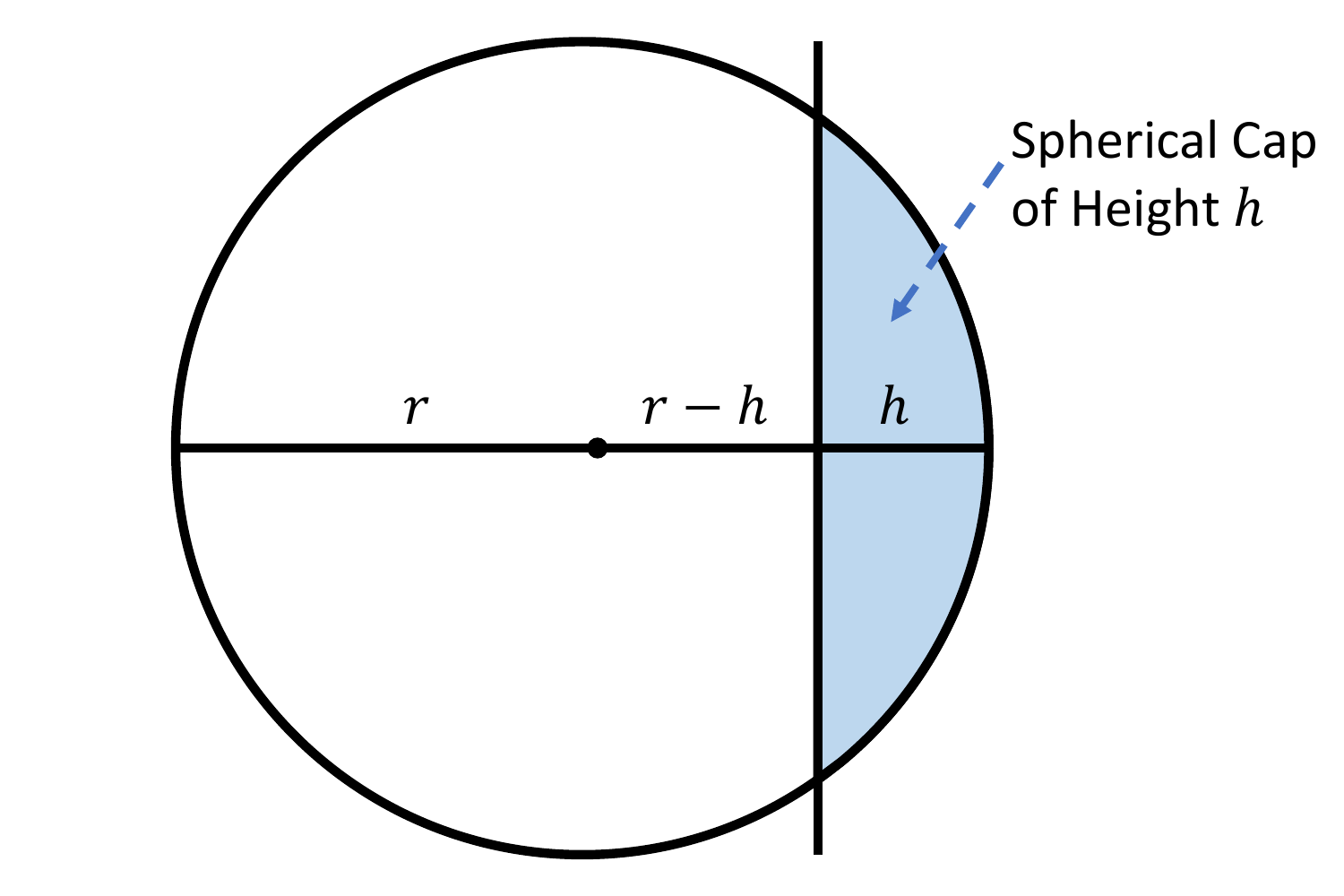}
    \caption{Spherical Cap}
    \label{fig:sphericalcap}
\end{figure}
Consider the distribution that is uniform on the $\ell_2$ unit ball $\{x: \|x\|_2 \le 1\}$.
The \emph{spherical cap} of height $h \le 1$ in this unit ball is the portion of the ball that is cut away by a hyperplane of distance $1-h$ from the origin; see \cref{fig:sphericalcap}.
By \cref{lemma:ballCDF}, this spherical cap has volume
\begin{align*}
    V_{d}^h \defeq V_d \BetaCDF\lp \f{h}2; \f{d+1}2, \f{d+1}2\rp.
\end{align*}
where $V_d$ is the volume of the unit sphere in $\R^d$.

Two unit radius spheres with centers $\epsilon$ apart intersects in a region that is the union of two spherical caps of height $1-\epsilon/2$.
This intersection thus has volume $2 V^{1-\epsilon/2}_{d}$, and the volume of one of the spheres outside this intersection is $V_d - 2 V^{1-\epsilon/2}_{d} = V_d\lp 1 - 2 \BetaCDF\lp \f{1-\epsilon/2}2; \f{d+1}2, \f{d+1}2\rp\rp.$

\begin{thm}\label{thm:BallUniformAdvL2}
Suppose $H$ is a smoothed classifier smoothed by the uniform distribution $q$ over a ball of radius $\lambda$ centered at the origin, such that $H(x) = (H(x)_1, \ldots, H(x)_C)$ is a vector of probabilities that $H$ assigns to each class $1, \ldots, C$.
If $H$ correctly predicts the class $y$ on input $x$, and the probability of the correct class is $\rho \defeq H(x)_y > 1/2$, then $H$ continues to predict the correct class when $x$ is perturbed by any $\eta$ with
\begin{align*}
    \|\eta\|_2 < \lambda\lp 2 - 4 \BetaCDF^{-1}\lp \f 3 4 - \f \rho 2; \f{d+1}2, \f{d+1}2\rp\rp
                .
\end{align*}
\end{thm}

\begin{proof}
By linearity in $\lambda$, it suffices to show this for $\lambda = 1$.

By assumption, there is a region of probability $\rho$ under the uniform distribution $q(x + \cdot)$ centered at $x$ that the base classifier classifies as $y$.
The intersection between the support of $q(x + \cdot)$ and $q(x + \delta + \cdot)$ for any $\|\delta\|_2 \le \epsilon$ contains a region of probability at least
\[\rho - \lp 1 - 2 \BetaCDF\lp \f{1-\epsilon/2}2; \f{d+1}2, \f{d+1}2\rp \rp\]
that the base classifier classifies as $y$.
For this probability to be at least $1/2$, we require
\begin{align*}
    \f 1 2
        &\le
            \rho - (1 - 2 \BetaCDF\lp \f{1-\epsilon/2}2; \f{d+1}2, \f{d+1}2\rp)
            \\
    \f 3 4 - \f\rho 2
        &\le
            \BetaCDF\lp \f{1-\epsilon/2}2; \f{d+1}2, \f{d+1}2\rp
            \\
    1-\epsilon/2
        &\ge
            2\BetaCDF^{-1}\lp \f 3 4 - \f \rho 2; \f{d+1}2, \f{d+1}2\rp
            \\
    \epsilon
        &\le
            2 - 4 \BetaCDF^{-1}\lp \f 3 4 - \f \rho 2; \f{d+1}2, \f{d+1}2\rp
            ,
\end{align*}
as desired.
\end{proof}

\subsection{General \texorpdfstring{$\ell_2$}{L2}-Norm Based Distributions via the Level Set Method}
\label{sec:levelsetl2}

\subsubsection{\texorpdfstring{$\ell_2$ Adversary}{L2 Adversary}}

\newcommand{\qrad}{\bar{q}}
Define $W_d(r, s, \epsilon)$ to be the probability a point sampled from the \emph{surface} of a ball of radius $r$ centered at the origin is outside a ball of radius $s$ with center $\epsilon$ away from the origin.
By \cref{lemma:sphericalCDF}, we have
\begin{align*}
    W_d(r, s, \epsilon) = \BetaCDF\left(\f{(r+\epsilon)^2 - s^2}{4 \epsilon r}; \f{d-1}2, \f{d-1}2\right).
    \numberthis\label{eqn:W}
\end{align*}
Note that $W_d$ can be evaluated quickly using standard \texttt{scipy} functions.

\begin{thm}\label{thm:levelsetl2}
Suppose that the density of a distribution satisfies $q(x) = \qrad(\|x\|_2)$ for some differentiable, decreasing function $\qrad: \R^{\ge 0} \to \R^{\ge 0}$.
Then for any $\kappa > 0$ and any $v \in \R^d$, the growth function satisfies
\begin{align*}
    \growth_q(p_0, v) = p_1
\end{align*}
where
\begin{align*}
p_0 &= 1 - \EV_{r} W_d(r, \bar q^{-1}(\bar q(r)/\kappa), \|v\|_2)\\
p_1 &= \EV_{r} W_d(r, \bar q^{-1}(\bar q(r)\kappa), \|v\|_2),
\\
r &\sim \text{distribution with density $\propto r^{d-1} \qrad(r)$.}
\end{align*}

\end{thm}
For most $\qrad$, $p_0$ and $p_1$ can be evaluated numerically and quickly for each $\kappa$ and $\|v\|_2$ using 1-dimensional integrals.

\begin{proof}
Let $r_t \defeq \qrad^{-1}(t)$.
Then the superlevel set $U_t = \{x: q(x) \ge t\}$ is a ball with radius $r_t$.
Furthermore,
\begin{align*}
    \nabla q(x) &= \qrad'(\|x\|_2) \f{x}{\|x\|_2}\\
    \|\nabla q(x)\|_2^{-1}
        &=
            -\qrad'(\|x\|_2)^{-1}
        =
            -r'_{\qrad(\|x\|_2)}.
\end{align*}

Let $SA_d$ be the surface area of the unit sphere in $\R^d$.
Then
\begin{align*}
    q(\npset_\kappa)
    =
        -SA_d \int_0^\infty r'_t t r_t^{d-1}  (1 - W_d(r_t, r_{t/\kappa}, \|v\|_2)) \dd t\\
    q(\npset_\kappa - v)
    =
        -SA_d \int_0^\infty r'_t t r_t^{d-1}  W_d(r_t, r_{t\kappa}, \|v\|_2) \dd t.
\end{align*}
If we change coordinates from $t$ to $r$, then
\begin{align*}
    &q(\npset_\kappa)=SA_d \times\\
    &\quad\quad
        \int_0^\infty \bar q(r) r^{d-1}  (1 - W_d(r, \bar q^{-1}(\bar q(r)/\kappa), \|v\|_2)) \dd r\\
    &q(\npset_\kappa - v)=SA_d \times \\
    &\quad\quad
        \int_0^\infty \bar q(r) r^{d-1}  W_d(r, \bar q^{-1}(\bar q(r)\kappa), \|v\|_2) \dd r.
\end{align*}
Since $\bar q(r) r^{d-1}$ is proportional to the density of $\|x\|_2, x \sim q$, we can also write this as
\begin{align*}
    &q(\npset_\kappa) = \EV_{r = \|x\|_2, x \sim q} (1 - W_d(r, \bar q^{-1}(\bar q(r)/\kappa), \|v\|_2))\\
    &q(\npset_\kappa - v) = \EV_{r = \|x\|_2, x \sim q} W_d(r, \bar q^{-1}(\bar q(r)\kappa), \|v\|_2)
    .
    \numberthis\label{eqn:NPspherical}
\end{align*}
The distribution of $r$ here has density $\propto r^{d-1}\qrad(r)$.
Then setting $p_0 = q(\npset_\kappa), p_1 = q(\npset_\kappa-v)$ yields the desired result by \cref{eqn:NP}.
\end{proof}

\begin{exmp}\label{sec:l2exponential}
If $q(x) \propto \|x\|_2^{-j} \exp(-\|x\|_2^k)$, then $\bar q(r) \propto r^{-j} \exp(-r^k)$, and the radius is distributed as $\sqrt[k]{\Gammadist(d/k - j/k)}$ by \cref{lemma:expNormSampling}.
A table of robust radii can then be built according to \cref{alg:levelsettable}, and certification can be done via \cref{alg:levelsetcert}.
\end{exmp}

\begin{exmp}
\label{sec:l2powerlaw}
If $q(x) \propto (1 + \|x\|_2^k)^{-a}$, then $\bar q(r) \propto (1 + r^k)^{-a}$, and the radius is distributed as $\sqrt[k]{\BetaPrime(d/k, a - d/k)}$.
A table of robust radii can then be built according to \cref{alg:levelsettable}, and certification can be done via \cref{alg:levelsetcert}.
\end{exmp}

\subsection{Basic Facts about Probability Distributions}

\begin{lemma}\label{lemma:sphericalCDF}
If $(x_1, \ldots, x_d)$ is sampled uniformly from the unit sphere $S^{d-1} \sbe \R^d$, then
\[\f {1 + x_1}2 \text{ is distributed as $\BetaVar\left(\f{d-1}2, \f{d-1}2\right)$},\]
i.e.
\begin{align*}
    \Pr[x_1 \ge c]
    &=
        \BetaCDF\lp \f{1-c}2; \f{d-1}2, \f{d-1}2\rp
        \\
    &=
        1 - \BetaCDF\lp \f{1+c}2; \f{d-1}2, \f{d-1}2\rp
        .
\end{align*}
\end{lemma}
\begin{proof}
From simple geometric reasoning, we get
\begin{align*}
    \Pr[x_1 \ge c]
        &\propto \int_c^1 (1 - t^2)^{\f{d-3}2} \dd t\\
        &= \int_c^1 (1 - t)^{\f{d-3}2} (1 + t)^{\f{d-3}2} \dd t\\
        &= \int_0^{\f{1-c}2}  \lp 2 x \rp^{\f{d-3}2} (2(1 - x))^{\f{d-3}2} \dd x.
\end{align*}
\end{proof}

\begin{lemma}\label{lemma:ballCDF}
If $(x_1, \ldots, x_d)$ is sampled uniformly from the ball $\{y: \|y\|_2 \le 1\} \sbe \R^d$, then
\[\f {1 + x_1}2 \text{ is distributed as $\BetaVar\left(\f{d+1}2, \f{d+1}2\right)$},\]
i.e.
\begin{align*}
    \Pr[x_1 \ge c]
    &=
        \BetaCDF\lp \f{1-c}2; \f{d+1}2, \f{d+1}2\rp
        \\
    &=
        1 - \BetaCDF\lp \f{1+c}2; \f{d+1}2, \f{d+1}2\rp
        .
\end{align*}
\end{lemma}
\begin{proof}
Similar to \cref{lemma:sphericalCDF}.
\end{proof}

\begin{lemma}\label{lemma:expNormSampling}
For any norm $\| \cdot \|$ on $\R^d$, the distribution
\[q(x) \propto \|x\|^{-j}\exp(-\|x\|^k),\]
with $j < d$, can be sampled as follows:
\begin{enumerate}
    \item Sample the radius $r \sim \sqrt[k]{\Gammadist(\f d k - \f j k)}$
    \item Sample a point $v$ from the unit sphere of $\| \cdot \|$
    \item return $rv$
\end{enumerate}
\end{lemma}

\begin{lemma}\label{lemma:GammaExpect}
For any $c, s \ge 0$ and $r > 0$,
\begin{align*}
    &\phantomeq
        \EV_{z\sim \Gammadist(r)} z^s \ind(z > c)
        \\
    &=
        \f{\Gamma(r+s)}{\Gamma(r)}
        (1-\GammaCDF(c; r+s))
        .
\end{align*}
\end{lemma}

\section{Proof of \texorpdfstring{\cref{thm:main-lb}}{Our No-Go Theorem}}
In this section we prove our main impossibility result, \cref{thm:main-lb}.
We will assume throughout this proof that the reader is familiar with standard notions in functional analysis.
Our proof will proceed in two steps.

For all $p \in (0,2]$ and $d' \geq 1$, we let $\ell_p^{d'}$ denote $\R^{d'}$ equipped with the $p$-quasinorm.
First, we show that if there exists a useful smoothing scheme, this implies a low embedding distortion of our normed space into $\ell_{0.99}^{d'}$, for some $d'$.

Formally, let $(X, d_X)$ and $(Y, d_Y)$ be two metric spaces.
We say an embedding $f: X \to Y$ has distortion $D$ if there exist positive constants $\alpha < 1 < \beta$ so that
\[
\alpha d_Y(f(x_1), f(x_2)) \leq d_X (x_1, x_2) \leq \beta d_Y (f(x_1), f(x_2))
\]
for all $x_1, x_2 \in X$, where $\beta / \alpha \leq D$.
We will first show:
\begin{lemma}
\label{thm:useful-to-embedding}
Suppose there exists an $(\eps, s, \ell)$-useful smoothing scheme for $\| \cdot \|$, and $s / \ell \leq 1/162$.
Then, there exists $d'$ and a linear embedding from $(\R^d, \| \cdot \|)$ into $\ell_{0.99}^{d'}$ with distortion at most
\[
O \left( \left( \frac{s}{\ell} \right)^{1/4} \cdot \frac{1}{\eps} \right) \; .
\]
\end{lemma}

Next we show that any linear embedding into $\ell_{0.99}^{d'}$ will suffer distortion which is at least $C_2 ((\R^d, \| \cdot \|))$:
\begin{lemma}
\label{thm:distortion-lp}
	Any linear embedding from $(\R^d, \| \cdot \|)$ into $\ell_{0.99}^{d
	'}$ must have distortion $\Omega(C_2 ((\R^d, \| \cdot \|)))$, where $\Omega$ hides constant independent of $d$ and $d'$.
\end{lemma}
This result is essentially folklore in the metric embedding community, but we include a proof for completeness.

These two lemmas together immediately imply \cref{thm:main-lb}.
The rest of this section is dedicated to proofs of these two lemmas.
To do so, it will first be useful to establish some regularity conditions on a variant of the growth function considered previously in this paper.

\subsection{The Pairwise Growth Function}
For any two two probability densities $q_1, q_2$, over $\R^d$, define the \emph{pairwise growth function} between $q_1$ and $q_2$, denoted $\growth_{q_1, q_2}$, to be
\[
\growth_{q_1, q_2} (p) = \sup_{U: q_1 (U) = p} q_2 (U) \; .
\]
We will assume throughout this proof that $q_2$ is absolutely continuous with respect to $q_1$.
The more general case can be easily handled by the theory of Radon-Nikodym derivatives and Lebesgue's decomposition theorem.
This growth function satisfies the following, basic properties, whose proofs are easy and are omitted.
\begin{fact}
\label{fact:growth-facts}
Let $q_1, q_2, \growth_{q_1, q_2}$ be above.
Then:
\begin{itemize}
\item $\growth_{q_1, q_2} (p)$ is monotonically increasing.
\item $\growth_{q_1, q_2} (1) = 1$, and $\growth_{q_1, q_2} (0) \geq 0$.
\item $\dtv (q_1, q_2) = \sup_{p \in [0, 1]} \left( \growth_{q_1, q_2} (p) - p \right)$.
\end{itemize}
\end{fact}

Then, we have:
\begin{lemma}
\label{lem:growth-concave}
For any $q_1, q_2 \in \Delta_d$, the function $\growth_{q_1, q_2}$ is concave.
\end{lemma}
\begin{proof}
For clarity, since $q_1, q_2$ will be fixed throughout this proof, we will omit the subscripts in the definition of $\growth$.

For any $t > 0$, define the set
\[
S_t = \left\{ x \in \R^d: \frac{d q_2}{d q_1} (x) \geq t \right\} \; ,
\]
which we can think of as a generalized Neyman-Pearson set, for the two distributions $q_1, q_2$.

Then, by classical arguments, for every $p$, the set which obtains the supremum in the definition of the growth function for that value of $p$ is given by
\begin{align*}
    S_{K(p)} = \argmax_{U: q_1(U) = p} q_2(U),
\end{align*}
where $K(p)$ is defined so that $q_1(S_{K(p)}) = p$.
Therefore, for all $p$, we have that $\growth (p) = q_2 (S_{K(p)})$.

We will show that for all $p < p' < p''$, the growth function satisfies
\begin{equation}
\label{eq:growth-concave}
\frac{\growth(p') - \growth(p)}{p' - p} \geq  \frac{\growth(p'') - \growth(p')}{p'' - p'} \; ,
\end{equation}
which is equivalent to the claim.

Note that for any $0 \leq r \leq r'$, we have that $\growth(r') - \growth(r) = q_2 (\Delta_{r', r})$, where $\Delta_{r', r} = S_{K(r')} \setminus S_{K(r)}$.
However, observe that for $p < p' < p''$, we have that $\tfrac{d q_2}{d q_1} (x) \geq \tfrac{d q_2}{d q_1} (x')$ for all $x \in \Delta_{p', p}$ and $x' \in \Delta_{p'', p'}$.
But we also have
\begin{align*}
 \growth(p') - \growth(p) &\geq q_1 (\Delta_{p', p}) \cdot \min_{x \in \Delta_{p', p}} \frac{d q_2}{d q_1} (x) \\
&= (p' - p) \min_{x \in \Delta_{p', p}} \frac{d q_2}{d q_1} (x) \; ,
\end{align*}
and similarly
\[
\growth(p'') - \growth(p') \leq (p'' - p') \max_{x \in \Delta_{p'', p'}} \frac{d q_2}{d q_1} (x) \; ,
\]
which implies~\cref{eq:growth-concave}.
\end{proof}

\subsection{Proof of \texorpdfstring{\cref{thm:useful-to-embedding}}{Lemma J.1}}
Let $\cQ = \{q_x\}_{x \in \R^d}$ be an $(\eps, r, \ell)$-useful smoothing scheme for $\| \cdot \|$.
For simplicity, throughout this proof, we will assume that $q_x$ has a probability density function, denoted $Q_x$, for all $x$, that is, the distributions are absolutely continuous with respect to the Lebesgue measure.
It is not hard to generalize this proof to handle general probability distributions by using Lebesgue decomposition and taking the appropriate Radon-Nikodym derivatives.

First, we demonstrate that a useful smoothing scheme actually implies an embedding of the norm $\| \cdot \|$ into an infinite dimensional $L_1$ space, namely, the space of all distributions with distance given by total variation distance.
Recall the total variation distance between two distributions $q_1, q_2$, denoted $\dtv (q_1, q_2)$, is given by
\[
\dtv(q_1, q_2) = \sup_{U \subseteq \R^d} |q_1 (U) - q_2 (U)| = \frac{1}{2} \| Q_1 - Q_2 \|_1 \; .
\]
We denote the space of probability distributions over $\R^d$ by $\Delta_d$.
Note that the metric space $(\Delta_d, \dtv)$ is an infinite dimensional $L_1$ space.
Thus, classical results yield:
\begin{fact}[see e.g.~\citet{wojtaszczyk1996banach}]
For all $d \geq 1$, we have $C_2 ((\Delta_d, \dtv)) = \Theta (1)$.
\end{fact}

We now need another notion, introduced in~\citet{andoni2018sketching}.
\begin{defn}[\citet{andoni2018sketching}]
A map $f: X \to Y$ between two metric spaces $(X, d_X)$ and $(Y, d_Y)$ is an \emph{$(s_1, s_2, \tau_1, \tau_2)$-threshold map} if it satisfies:
\begin{itemize}
    \item If $d_X (x_1, x_2) \leq s_1,$ then $d_Y (f(x_1), f(x_2)) \leq \tau_1$.
    \item If $d_X (x_1, x_2) \geq s_2$, then $d_Y (f(x_1), f(x_2)) \geq \tau_2$.
\end{itemize}
\end{defn}
\noindent
Any smoothing scheme $\cQ = \{q_x\}_{x \in \R^d}$ can be viewed as a map $\R^d \to \Delta_d, x \mapsto q_x$ that takes a point in $\R^d$ and maps it to its associated distribution after smoothing.
Our main technical work will be to demonstrate the following lemma:
\begin{lemma}
\label{lem:robust-to-threshold}
Let $q$ be a $(\eps, s, \ell)$-useful smoothing distribution for $\| \cdot \|$.
Then $\cQ$ is a $(\eps, 1, 2 s, \ell)$-threshold map between $(\R^d, \| \cdot \|)$ and $(\Delta, \dtv)$.
\end{lemma}
\noindent
Our first observation is that \cref{lem:growth-concave} allows us to relate the usefulness of the smoothing scheme to total variation distance:
\begin{cor}
\label{cor:robust-to-tv}
For any $q_1, q_2$, we have that
\[
\growth_{q_1, q_2} (1/2) - 1/2 \geq \frac{1}{2} \dtv (q_1, q_2) \; .
\]
\end{cor}
\begin{proof}
As before, for conciseness we will drop the subscripts in the definition of $\growth$.
We will show that for all $p \in [0, 1]$, we have that
\[
\growth (1/2) - 1/2 \geq \frac{1}{2} \left( \growth(p) - p \right) \; ,
\]
which by~\cref{fact:growth-facts} implies the lemma.

First, consider the case where $p \geq 1/2$.
Then, by concavity of $\growth(p)$, we have that
\begin{align*}
\growth(1/2) &\geq \frac{1}{2p} \growth(p) + \frac{2p - 1}{2p} \growth(0) \\
&\geq \frac{1}{2p} \growth(p) \; ,
\end{align*}
since $\growth(0) \geq 0$ by~\cref{fact:growth-facts}.
Therefore, we have that
\[
\growth(1/2) - \frac{1}{2} \geq \frac{\growth(p) - p}{2p} \geq  \frac{1}{2} \left( \growth(p) - p \right) \; .
\]
The case where $p < 1/2$ follows symmetrically by considering the line segment between $p$ and $1$.
\end{proof}

From this, the proof of \cref{lem:robust-to-threshold} is simple.
\begin{proof}[Proof of~\cref{lem:robust-to-threshold}]
We first prove that it satisfies the first condition.
Let $x, y$ be so that $\| x - y\| \leq \eps$.
Then, the robustness condition implies that
\[
\growth_{q_2, q_1} (1/2) - 1/2 \leq r \; .
\]
By \cref{cor:robust-to-tv} this implies that $\dtv (q_x, q_y) \leq 2r$.

We now prove it satisfies the second condition.
But, the accuracy condition immediately implies that if $x, y$ satisfy $\| x- y \| \geq 1$, we must have $\dtv (q_x, q_y) \geq \ell$.
This proves the claim.
\end{proof}

With~\cref{lem:robust-to-threshold} in hand, we can now invoke a number of classical results from the theory of metric embeddings to obtain our desired result.
We first use the following fact, which follows since $L_1$ embeds isometrically into squared-$L_2$.
\begin{fact}[see e.g.~\citet{matouvsek2013lecture}]
	Suppose there exists an $(s_1, s_2, \tau_1, \tau_2)$-threshold map from $(\R^d, \| \cdot \|)$ to $(\Delta_d, \dtv)$.
	Then there exists an $(s_1, s_2, \sqrt{\tau_1}, \sqrt{\tau_2})$-threshold map from $(\R^d, \| \cdot \|)$ to a Hilbert space $H$.
\end{fact}
This implies:
\begin{cor}
\label{cor:useful-to-l2}
	Suppose there exists an $(\eps, s, \ell)$-useful smoothing distribution for $\| \cdot \|$.
Then there exists a $(\eps, 1, \sqrt{2 s}, \sqrt \ell)$-threshold map between $(\R^d, \| \cdot \|)$ and a Hilbert space $H$.
\end{cor}
We now require the following theorem, first proven in~\citet{andoni2018sketching}, which we reproduce below in a slightly simplified form:
\begin{thm}[Theorem 4.12 in~\citet{andoni2018sketching}]
\label{thm:andoni-4.12}
	Suppose there exists a $(\eps, 1, \tau_1, \tau_2)$-threshold map from $(\R^d, \| \cdot \|)$ to a Hilbert space, for $\tau_2 \ge 9 \tau_1$.
	Then there exists a map $h$ from $\R^d$ into a Hilbert space with induced norm $\| \cdot \|_H$ such that for every $x_1, x_2 \in \R^d$, we have:
	\begin{align*}
	\sqrt{\tau_2} \cdot \min (1, \eps \| x_1 - x_2 \| ) &\leq \| h(x_1) - h(x_2) \|_H \\
	&~~~~\leq 10 \cdot \sqrt{2 \tau_1 \| x_1 - x_2 \|} \; .
	\end{align*}
\end{thm}
Combining \cref{cor:useful-to-l2} and \cref{thm:andoni-4.12}, we obtain:
\begin{cor}
\label{cor:bilip-map}
Suppose there exists an $(\eps, s, \ell)$-useful smoothing distribution for $\| \cdot \|$, and suppose that $s / \ell \leq 1/162$.
Then, there exists a Hilbert space $H$ with induced norm $\| \cdot \|_H$ and a map $h: \R^d \to H$ so that for all $x, y \in \R^d$, we have
\begin{align*}
	\min (1, \eps \| x_1 - x_2 \| ) &\leq \| h(x_1) - h(x_2) \|_H \\
	&~~~~\leq 10 \cdot 2^{3/4} \cdot \left( \frac{s}{\ell} \right)^{1/4} \sqrt{\| x_1 - x_2 \|} \; .
\end{align*}
\end{cor}
Finally, we require the following theorem from~\citet{andoni2018sketching}, which we reproduce for completeness:
\begin{thm}[Theorem 5.1 in~\citet{andoni2018sketching}]
\label{thm:bilip-to-linear}
	Let $X$ be a finite-dimensional normed space with norm $\| \cdot \|$, and let $\Delta > 0$.
	Let $H$ be a Hilbert space with associated norm $\| \cdot \|_H$.
	Assume we have a map $f: X \to H$, such that, for some absolute constant $K > 0$, and for all $x, y \in X$, we have:
	\begin{itemize}
		\item $\| f(x_1) - f(x_2) \|_H \leq K \cdot \sqrt{\| x_1 - x_2 \|}$, and
		\item if $\| x - y \| \geq \Delta$, then $\| f(x_1) - f(x_2) \|_H \geq 1$.
	\end{itemize}
	Then, for any $\xi \in (0, 1/3)$, the space $X$ linearly embeds into $\ell_{1 - \xi}^{d'}$ with distortion $O(\Delta / \xi)$, for some finite $d'$.
\end{thm}
Combining \cref{cor:bilip-map} and \cref{thm:bilip-to-linear} immediately yields
\cref{thm:useful-to-embedding}.

\subsection{Proof of \texorpdfstring{\cref{thm:distortion-lp}}{Lemma J.2}}
We now turn to the proof of \cref{thm:distortion-lp}.
The only reason why this is slightly non-standard is that $\ell_p^{d'}$ for $p < 1$ are not norms, as they do not satisfy the triangle inequality.
Despite this, we show that the standard results that the cotype constant is a lower bound on distortion of any linear embedding still holds in these spaces.

\begin{fact}[Khintchine's Inequality]\label{fact:khintchine}
  For any $p \in (0, \infty)$ there exist constants $A_p, B_p$ such
  that for any $x_1, \ldots, x_n \in \R$
  \[
    A_p\sqrt{\sum_{i = 1}^n x_i^2} \le
    \left(\mathbb{E} \left| \sum_{i = 1}^n \sigma_i x_i \right|^p \right)^{1/p}
    \le B_p \sqrt{\sum_{i = 1}^n x_i^2}.
  \]
  Here $\sigma_1, \ldots, \sigma_n$ are independent Rademacher random variables.
  In particular, for $0 < p \le p_0 \approx 1.8$, $A_p = 2^{1/2 - 1/p}$.
\end{fact}
\noindent
This implies:
\begin{lemma}[Cotype Estimate]
\label{lemma:cotype-estimate}
  For any $p \in (0, 1]$, $\ell_{p}^{d'}$ has cotype $2$ with cotype constant $1/A_p$ with $A_p$ as in \cref{fact:khintchine}, i.e. for any $x_1 \ldots, x_n \in \R^{d'}$,
  \[
    A_p \sqrt{\sum_{i = 1}^n \|x_i\|_p^2} \le
    \EV \left\| \sum_{i = 1}^n \sigma_i x_i    \right\|_p \; ,
  \]
  where $\sigma_1, \ldots, \sigma_n$ are independent Rademacher random variables.
\end{lemma}
\begin{proof}
  Let $x_{ij}$ denote coordinate $j$ of $x_i$, i.e., $X$ is the $n
  \times d$ matrix whose \emph{rows} are $x_1, \ldots, x_n$.
  By Khintchine's inequality,
  \begin{align*}
    \mathbb{E}\left\| \sum_{i = 1}^n \sigma_i x_i  \right\|_p^p
    &=  \sum_{j = 1}^d \mathbb{E}\left|\sum_{i = 1}^n \sigma_i x_{ij}    \right|^p   \\
    &\ge A_p^p  \sum_{j = 1}^d \left(\sum_{i =  1}^nx_{ij}^2\right)^{p/2}.
  \end{align*}

  Let us now consider the case $p \le 2$.
  By the triangle inequality for $\|\cdot\|_q$, where $q = 2/p \ge 1$,
  applied to the vectors $(|x_{1j}|^p, |x_{2j}|^p, \ldots,
  |x_{nj}|^p)$, $j \in [d]$,
  \begin{align*}
    \sum_{j = 1}^d \left(\sum_{i = 1}^n x_{ij}^2\right)^{p/2}
    &= \sum_{j = 1}^d \left(\sum_{i = 1}^n (|x_{ij}|^p)^{2/p}\right)^{p/2}\\
    &\ge  \left(\sum_{i = 1}^n \left| \sum_{j = 1}^d |x_{ij}|^p\right|^{2/p} \right)^{p/2}\\
    &=  \left(\sum_{i = 1}^n \|x_i\|_p^2 \right)^{p/2}.
  \end{align*}
 Finally, by the concavity of the function $x \mapsto x^{p}$ for $p \in (0, 1]$, we have
 \begin{align*}
    \lp\EV\left\| \sum_{i = 1}^n \sigma_i x_i  \right\|_p\rp^p \ge \EV\left\| \sum_{i = 1}^n \sigma_i x_i  \right\|_p^p.
 \end{align*}
\end{proof}
We now have all the tools we need to prove \cref{thm:distortion-lp}:
\begin{proof}[Proof of \cref{thm:distortion-lp}]
For brevity, let $p = 0.99$ in this proof.
Let $T: \R^d \to \R^{d'}$ be any linear map satisfying
\[
\alpha \| T x \|_p \leq \| x \| \leq \beta \| T x \|_p\; .
\]
Let $C_2 = C_2 ((\R^d, \| \cdot \|))$, and let $x_1, \ldots, x_n$ be a sequence in $(\R^d, \| \cdot \|)$ satisfying
\[
\mathbb{E} \left[ \left\| \sum_{j = 1}^n \sigma_j x_j \right\| \right] = C_2 \sqrt{ \sum_{j = 1}^n \| x_i \|^2} \; .
\]
However, we have that
\[
\alpha \mathbb{E} \left[ \left\| \sum_{j = 1}^n \sigma_i T x_j \right\|_p \right] \le\cdot \mathbb{E} \left[ \left\| \sum_{j = 1}^n \sigma_j x_j \right\| \right] \; ,
\]
and simultaneously, we have
\[
    \sqrt{\sum_{j = 1}^n \| x_i \|^2 }
    \le
    \beta \sqrt{ \sum_{j = 1}^n \|T x_i \|_p^2 }
    \;.
\]
Combining these facts and \cref{lemma:cotype-estimate}, we obtain that $\beta / \alpha \geq A_p C_2 = \Omega (C_2)$, as claimed, where $A_p$ is as in \cref{lemma:cotype-estimate} and \cref{fact:khintchine}.
\end{proof}

\end{document}